\documentclass[11pt,a4paper]{article}

\usepackage[utf8]{inputenc} 
\usepackage[T1]{fontenc}    
\usepackage{hyperref}       
\usepackage{url}            
\usepackage{booktabs}       
\usepackage{amsfonts}       
\usepackage{nicefrac}       
\usepackage{microtype}      
\usepackage{xcolor}         
\usepackage{amsmath}
\usepackage{amsthm}
\usepackage{amssymb}
\usepackage{graphicx}
\usepackage{bbm}
\usepackage{thm-restate}
\usepackage{caption}
\usepackage{subcaption}
\usepackage{wrapfig}
\usepackage{multicol}
\usepackage{chngcntr}
\usepackage{apptools}
\usepackage[left=1in,right=1in,top=1in,bottom=1in]{geometry}
\usepackage{times} 
\usepackage{helvet}
\usepackage{courier}
\usepackage{tikz}
\usetikzlibrary{arrows}
\usetikzlibrary{shadows}
\usepackage{accents}
\usepackage{indentfirst}
\usepackage[title]{appendix}
\usepackage{tabularx,booktabs}
\usepackage{nicematrix}
\usepackage{arydshln}
\usepackage{wrapfig}
\usepackage{natbib}
\usepackage{algorithm}
\usepackage{algpseudocode}
\usepackage{multirow, mathtools}
 
\newtheorem{thm}{Theorem}
\newtheorem{defn}[thm]{Definition}
\newtheorem{prop}[thm]{Proposition}

\newtheorem{ex}[thm]{Example}
\newtheorem{rem}[thm]{Remark}

\AtAppendix{\counterwithin{lem}{section}}

\DeclareMathOperator*{\argmin}{argmin}
\DeclareMathOperator*{\argmax}{argmax}

\DeclareMathOperator{\tc}{\Pi_{TC}}

\makeatletter
\renewcommand{\ALG@name}{Workflow}
\makeatother
\newcommand{\Pitc}{\Pi_{\mbox{\tiny TC}}}

\newcommand{\oj}{\mathcal{S}_{\tiny c}}
\newcommand{\ojk}{\mathcal{S}_{\tiny c_k}}

\newcommand{\bbR}{\mathbb{R}}
\newcommand{\bbN}{\mathbb{N}}

\newcommand{\cA}{\mathcal{A}}
\newcommand{\cB}{\mathcal{B}}
\newcommand{\cX}{\mathcal{X}}
\newcommand{\cY}{\mathcal{Y}}

\newcommand{\eqd}{\stackrel{\mbox{\tiny d}}{=}}

\newcommand{\real}{\mathbb{R}}
\newcommand{\In}{\mathbb{I}}
\newcommand{\E}{\mathbb{E}}

\newcommand{\bP}{\mbox{P}}
\newcommand{\bp}{\mbox{p}}
\newcommand{\bQ}{\mbox{Q}}
\newcommand{\bq}{\mbox{q}}
\newcommand{\bR}{\mbox{R}}

\title{\textbf{Alignment and Comparison of Directed Networks via Transition Couplings of Random Walks}}

\author{
Bongsoo Yi\thanks{Bongsoo Yi and Kevin O'Connor are co-first authors for the paper.} \\
Department of Statistics and Operations Research, \\
University of North Carolina at Chapel Hill,
Chapel Hill, North Carolina,
USA
\and 
Kevin O'Connor$^{*}$   \\
Department of Statistics and Operations Research, \\
University of North Carolina at Chapel Hill,
Chapel Hill, North Carolina,
USA
\and 
Kevin McGoff \\
Department of Mathematics and Statistics, \\
 University of North Carolina at Charlotte, Charlotte, North Carolina, USA
\and 
Andrew B. Nobel\footnote{Corresponding author; E-mail: \href{mailto:nobel@email.unc.edu}{nobel@email.unc.edu}}\\
Department of Statistics and Operations Research, \\
University of North Carolina at Chapel Hill,
Chapel Hill, North Carolina,
USA
}

\date{}
\begin{document}

\maketitle

\begin{abstract}
We describe and study a transport based procedure called NetOTC (network optimal transition coupling)
for the comparison and alignment of two networks. 
The networks of interest may be directed or undirected, weighted or unweighted, and may have distinct vertex sets of different sizes.   
Given two networks and a cost function relating their vertices, NetOTC finds a transition coupling of their associated random 
walks having minimum expected cost. 
The minimizing cost quantifies the difference between the networks, while the optimal transport plan itself
provides alignments of both the vertices and the edges of the two networks. 
Coupling of the full random walks, rather than their marginal distributions, ensures that NetOTC captures 
local and global information about the networks, and preserves edges.  NetOTC has no free parameters,
and does not rely on randomization.
We investigate a number of theoretical properties of NetOTC and present experiments establishing its empirical performance.
\end{abstract}

\small\textbf{Keywords:} graph alignment, graph comparison, graph factor, optimal transport

\newpage

\section{Introduction}

Networks have long been used as a means of representing and studying the pairwise interactions between a set of individuals or objects under study.  More recently, networks themselves have become objects of study, including exploratory analysis and statistical modeling. When an application of interest involves multiple networks, analysis often begins with the problem of network alignment or comparison, tasks that have been studied in a number of fields. 
Network alignment compares and finds correspondences among nodes or edges within multiple networks. The aim is to recognize similar substructures, unveiling hidden relationships and functional similarities that exist within different networks.
In the simplest version of the network alignment problem, one is given two networks with vertex sets of equal size and seeks to find a bijection between the vertex sets that maximizes the number of aligned edges. 
Numerous approaches to network alignment have been considered in the literature, e.g. \cite{Kelley2003ConservedPW, Kuchaiev2010TopologicalNA, Kuchaiev2011IntegrativeNA, Kalaev2008FastAA, Klau2009ANG}.
By contrast, the goal of network comparison is more general: given two networks of different size or structure, identify and quantify similarities between them in a rigorous manner.  Perhaps the simplest form of comparison is a numerical measure of similarity between networks, which one might hope to have the properties of a metric. Potentially more informative measures of comparison include soft alignment of vertices and edges in the two networks.

Network alignment and comparison arise in a number of disciplines where network models are common.
In biology, networks have been used to represent protein-protein interactions and gene regulatory systems.
Network alignment has been used to identify interaction structures and topological similarities, facilitating the transfer of biological insights from familiar to unexplored species~\citep{Singh2008global,Elmsallati2016GlobalAO,ma2020review}.
In connectomics and neuroimaging, networks are used to model connectivity and interactions of different brain regions. 
Network alignment and comparison methods have been used to compare the connectomes of healthy and diseased 
individuals~\citep{zalesky2010network, Milano2017AnEA} as an initial step in identifying potential indicators of disease, 
understanding disease origins, and identifying specific locations within the brain that could influence the progression or onset of the disease.
Beyond biology and neuroscience, network comparison methods have found applications in economics
~\citep{fagiolo2010evolution, Engel2021}, where they have been used to compare trade networks over time,
and to identify shifts and trends in economic dynamics. 
In social science, comparisons of social networks have offered insights into the relationships and interactions between different groups~\citep{jackson2014social,mislove2007measurement}.


In this paper, we propose and analyze NetOTC, a procedure for the comparison and soft alignment of two networks. 
The networks of interest may be directed or undirected, weighted or unweighted, and may have distinct vertex sets 
of different sizes.
NetOTC, which is short for network optimal transition coupling, is based on applying a process-level
optimal transport method to the random walks arising from each network.   
In more detail, the NetOTC procedure takes as input two connected networks $G_1$ and $G_2$ 
with non-negative edge weights, which may be
directed or undirected, and a cost function relating their vertices.  
Each network gives rise to a stationary random walk on its vertex set whose transition probabilities are determined by
normalizing the edge weights at each node.  
NetOTC proceeds by finding a joint chain on the product of the vertex sets of $G_1$ and $G_2$
that has the following properties: the joint chain is stationary and Markov; 
the joint chain minimizes the expected value of the cost function 
at any fixed time point; and the transition distribution from each state $(u,v)$ in the joint chain is obtained by coupling 
the transition distribution of $u$ in $G_1$ with the transition distribution of $v$ in $G_2$.  The latter condition, which gives rise to
transition couplings, ensures that the joint chain is a coupling of the initial chains on the individual graphs.  More information
about process couplings and transition couplings can be found in Section \ref{sec:transition_couplings}.  We note that
minimization of the expected cost subject to the constraints given above is carried out analytically, and not through 
Monte Carlo methods.

As described above, the optimal transport plan arising from NetOTC is a stationary random walk on the product of the given networks 
that favors low cost pairs of vertices, while maintaining the marginal structure of the random walks on each individual network.
The cost function used in NetOTC is specified by the user, and will, in general, be application dependent.  
Cost functions can be based, for example, on the difference between externally specified vertex attributes, the distance between Euclidean embeddings of the vertices, or the difference between the degrees of the vertices.  If the given networks have the same vertex set, 
a cost function based on vertex identity may be used as well.

The NetOTC procedure has a number of desirable methodological and theoretical properties.
On the methodological side, NetOTC applies to directed and undirected networks, and readily handles networks with different sizes and 
connectivity structures.  
NetOTC has no free parameters and does not make use of randomization or Monte Carlo techniques.  
As NetOTC considers process-level couplings of random walks, the optimal transport plan
captures global information reflected in the stationary distributions of the random walks, 
as well as local information that is present in the transition probabilities between vertices. 
The expected cost of the optimal transport plan provides a numerical measure of the difference between the networks.  
The distribution of the optimal transport plan at a single time point yields a soft, probabilistic alignment of the vertices in the given networks.  
Moreover, the distribution of the optimal transport plan at two successive time points yields a soft, probabilistic alignment of the {\em edges} of the given networks.
To the best of our knowledge, native alignment of edges is unique to NetOTC among existing alignment and comparison methods.  
Once the vertex cost function has been specified, the exact version of the NetOTC procedure has no free parameters and does
not make use of randomization.  

On the theoretical side, we establish several key properties of the NetOTC procedure that support its use in comparison and alignment tasks.
The edge alignment of NetOTC respects edges, in the sense that vertex pairs in the given networks are aligned with positive probability 
only if each pair is connected by an edge in its respective network.
Although the NetOTC optimal transport plan minimizes the expected cost between vertices at time zero,
stationarity ensures that the same is true at any other fixed time, and that
the coupled random walk has a low average cost between vertices across time.
The NetOTC similarity measure is sensitive to differences in the $k$-step behavior of
the random walks on $G_1$ and $G_2$ if these differences affect the cumulative cost.  
For undirected networks with a common vertex set,
the NetOTC similarity is a metric on equivalence classes of networks 
having identical random walks if the cost $c$ is a metric.  
For the zero-one cost, the NetOTC similarity is lower bounded by the $\ell_1$-difference of the network 
degree sequences, and the $\ell_1$-difference of the network weight functions.

In addition, we study the structure of the NetOTC method through network factors.  
Our definition of factor, which arises naturally when considering functions of Markov chains, differs from other definitions in the network theory literature.
Informally, a network $G_2$ is a factor of a network $G_1$ if the vertices of $G_2$ can be associated, 
via a vertex map $f$, with disjoint
sets of vertices in $G_1$ between which aggregate weights can be consistently defined.
We show that if $G_2$ is a factor of $G_1$, then the vertex map
$f$ yields a deterministic transition coupling of their associated random walks.
Under suitable compatibility conditions on the cost function, this coupling will be optimal and will provide a solution to the NetOTC problem.  
The resulting expected cost, vertex alignment, and edge alignment are fully determined by the structure of $G_1$ and the map $f$.
Importantly, the existence and precise nature of the map $f$ need not be known to the NetOTC procedure.
These results extend to paired factors: an optimal transition coupling for a pair of networks can be mapped in a deterministic fashion to an optimal transition coupling of their factors
when the cost functions for each pair are compatible with the factor maps.  

As a complement to the theory, we carry out a number of simulations and numerical experiments to assess the performance of NetOTC and compare it with other optimal transport-based comparison methods in the literature.  
NetOTC is competitive with other methods on a number of network classification tasks.
In an extensive experiment on pairs of isomorphic networks with small to moderate sizes, NetOTC was consistently able to recover the isomorphism using a local (degree-based) cost function, substantially outperforming other methods.
When applied to stochastic block models (with equivalent blocks of different sizes) using a degree-based cost function, NetOTC was competitive with other methods in its ability to align vertices in equivalent blocks, and substantially better at aligning edges.  
We also considered the problem of comparing a network to an exact or approximate factor using a distance-based cost derived from Euclidean vertex embeddings of the given networks.
NetOTC outperforms other methods in its ability to align vertices in the parent and factor networks.  While the performance gap is modest for exact factors, it increases as one considers approximate factors.

\subsection{Outline of the Paper}

The next section gives an overview of existing work on optimal transport and related approaches to network comparison and alignment.  Section \ref{sec:transition_couplings} 
provides background concerning random walks on directed networks, optimal transport for Markov chains, 
and transition couplings.  
The NetOTC procedure is described in Section \ref{sec:graphotc}, including 
computation, the optimal transport cost, and vertex and edge alignment.
Section \ref{sec:theory} is devoted to the formal statement and discussion of the theoretical properties of NetOTC.
Proofs are given in Section \ref{sec:proofs}.
Section \ref{sec:experiments} contains a number of simulations and a number of experiments that demonstrate the flexibility and potential utility of NetOTC. 
Additional details concerning the experiments are given in Appendix \ref{app:exp_details}. 

\vskip.3in

\section{Related Work}
\label{sec:related_work}


The problems of network alignment and comparison have received a lot of recent attention in the literature.  Approaches using optimal transport ideas can be divided into several groups: spectral methods, variants of Gromov-Wasserstein, and methods involving random walks and Markov chains.  Other approaches make use of quadratic programming
and continuous approximations.  This related work is discussed below.

\paragraph{Spectral Methods.} 

One line of work \citep{dong2020copt, maretic2019got, maretic2020wasserstein} uses techniques from spectral graph theory to define optimal transport (OT) problems for networks.
In particular, this approach associates to each network a multivariate Gaussian with zero mean and covariance matrix equal to the pseudo-inverse of the graph Laplacian.
The Wasserstein distance between Gaussian distributions of the same dimension may be computed analytically 
in terms of their respective covariance matrices.
For networks with different numbers of vertices, \cite{maretic2020wasserstein} and \cite{dong2020copt} propose to optimize this distance over soft many-to-one assignments between vertices in either network.
At present, this family of approaches is unable to incorporate available feature information or underlying cost functions, relying only on their intrinsic structure.

\paragraph{Variants of Gromov-Wasserstein.}
Another line of work \citep{memoli2011gromov, peyre2016gromov, titouan2019optimal,vayer2019sliced,vayer2020fused} considers the Gromov-Wasserstein (GW) distance and related extensions.
In this work, one tries to couple distributions on the vertices in each network so as to minimize an expected transport cost between vertices while minimizing changes in edges between the two networks.
This approach allows one to capture differences in both features and structure between networks.
We refer the reader to \cite{dong2020copt} for a discussion on the differences between spectral-based network OT methods and GW distances.
A number of variants of the GW distance have been proposed for a variety of tasks including cross-domain alignment \citep{chen2020graph}, graph partitioning \citep{xu2019scalable}, graph matching \citep{xu2019scalable, xu2019gromov}, and node embedding \citep{xu2019gromov}.
The work \citep{barbe2020graph} proposes to incorporate global structure into the Wasserstein and Fused GW (FGW) distances by applying heat diffusion to the vertex features before computing the cost matrix.

\paragraph{Methods involving random walks and Markov chains.}
It is well-known that a weighted network $G$ with non-negative edge weights can be viewed as a Markov 
chain $X = \{X_k\}_{k=1}^{\infty}$, and there is previous work that uses this perspective to align or compare networks. 
The paper~\citep{svn2010graph} studies a flexible family of kernels for comparing two given networks. 
Given networks $G$ and $H$ with associated 
transition matrices $P$ and $Q$ the 
kernels take the form $\kappa(G,H) = \sum_{k \geq 1} \mu_k q^t (P \otimes Q)^k p$, where $p$ and $q$ are
starting and stopping distributions, $\mu_k$ are non-negative weights, and 
$(P \otimes Q)^k$ is the k-step transition matrix of the independent coupling of the
random walks on $G$ and $H$. 
The free parameters $p, q, \{\mu_k\}$ are user specified; appropriate choices 
allow for efficient computation. 
The kernels $\kappa(G,H)$ are distinct from NetOTC, as they employ only independent couplings, and do not involve the use of optimal transport.

Several recent papers~\citep{chen2022weisfeiler, Chen2023TheWD} have studied a Markov chain-based distance function
for networks that has close connections to the classical Weisfeiler-Lehman (WL) test for graph isomorphism. 
Given two weighted networks $G$ and $H$ and a fixed time $k \geq 1$, the $k$-step WL-distance 
$d_{\text{WL}}^k(G,H)$ is equal to the minimum expected cost $\min \mathbb{E}c(\tilde{X}_k,\tilde{Y}_k)$ at time $k$, where the minimum is taken over 
all (possibly time-inhomogeneous) Markovian couplings of $X$ and $Y$. 
The transition couplings used in the present work have the additional requirement that they must be 
time-homogeneous and stationary (see Section \ref{sec:transition_coupling}). 
The family $\{d_{\text{WL}}^k : k \geq 1\}$ 
is further investigated in~\citep{Brugere2023DistancesFM}, where it is shown (in Proposition 23) that 
\begin{equation} \label{d-infinity}
d_{\text{WL}}^{\infty}(G,H) := \lim_k d_{\text{WL}}^k(G,H) = \sup_k d_{\text{WL}}^k(G,H).
\end{equation} 
Moreover, it is shown in Proposition 5 of~\citep{Brugere2023DistancesFM} 
that 
\begin{equation} \label{d-WL-inequality}
d_{\text{WL}}^{\infty}(G,H) \leq d_{\text{OTC}}(G,H),
\end{equation} 
where $d_{\text{OTC}}$ is the NetOTC distance studied in this paper.
Using (\ref{d-infinity}) and (\ref{d-WL-inequality}) one may verify that the $k$-step WL-distance is not in general equal 
to $d_{\text{OTC}}$. 
Furthermore, (\ref{d-WL-inequality}) ensures that the NetOTC distance has at least as much discriminatory power as 
the WL-test in the graph isomorphism detection problem. 
We refer the reader to \cite{Brugere2023DistancesFM} for more details.

Lastly, let us mention that there is a constrained optimal transport method called Causal Optimal Transport (COT) that has been used in finance and machine learning~\citep{Lassalle2013CausalTP}. By Lemma 3.11 in \cite{Chen2023TheWD}, any Markovian coupling is also a bi-causal coupling, and therefore the transition couplings considered in this paper are also bi-causal. However, Theorem 3.12 from \cite{Chen2023TheWD} states that the $k$-step WL distance is equal to the bi-causal transport distance in which the given cost function is evaluated at time $k$. Thus, the relationship between the OTC distance and COT is the same as the relationship between OTC and the WL-distances: they are distinct notions, with $d_{\text{OTC}}$ greater than or equal to bi-causal transport distance at any fixed time. 

\paragraph{Other Methods for Network Alignment and Comparison.}
There is also a large body of work devoted to network alignment and comparison that does not use optimal transport methods.
The network alignment problem can be generally defined as a quadratic programming problem under discrete and doubly stochastic constraints \citep{yan2016ashort, jiang2017graph, loiola2007asurvey, cho2010reweighted, cour2007balanced}.
However, as the optimal network alignment problem is well known to be an NP-hard problem \citep{garey1990computers}, it is computationally challenging to obtain an optimal alignment for networks. 
For this reason, many authors have proposed approximate solutions for network alignment \citep{cho2010reweighted, zhou2016factorized, yu2018generalizing, enqvist2009optimal, van2004apocs, zaslavskiy2009apath}. 
Among these approximate methods, most of the successful algorithms start with relaxing the discrete constraints to create a continuous condition. Several authors
\citep{schellewald2005probabilistic, torr2003solving} relax the discrete conditions to form a positive semi-definite problem. A non-convex quadratic programming problem was adopted in \cite{gold1996agraduated, cho2010reweighted, zhou2016factorized}. In another direction, \cite{leordeanu2005aspectral} introduces spectral matching as a simple relaxation, and \cite{cour2007balanced} strengthens this approach by giving an affine constraint.
Also, \cite{jiang2017graph} proposes an algorithm that can efficiently solve a general quadratic programming problem with doubly stochastic constraints. Each step of the algorithm is easy to implement and the convergence is guaranteed. Generally, after finding the optimal solution for the relaxed continuous problem, the discrete alignment is attained through a final discretization process \citep{cho2010reweighted, leordeanu2005aspectral, leordeanu2009aninteger}.
We note that these approaches may find a solution that is locally optimal but not globally optimal.

Another line of research is devoted to the statistical and probabilistic analysis of graph matching when the 
given graphs are generated at random but are correlated with one another, e.g., each is a random 
perturbation of a given graph \citep{Ding2020EfficientRG, Barak2019NearlyEA, Cullina2020PartialRO, Cullina2017ExactAR, Feizi2020SpectralAO, Korula2014AnER, Lyzinski2014SeededGM, Yartseva2013OnTP}. 
Much of this work investigates various matching procedures under specific random graph models, 
such as correlated Erdos-Renyi graphs,
and is concerned with the information-theoretic threshold for exact recovery, and the time complexity 
of the matching procedures.

\vskip.3in

\section{Transition Couplings of Random Walks on Networks}
\label{sec:transition_couplings}

This section provides background for the detailed description of the NetOTC procedure. 
We begin by recalling how a weighted network gives rise to a random walk on its vertex set and reviewing the definition and framework
of optimal transport. 
We then consider transition couplings of random walks, which preserve stationarity and the Markov property.  
The computation of optimal transition couplings is the basis for the NetOTC procedure, which is described in Section \ref{sec:graphotc} below. 

\subsection{Random Walks on Directed Networks}

Let $G = (U, E, w)$ denote a network with finite vertex set $U$, edge set
$E \subseteq U \times U$, and non-negative weight function $w: U \times U \to \real_+$.
An ordered pair $(u,u') \in E$ represents a directed edge from $u$ to $u'$ with weight $w(u,u')$.  
We assume in what follows that $w(u,u') > 0$ if and only if $(u,u') \in E$.  
For any vertex $u \in U$, we let $d(u) = \sum_{u' \in U} w(u,u')$ be the weighted out-degree of $u$. 
An undirected network is represented by a directed network in which
$w(u,u') = w(u',u)$ for each $u, u' \in U$. 
A path in $G$ is an ordered sequence of vertices $u_0,\dots,u_n \in U$ such that $(u_{i-1},u_{i}) \in E$ for each $i=1,\dots,n$.  
A network $G$ is strongly connected if for each ordered pair $(u,v) \in U \times U$ there exists a path 
$u_0, \dots, u_n$ in $G$ such that $u_0 = u$ and $u_n = v$.  In this case $d(u) > 0$ for each vertex $u \in U$.

To any network $G = (U,E,w)$, one may associate a Markov transition kernel $\bP( \cdot \mid \cdot)$ 
with state space $U$ as follows: for each pair of vertices $u$ and $u'$, the probability of transitioning from $u$ to $u'$ is 
\begin{equation}\label{eq:net_to_trans}
\bP( u' \mid u) = \frac{w(u,u')}{ d(u) }.
\end{equation}
Recall that a probability distribution $\mbox{p}$ on $U$ is said to be stationary for $\bP(\cdot \mid \cdot)$ if
$\mbox{p}(u') = \sum_{u \in U} \mbox{p}(u) \, \bP( u' \mid u)$ for all $u' \in U$.
It follows from the Perron-Frobenius theorem that the transition kernel $\bP(\cdot \mid \cdot)$ admits at least one stationary distribution $\mbox{p}$. 
Together, $\bP$ and $\mbox{p}$ define a stationary 
Markov chain $X = X_0, X_1, X_2, \dots \in U$ such that for any $u_0,\dots,u_n$ in $U$ 
\begin{equation*}
\mathbb{P} \bigl( X_0 = u_0,\dots,X_n = u_n) \ = \ \mbox{p}(u_0) \, \prod_{i=0}^{n-1} \bP(u_{i+1} \mid u_i).
\end{equation*}
The Markov chain $X$ is commonly referred to as a random walk on $G$. 
When $G$ is strongly connected, the kernel $\bP( \cdot \mid \cdot)$ admits a \textit{unique} stationary distribution,
and we refer to $X$ as \textit{the} random walk on $G$.  Random walks on networks have
been studied extensively in the probability literature, and have found numerous applications in fields ranging
from genomics to computer science, including 
recent work on network embedding \citep{HamiltonGraph, Grover2016node2vecSF, Perozzi2014DeepWalkOL}.
In what follows, we will generally assume that the networks under consideration are strongly connected.  
In this case, the random walk $X$ on $G$ is an irreducible Markov chain (Chapter 4 of \cite{blum2020foundations}).

\subsection{Optimal Transport}

Let $X$ and $Y$ be random objects taking values in sets $\cX$ and $\cY$, respectively.
In what follows we are primarily interested in the case that $X$ and $Y$ are processes.
A \textit{coupling} of $X$ and $Y$ is a jointly distributed pair $(\tilde{X}, \tilde{Y})$ of random objects
taking values in $\cX \times \cY$ with the property that $\tilde{X} \eqd X$ and $\tilde{Y} \eqd Y$.
Here $\tilde{X} \eqd X$ means that $\tilde{X}$ and $X$ have the same distribution on $\cX$,
and $\tilde{Y} \eqd Y$ is interpreted similarly.  The distinction between $X, Y$ and
$\tilde{X}, \tilde{Y}$ arises from the fact that $X$ and $Y$ are understood and specified 
through their individual distributions, whereas $\tilde{X}$ and $\tilde{Y}$ are understood and specified as
a jointly distributed pair.  (In general, $X$ and $Y$ may be defined on different probability spaces,
whereas $\tilde{X}$ and $\tilde{Y}$ are necessarily defined on the same probability space.)
Couplings have been widely studied in the probability literature, 
and are the basic objects of interest in optimal transport. 

Let $\Pi(X,Y)$ denote the set of all couplings $(\tilde{X}, \tilde{Y})$ of $X$ and $Y$.  
Note that $\Pi(X,Y)$ is not empty, as it always contains the independent coupling $(\tilde{X}, \tilde{Y})$ in which 
$\tilde{X}$ and $\tilde{Y}$ are independent copies of $X$ and $Y$, respectively.
Each coupling $(\tilde{X},\tilde{Y}) \in \Pi(X,Y)$ is associated with a joint distribution $\pi$ on 
$(\cX \times \cY, \cA \times \cB)$ that can be viewed conditionally as a plan for transporting the distribution 
of $X$ to that of $Y$ and vice versa.
Let $c : \cX \times \cY \to \mathbb{R}_+$ be a measurable, non-negative cost function relating the 
elements of $\cX$ and $\cY$. 
The optimal transport problem is to minimize the expected value of the cost function over all couplings of $X$ and $Y$, namely
\begin{equation*}
\mbox{minimize } \ \mathbb{E}c(\tilde{X},\tilde{Y}) \ \mbox{ over } \ (\tilde{X},\tilde{Y}) \in \Pi(X,Y).
\end{equation*}
A minimizer of the optimal transport problem is called an optimal coupling of $X$ and $Y$, or an optimal transport plan.  
The theory and applications of 
optimal transport are active areas of research. 
See \cite{peyre2019computational, Villani2008OptimalTO} for further reading and more details.

\subsection{Transition Couplings}\label{sec:transition_coupling}

Let $G_1 = (U, E_1, w_1)$ and $G_2 = (V, E_2, w_2)$ be weighted directed networks, and 
let $c: U \times V \to \real_+$ be a cost function relating their vertex sets.   
As described above, the network $G_1$ is associated with a random walk
$X = X_0, X_1, \ldots$ on the vertex set $U$.  We may regard the process $X$ as a random element of  
the set $\cX = U^{\bbN}$ equipped with the Borel sigma-field arising from the usual product topology on $U^{\bbN}$. 
Similarly, the network $G_2$ is associated with a random walk $Y = Y_0, Y_1, \ldots$ taking values in 
$\cY = V^{\bbN}$.
A coupling of the processes $X$ and $Y$ is a joint process 
\[
(\tilde{X}, \tilde{Y}) = (\tilde{X}_0, \tilde{Y}_0), (\tilde{X}_1, \tilde{Y}_1), (\tilde{X}_2, \tilde{Y}_2), \ldots
\]
with values in $U \times V$
such that $\tilde{X} = \tilde{X}_0, \tilde{X}_1, \ldots \eqd X$ and $\tilde{Y} = \tilde{Y}_0, \tilde{Y}_1, \ldots \eqd Y$,
where $\eqd$ indicates equality of distribution.  
In general, a coupling of $X$ and $Y$ need not be stationary
or Markov, an issue that we take up below.

In studying optimal transport of the random walks $X$ and $Y$ we make use of the single letter cost 
$\tilde{c}: \cX \times \cY \to \real$ defined by $\tilde{c}(x,y) = c(x_0,y_0)$.
The standard optimal transport problem with the cost $\tilde{c}$ seeks to minimize 
$\E c(\tilde{X}_0, \tilde{Y}_0)$ over the family $\Pi(X,Y)$ of all couplings of 
the Markov chains $X$ and $Y$. 
However, for most purposes $\Pi(X,Y)$ is too large: in general, it will include
couplings that are non-stationary, and not Markov of any order.  
Without further restrictions, an optimal coupling will minimize the expected cost 
only at time zero, after which the processes $\tilde{X}$ and $\tilde{Y}$ may evolve independently 
(and potentially have a large realized cost).
Restricting attention to stationary couplings addresses some of these 
issues \citep{o2020optimal}.  We note that 
stationary couplings of stationary processes, also known as joinings,
have been widely studied in the ergodic theory literature (see
\cite{de2005introduction, glasner2003ergodic, ornstein1973application} and the references therein).

When considering random walks $X$ and $Y$ on graphs, which are Markov chains,
it is natural to consider couplings $(\tilde{X}, \tilde{Y})$ 
that are themselves Markov chains, so that the structure of the couplings matches that of the walks.
Unfortunately, even the family of stationary first order Markov couplings presents some difficulties: 
there is no fast method for computing optimal couplings, and the optimal expected cost need not
have the properties of a metric even when the cost $c$ does \citep{ellis1976thedj, ellis1978distances}.
For these reasons, we restrict attention to the subfamily of transition couplings, which are defined below.

\begin{defn} \label{def:TC}
Let $X$ be a stationary Markov chain with values in $U$ and transition kernel $\bP$,
and let $Y$ be a stationary Markov chain with values in $V$ and transition kernel $\bQ$. 
A stationary Markov chain $(\tilde{X},\tilde{Y})$ with values in $U \times V$ is a 
{\em transition coupling} of $X$ and $Y$ 
if it is a coupling of $X$ and $Y$ and if it has a transition kernel $\bR$ such that for every $u_0,u_1 \in U$ and $v_0,v_1 \in V$,
\begin{equation} \label{defn:TC}
\sum_{v \in V} \bR(u_1, v  \, | \, u_0, v_0)  =  \bP(u_1 \, | \, u_0)
\hskip.2in \text{and} \hskip.2in
\sum_{u \in U} \bR(u , v_1  \, | \, u_0, v_0)  =  \bQ(v_1 \, | \, v_0).
\end{equation} 
Let $\Pitc(X,Y)$ denote the set of all transition couplings of $X$ and $Y$. 
When (\ref{defn:TC}) holds, we will also say that $\bR$ is a transition coupling of $\bP$ and $\bQ$.
\end{defn}

The transition coupling condition \eqref{defn:TC} can be stated equivalently as follows:
for every state $(u_0,v_0) \in U \times V$ of the joint chain, the distribution 
$\bR(\cdot  \mid  u_0, v_0)$ of the next state
is a coupling of the next state distributions $\bP(\cdot  \mid  u_0)$ and $\bQ(\cdot  \mid  v_0)$ 
of the individual chains.  
The set of transition couplings $\tc(X,Y)$ is non-empty, as the independent coupling of 
 $X$ and $Y$, with transition kernel 
 $\bR(u', \hspace{-.01in} v' \mid u, v) = \bP(u' \hspace{-.01in} \mid u) \, \bQ(v' \hspace{-.01in} \mid v)$, is a transition coupling. 

Couplings have long been employed in the analysis of Markov chains, often to study the rate at which the marginal 
distribution of a chain started from a particular state converges to the stationary distribution of the chain. 
In a typical analysis, two versions of a chain are run from different initial conditions until they reach the same state, 
after which they coincide.
The transition couplings considered here are couplings of two distinct processes, one for each of the given networks. 
Transition couplings as defined in Definition \ref{def:TC} are sometimes called Markovian couplings in the probability literature~\citep{levin2017markov}, 
but the use of this terminology is not standardized.  
The term transition coupling that is used here was introduced in \cite{o2020optimal}.
\cite{Chen2023TheWD} use the term Markovian coupling to refer to the class of time-inhomogeneous (non-stationary) Markov couplings 
in which transition probabilities may vary from time point to time point, and for which the initial distribution is a coupling of the initial distribution of the
given processes.

\vskip.3in

\section{NetOTC}
\label{sec:graphotc}

In this section, we describe the NetOTC procedure in more detail, including a statement of the 
NetOTC problem, as well as exact and approximate computational methods for its solution.

\subsection{The NetOTC Problem}

Let $G_1$ and $G_2$ be strongly connected networks of interest.  
Each network gives rise to a unique random walk on its vertex set, whose transition probabilities
are determined by their connectivity and edge weights; the stationary distribution of the walk 
reflects the global structure of the network, while the transition probabilities of the walk reflect the local structure of the network.  
Let $X$ and $Y$ be the walks associated with $G_1$ and $G_2$, respectively. 
In the NetOTC problem, we seek to minimize the expected cost $\E c(\tilde{X}_0, \tilde{Y}_0)$ over all transition couplings
$(\tilde{X}, \tilde{Y})$ of $X$ and $Y$.   In particular, we wish to identify both the minimizing value
\begin{equation} \label{eqn:GOTCmin}
\rho(G_1, G_2) \ = \ 
\min_{(\tilde{X}, \tilde{Y}) \, \in \, \tc(X,Y)} \E c(\tilde{X}_0, \tilde{Y}_0),
\end{equation}
and an associated optimal transition coupling
\begin{equation} \label{eqn:GOTCargmin}
(X^*,Y^*) \ \in \ \argmin_{(\tilde{X}, \tilde{Y}) \, \in \, \tc(X,Y)} \E c(\tilde{X}_0, \tilde{Y}_0). 
\end{equation}
An optimal transition coupling is a stationary random walk
\[
(X^* \hspace{-.02in},Y^*) = (X_0^* \hspace{-.02in},Y_0^*) , (X_1^* \hspace{-.02in},Y_1^*), \ldots
\] 
on the product $U \times V$ that preserves the marginal behavior of the walks $X$ and $Y$, 
while favoring pairs $u,v$ with low cost.  
In particular, $(X^* \hspace{-.02in},Y^*)$ is an optimal coupling of the {\it processes} $X$ and $Y$, not just
their one-dimensional (stationary) distributions.
As such, the optimal transport plan identified by NetOTC 
captures and links the local and global structure of the given networks.

As noted above, the set $\tc(X,Y)$ of transition couplings is non-empty. 
We endow $\tc(X,Y)$ with the standard topology (inherited as a subset of the weak* topology on the space 
of finite-valued stochastic processes) under which it is compact and the expected cost function 
$(\tilde{X}_0,\tilde{Y}_0) \mapsto \mathbb{E} c(\tilde{X}_0,\tilde{Y}_0)$ is continuous.
Thus, the minimum in (\ref{eqn:GOTCmin}) is achieved, and there exists an optimal transition coupling 
in (\ref{eqn:GOTCargmin}).  In general, there may be many solutions to the NetOTC problem. 
For example, if the cost function is constant, then all transition couplings are optimal.

While the objective function of the NetOTC problem involves only the first time point 
of the joint process $(\tilde{X},\tilde{Y})$, 
the restriction to transition couplings ensures that 
the optimal coupling performs well on average over multiple time points
(see Proposition \ref{prop:gotceq} below), and that it captures the dynamics of the individual chains.
In general, the minimizing value of $\E c(\tilde{X}_0, \tilde{Y}_0)$ will (strictly) 
decrease as one moves from transition couplings to general Markov couplings, 
from Markov couplings to stationary couplings, and from stationary couplings to general couplings 
\citep{ellis1976thedj, ellis1978distances, ellis1980conditions, ellis1980kamae, o2020optimal, o2021estimation}.

We note that the NetOTC problem is {\em not} equivalent to the problem of finding a one-step 
optimal transition coupling, which is considered in \citep{song2016measuring, zhang2000existence}.  
In the latter problem one
finds, for every $u \in U$ and $v \in V$, a coupling $(\tilde{X}_0, \tilde{Y}_0)$
of $X_0 \sim \bP(\cdot | u)$ and $Y_0 \sim \bQ(\cdot | v)$ minimizing $\E c(\tilde{X}_0, \tilde{Y}_0)$.   
A one-step optimal transition coupling does not necessarily exhibit good performance over multiple time steps,  
as it does not account for the global structure of the given networks.

\subsection{Cost Functions}\label{sec:cost_function}

In practice, the specification of a cost function depends on the goals of the network alignment or comparison problem.
The cost function is typically based on prior information about the vertex sets of the given networks, including vertex features 
and embeddings, if these are available. 
If $U = V$ we may use the 0-1 cost $c(u,v) = \In(u \neq v)$.
If the vertices of $G_1$ and $G_2$ are associated with features or attributes in a common, discrete set $\mathcal{S}$ then 
one may take $c(u,v) = \rho(\tilde{u}, \tilde{v})$ where $\rho$ is a cost function relating the 
elements of $\mathcal{S}$, and $\tilde{u}, \tilde{v} \in \mathcal{S}$ are the features associated with vertices $u$ and $v$, respectively.
If $\mathcal{S}$ is a finite set, the zero-one cost $c(u,v) = \In(\tilde{u} \neq \tilde{v} )$ is often a good choice.
If the vertices of $G_1$ and $G_2$ are embedded in a common Euclidean space $\real^d$ via embeddings 
$h_1: U \to \real^d$ and $h_2: V \to \real^d$, then it is natural to use an embedding-based cost such as
$c(u,v) = || h_1(u) - h_2(v) ||$ or $c(u,v) = || h_1(u) - h_2(v) ||^2$.
In cases where such prior maps are unavailable, one may consider costs defined in 
terms of intrinsic properties of the networks of interest, or embed the vertices in a 
Euclidean space a priori using methods such as Laplacian eigenmaps \citep{belkin2003laplacian}.

A cost function that is applicable in general is the degree-based cost: $c(u,v) = (\mbox{deg}(u) - \mbox{deg}(v))^2$. 
For undirected networks, $\mbox{deg}(u)$ is sum 
of weights of all edges adjacent to $u$.
For directed networks, one may use in-degree, out-degree, or a combination of these.
Unless otherwise specified, we use out-degree in this paper.
One may also use the standardized degree $\overline{d}(u) = \mbox{deg}(u)/ \sum_{u' \in U} \mbox{deg}(u')$ 
when comparing networks of significantly different sizes.
Extending this idea, one may employ cost $c(u,v)$ based on the degree distributions of a fixed local neighborhood of $u$ and $v$.

\subsection{Computation of NetOTCs}

\begin{algorithm*}[t]
\caption{Solving the NetOTC Problem}\label{alg:netotc}
\vskip.05in
\textbf{Input:} Networks $G_1 = (U, E_1, w_1)$ and $G_2= (V, E_2, w_2)$. Cost function $c(u,v)$. 
\begin{algorithmic}
\State \textbf{Step 1.} Compute the transition probabilities $P$ and $Q$ of the random walks associated with $G_1$ and $G_2$ according to \eqref{eq:net_to_trans} 
\vskip.05in
\State \textbf{Step 2.} Pass $P$ and $Q$ to the procedure of \cite{o2020optimal}, which yields the optimal cost $\rho$, as well as the
stationary distribution $\pi$ and transition kernel $\bR$ of an optimal transition coupling
\vskip.05in
\State \textbf{Step 3.} Calculate vertex alignment as in (\ref{eqn:vtxalign}) and edge alignment as in (\ref{eqn:edgealign})
.\end{algorithmic}
\vskip.05in
\textbf{Output:} NetOTC cost $\rho$, Vertex alignment $\pi_{\text{v}}$, Edge alignment $\pi_{\text{e}}$
\end{algorithm*}

A workflow for NetOTC is given in Workflow~\ref{alg:netotc}.
The NetOTC procedure does not rely on randomization, and has no free parameters: 
its output is fully determined by the given networks and the cost function $c$.
Finding an optimal transition coupling (OTC) of the random walks $X$ and $Y$ 
derived from the given networks is a non-convex, constrained optimization problem that is not amenable to standard
optimization routines.    
Instead, NetOTC uses the method of~\citep{o2020optimal}, in which the problem of finding an OTC 
is reframed as a Markov decision process (MDP)
with state space $\mathcal{S} = \cX \times \cY$, where the
admissible actions in state $s = (x,y)$ correspond to couplings $r_s$ of the transition distributions $P(x,\cdot)$ and $Q(y, \cdot)$.
The transition distribution of the MDP in state $s$ with action $r_s$ is simply given by $r_s$, while the reward function of the MDP is simply the negative of the cost function
$R(s, r_s) = - c(x,y)$, where $s = (x,y)$.
Reformulated in this way, the OTC problem corresponds to finding an optimal policy for the MDP, a problem to which policy 
iteration~\citep{howard1960dynamic} may be applied (with standard optimal transport solvers used in the policy update steps). 
The algorithm requires $\mathcal{O}((|U||V|)^3)$ operations per iteration.
In practice, it converges after fewer than 5 iterations.
\cite{o2020optimal} also describes a more efficient algorithm based on entropic regularization and Sinkhorn iterations.
When applied to NetOTC, the regularized algorithm requires $\mathcal{O}((|U| |V|)^2)$ operations per iteration 
(up to poly-logarithmic factors), which is nearly-linear in the dimension of the couplings under consideration, 
and in this sense comparable to the state-of-the-art for entropic OT algorithms \citep{peyre2019computational}.
Pseudocode and more details on the method can be found in Section 4 of \cite{o2020optimal}. 

In general, NetOTC problem may have multiple solutions.  The NetOTC algorithm is only guaranteed to 
return a single minimizer, which is not guaranteed to be irreducible.  
On the other hand, the entropically regularized problem has a unique minimizer, 
which is aperiodic and irreducible when $X$ and $Y$ are aperiodic and irreducible.
The current implementation of NetOTC can handle networks with up to 200 vertices.  Research on faster computation
of OTCs is currently ongoing.

\subsection{NetOTC Deliverables}

%
%

\noindent{\bf Difference measure for networks.}
The solution of the NetOTC problem yields the minimizing value $\rho(G_1, G_2)$ of the expected cost, 
and the associated optimal transition plan $(X^*,Y^*)$. 
The minimum cost $\rho(G_1, G_2)$ measures the difference between $G_1$ and $G_2$ and can be 
utilized for network comparison tasks. 
For undirected networks with the same vertex set, $\rho( \cdot, \, \cdot)$ is a metric if the cost function is a metric 
(see Proposition \ref{prop:gotcismetric}).

\vskip.1in

\noindent{\bf Vertex alignment.}
The optimal transport plan $(X^*, Y^*)$ itself provides soft, probabilistic alignments of the vertices and 
edges of $G_1$ and $G_2$ based on the joint distribution of pairs in the coupled chain. 
The vertex alignment $\pi_{\text{v}}$ produced by NetOTC is derived from the stationary distribution
of the optimal transport plan $(X^*, Y^*)$, which is the distribution of the pair $(X_0^*, Y_0^*)$.  
The vertex alignment is defined by
\begin{align}
\label{eqn:vtxalign}
\pi_{\text{v}}(u,v) = \mathbb{P}((X_0^*, Y_0^*)= (u,v)).
\end{align}
One may define probabilistic vertex alignments in a similar manner for other OT-based network comparison methods (see Section \ref{sec:graph_isomorphism} and \ref{sec:sbm_alignment}).

\vskip.1in

\noindent{\bf Edge alignment.}
A unique feature of NetOTC is that the optimal transport plan naturally yields a probabilistic alignment of 
the edges of the given networks.  The edge alignment is obtained from the first two pairs $\{ (X_0^* , Y_0^*), (X_1^*, Y_1^*) \}$
in the optimal transport plan.  For $u, u' \in U$ and $v, v' \in V$ the edge alignment is defined by
\begin{align}
\label{eqn:edgealign}
    \pi_{\text{e}}((u,u'),(v,v')) = \mathbb{P}((X_0^* , Y_0^*) = (u,v), (X_1^*, Y_1^*)= (u',v')).
\end{align}
It is straightforward to show (see Proposition \ref{prop:edge_preservation}) that vertex pairs aligned with positive probability
must be adjacent in their respective networks.  
In other words, NetOTC aligns only existing edges; it does not create new ones.
By contrast, most alignment methods in the literature have as their primary focus the matching of vertices, 
with edges functioning primarily as a means of evaluating matchings.  Alignments arising in this way can map adjacent vertices 
in $G_1$ to non-adjacent vertices in $G_2$, or vice versa.

\vskip.3in

\section{Theoretical Properties of NetOTC}
\label{sec:theory}

In this section, we explore some theoretical properties of NetOTC, beginning 
with the edge-alignment property and several results concerning the behavior of NetOTC under average cost criteria.
For undirected networks with a common vertex set, we establish that the NetOTC cost has the properties of a metric 
when the cost $c$ does, and we investigate the sensitivity of NetOTC to local information arising from degree and 
weight functions.
We then define a notion of \textit{network factor} that captures the idea of one network being
``folded'' or ``compressed'' to produce another.  Although the literature contains several definitions of 
network factors or factor networks, the definition given here appears to be new.  We establish a 
close connection between factors and transition couplings, and then we
present two results describing the behavior of NetOTC in the presence of this type of network factor structure.
All proofs of the results in this section may be found in Section \ref{sec:proofs}.

\subsection{NetOTC Edge Alignment}

The NetOTC edge alignment function $\pi_{\text{e}}$ respects edges, in the sense that vertex pairs aligned with 
positive probability must be adjacent in the given networks.

\begin{restatable}[]{prop}{edgepreservation}
\label{prop:edge_preservation}
Let $\pi_{\text{e}}$ be the NetOTC edge alignment of networks 
$G_1 = (U, E_1, w_1)$ and $G_2 = (V, E_2, w_2)$ based on the optimal transport plan $(X^*, Y^*)$.
If $\pi_{\text{e}}((u,u'),(v,v')) > 0$ then $(u,u') \in E_1$ and $(v,v') \in E_2$. 
\end{restatable}

\subsection{NetOTC and Multistep Cost}\label{sec:multistep_cost}

The NetOTC cost $\rho(G_1,G_2)$ is the expected cost
$\E c(\tilde{X}_0,\tilde{Y}_0)$ at the initial state of the optimal transport plan.   
Stationarity ensures that the NetOTC problem is equivalent to minimizing the 
long-term average cost over the set of transition couplings.


\begin{defn}
For an infinite sequence $x_0, x_1, \ldots$ and integers $0 \leq i \leq j$ let $x_i^j = (x_i, \dots, x_j)$.  
For each $k \geq 1$ define the $k$-step average cost 
$c_k(x_0^{k-1}, y_0^{k-1}) = k^{-1} \sum_{j=0}^{k-1} c(x_j, y_j)$ and the limiting average cost
$\overline{c}(x,y) = \limsup_{k \to \infty} c_k(x_0^{k-1}, y_0^{k-1})$.
\end{defn}

\begin{restatable}[]{prop}{gotceq}
\label{prop:gotceq}
Let $G_1$ and $G_2$ be networks with associated random walks $X$ and $Y$.
Then
\[
\rho(G_1, G_2) 
\ = \
\min_{(\tilde{X}, \tilde{Y}) \in \tc(X,Y)} \E \overline{c}(\tilde{X}, \tilde{Y}),
\]
and the optimal transport plans minimizing $\E \overline{c}(\tilde{X}, \tilde{Y})$ coincide with those
minimizing $\E c(\tilde{X}_0, \tilde{Y}_0)$.
\end{restatable}

The long-term behavior of the random walks $X$ and $Y$ encodes information about 
the global structure of the networks $G_1$ and $G_2$, respectively. 
The next result shows that NetOTC also captures local information arising from the finite time behavior 
of the walks $X$ and $Y$. 
For example, if $G_1$ and $G_2$ are distinguishable based on optimal transport of their $k$-step random walks, 
then they are distinguishable by NetOTC. 

\begin{restatable}[]{prop}{gotclowerbound}
\label{prop:gotclowerbound}
Let $G_1$ and $G_2$ be networks with associated random walks $X$ and $Y$.
For each $k \geq 1$
\[
\rho(G_1, G_2) 
\ \geq \ 
\min \ \E \, c_k(\tilde{X}_0^{k-1}, \tilde{Y}_0^{k-1}) ,
\]
where the minimum is over the family of all couplings of $X_0^{k-1}$ and $Y_0^{k-1}$.
\end{restatable}

\subsection{Undirected Networks with a Common Vertex Set}

In this section, we consider undirected networks $G_1 = (U, E_1, w_1)$ and $G_2 = (U, E_2, w_2)$ with the same vertex set,
but potentially different edge sets and weight functions.  We assume throughout that the networks are connected.
We begin by defining a natural equivalence relation on such networks.

\begin{defn}
Undirected networks $G_1$ and $G_2$ are
\emph{equivalent}, denoted by $G_1 \sim G_2$, if they have the same vertex $U$, the same edge set $E$, and there exists a constant 
$C > 0$ such that $w_1(u, u') = C \, w_2(u, u')$ for every $u, u' \in U$.
\end{defn}

The following result relates the equivalence of networks to their random walks. 

\begin{restatable}[]{prop}{graphequiv}
\label{prop:equiv}
Connected undirected networks $G_1$ and $G_2$ are equivalent 
if and only if their respective random walks are identical.
\end{restatable}

Whatever the underlying cost $c$, the cost $\rho(G_1,G_2)$ and optimal transport plan arising from 
NetOTC depends only on the equivalence classes of the networks $G_1$ and $G_2$.  In particular,
NetOTC is invariant under (positive) scaling of weight functions.  
When the underlying cost $c$ is a metric on $U$, the NetOTC cost is a metric on these equivalence classes.

\begin{restatable}[]{prop}{gotcismetric}
\label{prop:gotcismetric}
If the cost function $c: U \times U \rightarrow \mathbb{R}_+$ satisfies the properties of a metric on $U$,
then $\rho$ is a metric on the equivalence classes of undirected networks with vertex set contained in $U$ defined by $\sim$.
\end{restatable}

We now investigate the sensitivity of NetOTC to differences between the degree and weight functions
of the given networks.

\begin{defn}
The degree function of an undirected network $G = (U, E, w)$ is given by 
$d(u) = \sum_{u' \in U} w(u,u')$ for $u \in U$.  Let $D = \sum_{u \in U} d(u)$ denote the total degree of $G$.
\end{defn}

The next proposition strengthens the general result of Proposition \ref{prop:gotclowerbound}.

\begin{restatable}[]{prop}{lowerbound}
\label{prop:locallowerbound}
Let $G_1$ and $G_2$ be undirected networks with the same vertex set.  
Let $d_1(u)$ and $d_2(u)$ be the degree functions of $G_1$ and $G_2$, respectively, 
and assume that each network has a total degree of $D$. 
Then under the zero-one cost $c(u, u') = \In(u \neq u')$, 
\begin{itemize}
\item $\rho(G_1, G_2) \geq \frac{1}{2D} \sum\limits_{u \in U} |d_1(u) - d_2(u)|$
\item $\rho(G_1, G_2) \geq \frac{1}{4D} \sum\limits_{u, u' \in U} |w_1(u,u') - w_2(u,u')|$ 
\end{itemize}
\end{restatable}

\begin{rem}
    Proposition~\ref{prop:locallowerbound} can be readily extended to cases where $G_1$ and $G_2$ have different total degrees $D_1$ and $D_2$, respectively.
    Nevertheless, when comparing two undirected networks using NetOTC, we may assume that they share equal total degrees. As indicated in Proposition~\ref{prop:equiv}, for a connected undirected network $G_2 = (U, E_2, w_2)$ with  total degree  $D_2$, an equivalent graph $G_2' = \left(U, E_2, \frac{D_1}{D_2} w_2\right)$  with total degree $D_1$ can be constructed.
\end{rem}

\subsection{Deterministic Transition Couplings and Factor Maps}
\label{sec:factors}

The graph theory and network literature contain several definitions of ``network factor'' and ``factor network''. 
A network factor of $G$ is often defined to be any spanning subnetwork of $G$, while   
the term factor network is used in the context of message passing algorithms and error-correcting codes to refer to a 
bipartite network that captures the factorization of a function or a probability distribution.
Here we define a notion of network factor that appears to be different than existing definitions in the literature,
see for example the survey \citep{plummer2007graph}.
We show that there is a close connection between factors and transition couplings, and we use this to rigorously
study the behavior of NetOTC when factor structure is present.  Our results establish a close link between
NetOTC and factors, behavior that distinguishes NetOTC from other comparison and alignment methods.

\begin{defn}
\label{def:factor}
Let $G_1 = (U,E_1,w_1)$ and $G_2 = (V,E_2,w_2)$ be strongly connected, weighted directed networks with out-degree functions $d_1$ and $d_2$, respectively. 
A map $f : U \to V$ is  a \emph{factor map} if 
for all $v,v' \in V$ and $u \in f^{-1}(v)$, 
\begin{equation} \label{eq:factor_def}
\sum_{u' \in f^{-1}(v')} w_1(u,u') \ = \ \frac{d_1(u)}{d_2(v)} \, w_2(v,v').
\end{equation}
In this case, we will say that $G_2$ is a \emph{factor} of $G_1$, and that $G_1$ is an \emph{extension} of $G_2$.
\end{defn}

\begin{ex}\label{ex:factor1}
Consider the networks $G_1$ and $G_2$ drawn in Figure \ref{fig:factor_ex1} with vertices embedded in $\bbR^2$.
$G_2$ is a factor of $G_1$ with respect to the map $f$ that takes $(-1, 1)$ and $(-1, -1)$ to $(-1, 0)$, $(0, 0)$ to $(0, 0)$, and $(1, 0)$ and $(1, -1)$ to $(1, 0)$.
\end{ex}

The definition of factor arises naturally from the random walk perspective.  If $G_1$ and $G_2$ have associated
random walks $X$ and $Y$, and $G_2$ is a factor of $G_1$ under the map $f$, then the process 
$f(X) := f(X_0), f(X_1), \ldots$ is equal in distribution to $Y$, see 
Theorem \ref{thm:factortc} below for more details. 
Factors have been well studied in ergodic theory and symbolic dynamics.
The existence of a factor map (in the sense above) ensures that $Y$ is a stationary coding of $X$, which is a special case 
of a factor relationship in ergodic theory. 
Moreover, if $G_2$ is a factor of $G_1$, then the subshift of finite type (SFT) consisting of all bi-infinite walks on 
$G_2$ is a topological factor of the SFT associated with $G_1$ given by a $1$-block code
(see \cite{lind1995anintroductiontosymbolicdynamics} for detailed definitions). 
Our definition of factor has points of contact with compressed representations of weighted networks, 
explored in \cite{toivonen2011compression}, but in general the relationship \eqref{eq:factor_def} need not hold for compressed representations.

\begin{figure}[t]
\centering
\begin{subfigure}{0.45\textwidth}
\centering
\scalebox{0.5}{
\begin{tikzpicture}[->,>=stealth',shorten >=1pt,auto,node distance=3cm,
                    thick,main node/.style={circle,draw,font=\sffamily\Large\bfseries}]

  \node[main node] (3) {(0,0)};
  \node[main node] (1) [above of=3, left of=3]{(-1,1)};
  \node[main node] (2) [left of=3, below of=3] {(-1,-1)};
  \node[main node] (4) [right of=3] {(1,0)};
  \node[main node] (5) [below of=4] {(1,-1)};

  \path[every node/.style={font=\sffamily\small}]
    (1) edge node[left] {1} (2)
    (2) edge (1)
    (2) edge [bend right] node[below] {1} (3)
    (3) edge [left] node[above, left] {1} (2)
    (3) edge [right] node[below, left] {1} (1)
    (1) edge [bend left] node[above] {1} (3)
    (3) edge [left] node[below] {1} (4)
    (4) edge [bend right] node[above] {1} (3)
    (3) edge [right] node[above, right] {1} (5)
    (5) edge [bend left] node[below] {1} (3)
    (4) edge node[right] {1} (5)
    (5) edge (4);
\end{tikzpicture}
}
\caption{$G_1$}
\label{fig:factor_ex1a}
\end{subfigure}
\begin{subfigure}{0.45\textwidth}
\centering
\scalebox{0.5}{
\begin{tikzpicture}[->,>=stealth',shorten >=1pt,auto,node distance=3cm,
                    thick,main node/.style={circle,draw,font=\sffamily\Large\bfseries}]

  \node[main node] (3) {(0,0)};
  \node[main node] (1) [left of=3]{(-1,0)};
  \node[main node] (4) [right of=3] {(1,0)};

  \path[every node/.style={font=\sffamily\small}]
    (3) edge node[below] {2} (1)
    (1) edge (3)
    (3) edge node[below] {2} (4)
    (4) edge (3)
    (1) edge [loop left] node[left] {2} (1)
    (4) edge [loop right] node[right] {2} (4);
\end{tikzpicture}
}\vspace{5mm}
\caption{$G_2$}
\label{fig:factor_ex1b}
\end{subfigure}
\caption{An example of two networks related by a factor map. Here $G_2$ is a factor of $G_1$ via 
the map that collapses vertices along vertical lines.}
\label{fig:factor_ex1}
\end{figure}
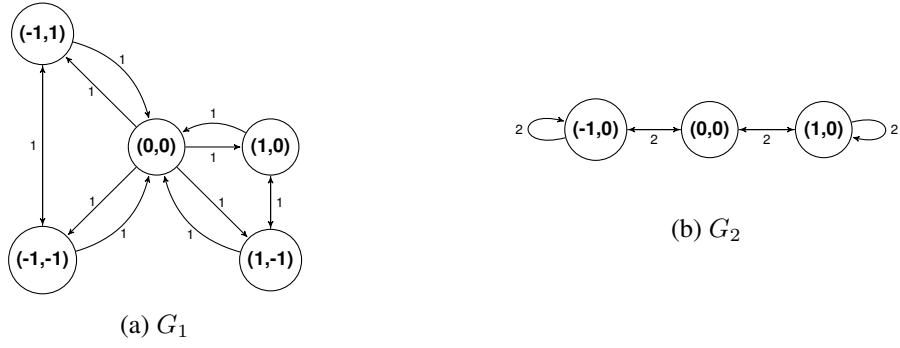

The definition of factor formalizes the idea that $G_2$ (the factor) is a collapsed or 
compressed version of $G_1$ (the extension).
The factor map $f$ associates the vertices in $G_2$ with a partition of the vertices in $G_1$. 
Condition \eqref{eq:factor_def} ensures that the partitioning of the vertices is consistent 
with the transition probabilities of the random walk on $G_1$. 
If $\bP$ and $\bQ$ are the transition kernels for $G_1$ and $G_2$, then Condition \eqref{eq:factor_def} is equivalent to the statement that 
for all $v,v' \in V$ and $u \in f^{-1}(v)$, 
\begin{equation} \label{eq:factor_trans_probs}
\sum_{u' \in f^{-1}(v')} \bP( u' \mid u) = \bQ( v' \mid v).
\end{equation}
This is also equivalent to the condition $\mathbb{P}( f(X_1) = v' \mid X_0 = u) = \mathbb{P}( Y_1 = v' \mid Y_0 = v)$,
where $X$ and $Y$ are the random walks associated with $G_1$ and $G_2$.
Since $\bP$ and $\bQ$ are irreducible, Condition \eqref{eq:factor_def} implies that the (unique) stationary distributions 
$\bp$ on $G_1$ and $\bq$ on $G_2$ are such that
for all $v \in V$, 
\begin{equation} \label{eq:stat_dist_factor_def}
\sum_{u \in f^{-1}(v)} \bp(u) = \bq(v),
\end{equation}
which is equivalent to $f(X_0) \eqd Y_0$. 
The factor relationship can also be expressed in matrix form. If $f$ is a factor map from $G_1$ to $G_2$, 
then \eqref{eq:factor_def} is equivalent to the condition $\bP F = F \bQ$ where
$F \in \mathbb{R}^{U \times V}$ is defined by $F(u,v) = 1$ if $f(u) = v$ and $F(u,v) = 0$ otherwise. 
Furthermore equation \eqref{eq:stat_dist_factor_def} is equivalent to $\bp F = \bq$.

\vskip.05in

The next proposition establishes a close connection between transition couplings and factor maps. 
Let $G_1$ and $G_2$ be strongly connected, weighted directed networks with associated random walks $X$ and $Y$.  
Note that any stationary Markov coupling $(\tilde{X},\tilde{Y})$ of $X$ and $Y$ corresponds to a weighted, directed 
network $H$ with vertex set $W \subset U \times V$.
Let $\pi_U : U \times V \to U$ and $\pi_V : U \times V \to V$ be the standard projections onto the first and second coordinates, respectively.

\begin{restatable}[]{prop}{TCinTermsFactors}
If $(\tilde{X},\tilde{Y})$ is a stationary Markov coupling corresponding to a strongly connected network $H$ with vertex set $W$, 
then $(\tilde{X},\tilde{Y})$ is a transition coupling of $X$ and $Y$ if and only if the restriction of $\pi_U$ to $W$ is a factor map from 
$H$ to $G_1$ and the restriction of $\pi_V$ to $W$ is a factor map from $H$ to $G_2$. 
\end{restatable}

We next investigate connections between factors and deterministic transition couplings.
 
\begin{defn} \label{def:DetTC}
Suppose $G_1 = (U, E_1, w_1)$ and $G_2 = (V, E_2, w_2)$ are two strongly connected, weighted, directed networks with associated Markov chains $X$ and $Y$, respectively. A transition coupling $(\tilde{X},\tilde{Y})$ is said to be deterministic from $X$ to $Y$ if for each $u$ in $U$ there exists $v \in V$ such that $\mathbb{P}( \tilde{Y}_0 = v \mid \tilde{X}_0 = u) = 1$. 
\end{defn}
In optimal transport theory, deterministic couplings are associated with the so-called Monge problem, see \cite{Villani2008OptimalTO} for more context and discussion.
A deterministic coupling $(\tilde{X}, \tilde{Y})$ from $X$ to $Y$ is associated with a map $f : U \to V$, where $f(u)$ is the 
(necessarily unique) element $v \in V$ for which $\mathbb{P}( \tilde{Y}_0 = v \mid \tilde{X}_0 = u) = 1$.
In particular, $(\tilde{X},\tilde{Y}) \eqd (X,f(X))$.  
Moreover, as $G_2$ is strongly connected, $\mathbb{P}(Y_0 = v) > 0$ for each $v \in V$, and
the edge-alignment property (Proposition \ref{prop:edge_preservation}) ensures that $f$ is a surjective 
graph homomorphism from $G_1$ to $G_2$.

\begin{restatable}[]{thm}{factortc}
\label{thm:factortc}
Suppose $G_1$ and $G_2$ are strongly connected, weighted directed networks with associated random walks $X$ and $Y$, respectively. 
\begin{enumerate}
\item If $G_2$ is a factor of $G_1$ with factor map $f$, then
$Y \eqd f(X)$, and 
$(X,f(X))$ is a deterministic transition coupling from $X$ to $Y$.
\item If $(\tilde{X},\tilde{Y})$ is a deterministic transition coupling from $X$ to $Y$, then the induced map $f : U \to V$ is a factor map from $G_1$ to $G_2$.
\end{enumerate}
\end{restatable}

When $G_2$ is a factor of $G_1$ under $f$, Theorem \ref{thm:factortc} ensures that $(X,f(X))$ is a transition coupling of
their random walks.  If the cost function $c$ is such that $c(u,v)$ is minimized by $v = f(u)$ then, as the next result shows, 
this coupling is also optimal, and there is a deterministic solution to the NetOTC problem.

\begin{defn}
Let $f$ be a factor map from $G_1$ to $G_2$.  A cost function $c$ is {\em compatible} with $f$ if
$c(u,f(u)) \leq c(u,v)$ for each $u \in U$ and $v \in V$.
\end{defn}

One may verify that the cost compatibility condition is satisfied in Example \ref{ex:factor1} under an Euclidean metric cost.

\begin{restatable}[]{cor}{factor}
\label{cor:factor} Suppose $G_1$ and $G_2$ are strongly connected, weighted directed networks and $f$ is a factor map from $G_1$ to $G_2$. 
If $c$ is compatible with $f$ then $(X,f(X))$ is an OTC of $X$ and $Y$.
\end{restatable}

An example illustrating Corollary \ref{cor:factor} is given in Figure \ref{fig:factor}.
Corollary \ref{cor:factor} provides some insight into the structure of the NetOTC problem.
If $G_1$ and $G_2$ are related by a factor map $f: U \rightarrow V$, then $G_2$ is essentially a compressed version of the network $G_1$.
Corollary \ref{cor:factor} ensures that an optimal coupling of the random walks on $G_1$ and $G_2$ is obtained by running the random
walk on $G_1$ and mapping every state $u \in U$ in this chain to the corresponding state $f(u) \in V$.
In practice, the conclusion of Corollary \ref{cor:factor} approximately holds when the factor condition \eqref{eq:factor_def} 
approximately holds; the results of experiments involving exact and approximate factors are given in Section \ref{sec:graph_factors}. In such situations, NetOTC can be used to identify (approximate) factor maps between $G_1$ and $G_2$.

\begin{figure}[t]
\begin{center}
\begin{subfigure}{0.31\linewidth}
\includegraphics[width=\linewidth, trim=3.2cm 2.5cm 2.5cm 1.5cm, clip]{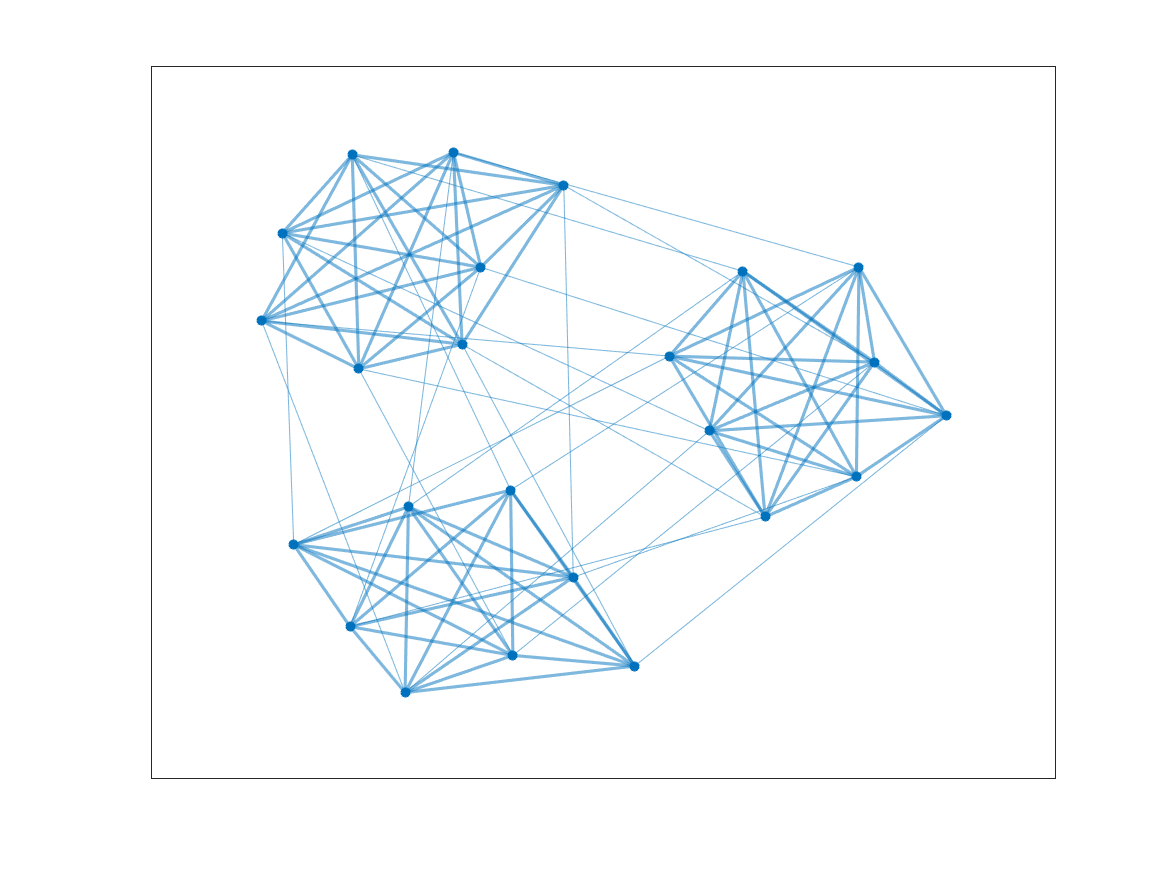}
\caption{$G_1$}
\end{subfigure}
\begin{subfigure}{0.31\linewidth}
\includegraphics[width=\linewidth, trim=3.2cm 2.5cm 2.5cm 1.5cm, clip]{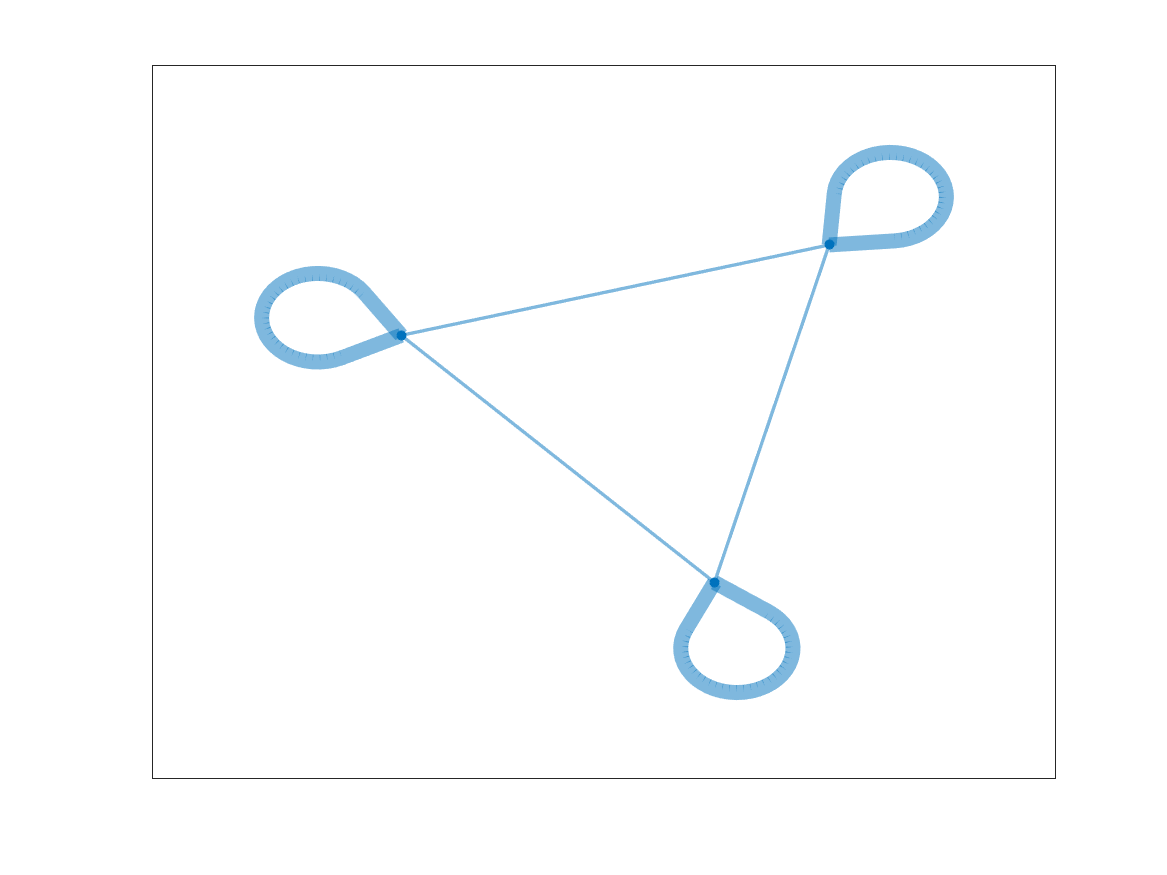}
\caption{$G_2$}
\end{subfigure}
\begin{subfigure}{0.34\linewidth}
\centering
\captionsetup{justification=centering}
\includegraphics[width=\linewidth, trim=1.4cm 0.5cm 0.2cm 0.5cm, clip]{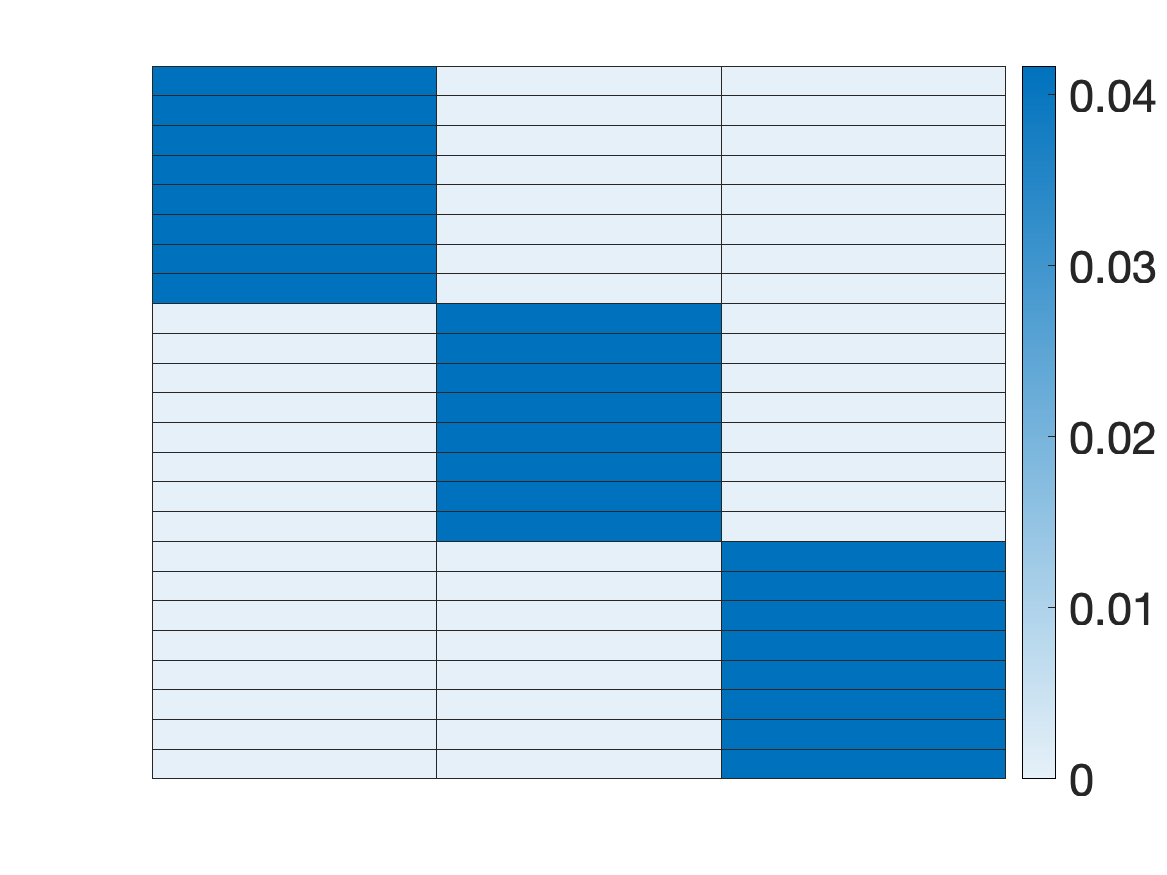}
\caption{Vertex alignment of $G_1$ and $G_2$ from NetOTC.}
\label{fig:factor_gotc}
\end{subfigure}
\caption{An illustration of the relationship between factors and the NetOTC problem.
Under the conditions described in Corollary \ref{cor:factor}, the NetOTC problem aligns vertices according to the factor map relating the two networks.
In this example, $G_2$ is a factor of $G_1$. 
Figure \ref{fig:factor_gotc} illustrates the NetOTC vertex alignment, which is supported on pairs of the form $(u, f(u))$.}
\label{fig:factor}
\end{center}
\end{figure}

Under the conditions of Corollary \ref{cor:factor}, the NetOTC cost and associated vertex and edge alignments will have a special form. 
In particular, the NetOTC cost will satisfy $\rho(G_1,G_2) = \sum_{u \in U} c(u, f(u)) \bp(u)$
where $\bp$ is the stationary distribution of the random walk on $G_1$. 
Furthermore, $\pi_v(u,v) = \mathbb{I}(f(u) = v)$ and
$\pi_e((u,u'),(v,v')) = \mathbb{I}((f(u),f(u')) = (v,v'))$.
Note that, while the deterministic coupling appears as a solution of the NetOTC problem, the NetOTC algorithm itself 
makes no reference to, and does not require prior information about, the factor map $f$.
The next result provides further information about how NetOTC behaves in the presence of factor maps.

\begin{restatable}[]{thm}{twofactor}
\label{thm:two_factor}
Let $G_1$, $G_2$, $H_1$, and $H_2$ be networks with vertex sets $U$, $V$, $A$, and $B$, and associated 
Markov chains $X$, $Y$, $W$, and $Z$, respectively. 
Suppose that $f : U \to A$ and $g : V \to B$ are factor maps from $G_1$ to $H_1$ and $G_2$ to $H_2$, 
and that there are cost functions $c_{ext} : U \times V \to \mathbb{R}_+$ and $c : A \times B \to \mathbb{R}_+$ such that $c_{ext}(u, v) = c( f(u) , g(v))$. 
\begin{enumerate}
\item If $(\tilde{X}, \tilde{Y})$ is an optimal transition coupling of $X$ and $Y$ with respect to $c_{ext}$, then $(f(\tilde{X}), g(\tilde{Y}) )$ is an optimal transition coupling of $W$ and $Z$ with respect to $c$.

\vspace{2mm}

\item If $(\tilde{W},\tilde{Z})$ is an optimal transition coupling of $W$ and $Z$ with respect to $c$, then there exists an optimal transition coupling $(\tilde{X},\tilde{Y})$ of $X$ and $Y$ with respect to $c_{ext}$ such that $(f(\tilde{X}), f(\tilde{Y})) \eqd (\tilde{W},\tilde{Z})$.
\end{enumerate}
\end{restatable}

A simple illustration of Theorem \ref{thm:two_factor} is given in Figure \ref{fig:two_factor}.
For compatible cost functions, the theorem ensures that an optimal transport plan for the extensions $G_1$ and $G_2$
can be transferred through the maps $f$ and $g$ to an optimal transport plan for the factors $H_1$ and $H_2$; moreover,
every optimal transport plan for the factors can be obtained in this way. 
Thus NetOTC respects factor structure whenever factor structure is present: the NetOTC alignment of the extensions is consistent with the NetOTC alignment of the factors.
This is a fundamental property of the NetOTC procedure, in the sense that the operation of NetOTC on the 
extensions $G_1$ and $G_2$ makes no reference to, and requires no knowledge of,
the factor maps $f$ and $g$ or the factors $H_1$ and $H_2$. From a practical point of view, if one has access to factor graphs $H_1$ and $H_2$, then one could save computational expense by aligning the smaller factor graphs, and Theorem~\ref{thm:two_factor} guarantees that the result would be consistent with the alignment of the larger graphs $G_1$ and $G_2$.

\begin{figure}[t]
\centering
\includegraphics[width=0.95\textwidth, trim=0cm 0cm 0cm 0cm, clip]{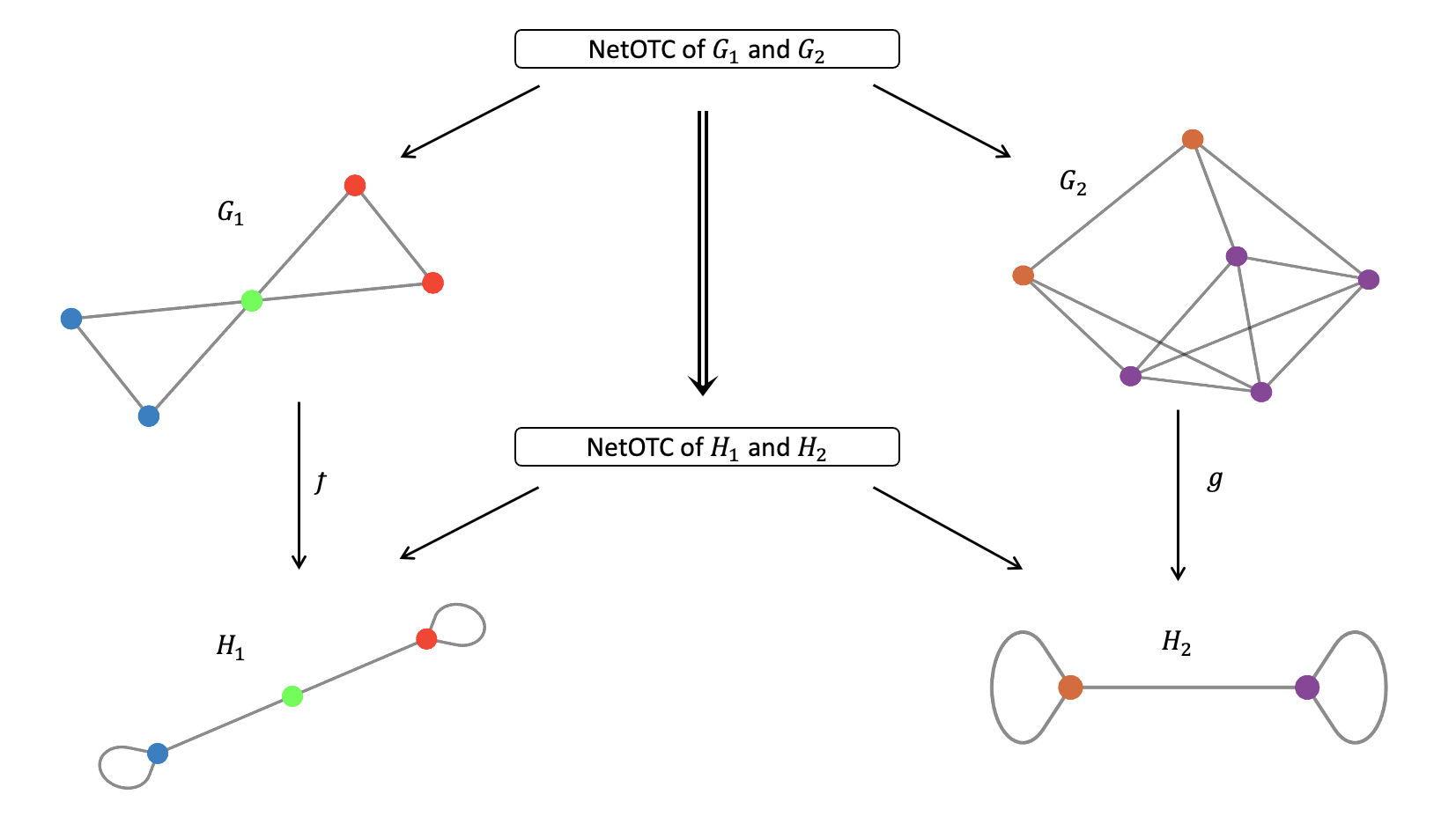}
\caption{An illustration of Theorem \ref{thm:two_factor}. If $G_1$ and $H_1$, and $G_2$ and $H_2$ are related by factor maps $f$ and $g$, respectively, the NetOTC of $H_1$ and $H_2$ can be naturally induced by the NetOTC of $G_1$ and $G_2$ using the maps $f$ and $g$.}
\label{fig:two_factor}
\end{figure}

\vskip.3in

\section{Example and Experiments}
\label{sec:experiments}
In this section, we illustrate the properties of NetOTC through an example 
and several numerical experiments.  The latter include
network classification, network isomorphism, alignment of stochastic block models, and network factors.
Complete experimental details may be found in Appendix \ref{app:exp_details}.
In the example and experiments, we compare NetOTC to several existing, optimal transport-based approaches to
network alignment: marginal optimal transport (OT), Gromov-Wasserstein (GW) \citep{peyre2016gromov}, Fused Gromov-Wasserstein (FGW) \citep{titouan2019optimal,vayer2020fused}, and Coordinated Optimal Transport (COPT) \citep{dong2020copt}.
Here marginal optimal transport refers to the optimal coupling of the stationary distributions of the random 
walks on the given networks. 
When applying the FGW method, following \cite{titouan2019optimal,vayer2020fused}, we use a uniform distribution 
on the vertices of each network.
Code for reproducing the example and experiments may be found at \url{https://github.com/austinyi/NetOTC}.

\subsection{Edge Awareness Example}
\label{sec:circle_example}

We begin with a toy example to demonstrate the edge preservation property of NetOTC (see Proposition \ref{prop:edge_preservation}): 
network $G_1$ is an octagon network, network $G_2$ is a copy of $G_1$ with one edge on the right removed, and 
network $G_3$ is topologically identical to $G_2$, but its vertices are located in the left half plane.  
See Figure \ref{fig:circle_example} below.

\begin{figure}[ht]
\centering
\begin{subfigure}{0.3\textwidth}
\includegraphics[width=\textwidth, trim=2.8cm 1cm 2.8cm 1cm, clip]{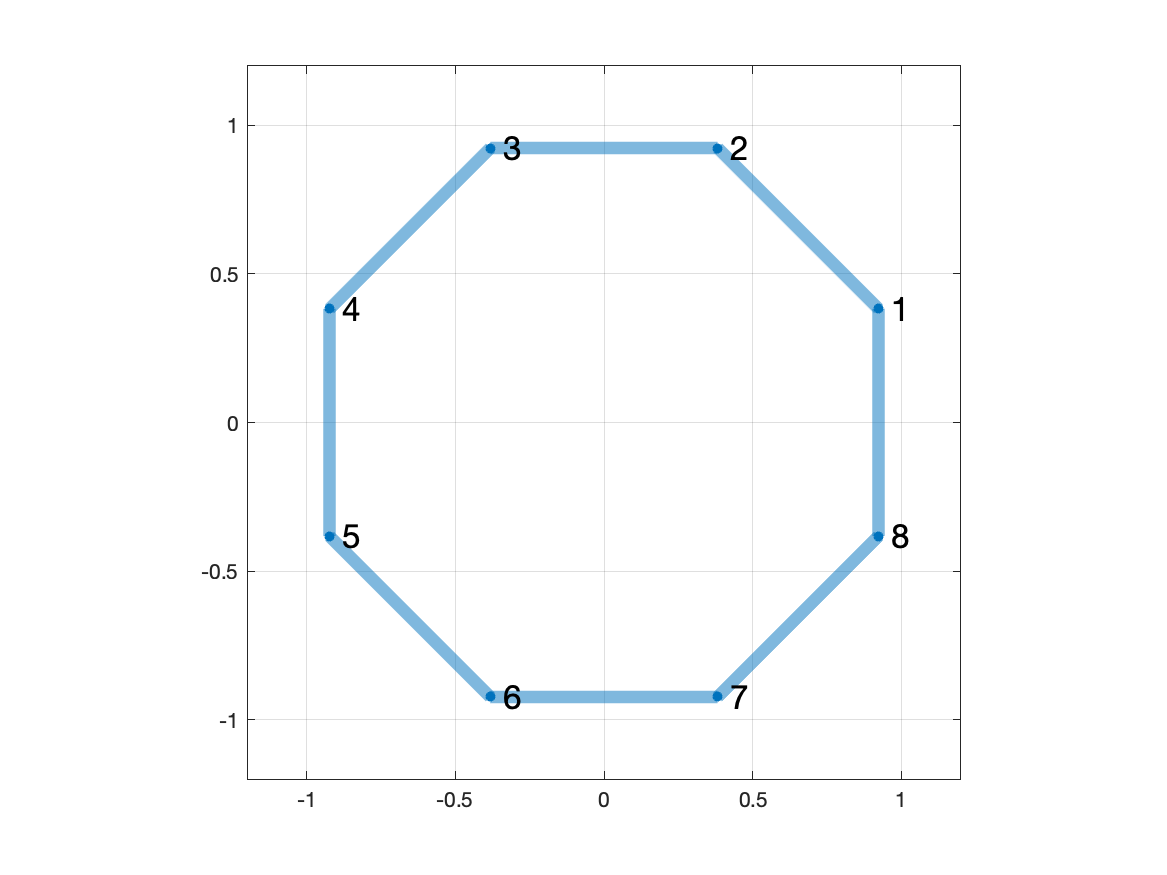}
\caption*{$G_1$}
\end{subfigure}
\begin{subfigure}{0.3\textwidth}
\includegraphics[width=\textwidth, trim=2.8cm 1cm 2.8cm 1cm, clip]{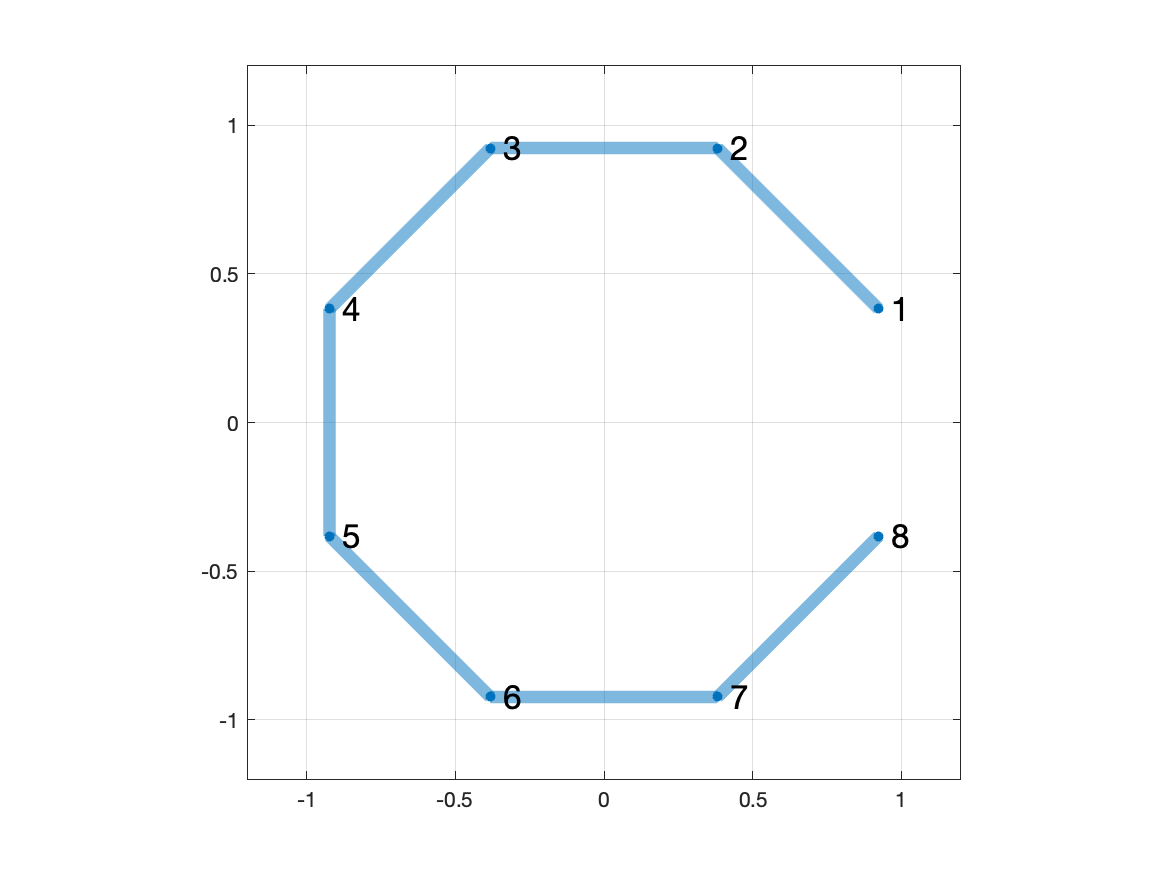}
\caption*{$G_2$}
\end{subfigure}
\begin{subfigure}{0.3\textwidth}
\includegraphics[width=\textwidth, trim=2.8cm 1cm 2.8cm 1cm, clip]{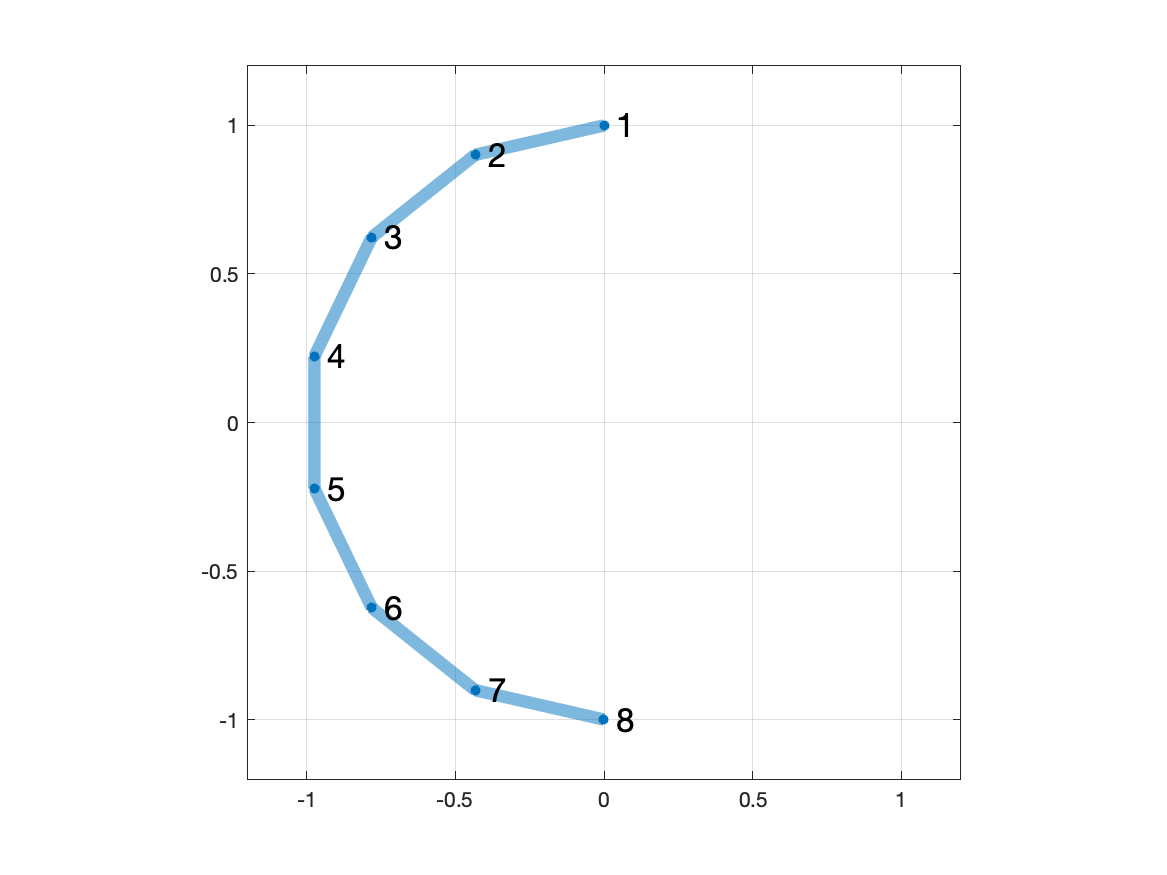}
\caption*{$G_3$}
\end{subfigure}
\caption{Three networks in which all vertices are located on the unit circle in $\mathbb{R}^2$.  $G_1$ is an octagon network. $G_2$ is obtained by removing an edge $G_1$. In $G_3$, the vertices are uniformly distributed in the left semicircle. }
\label{fig:circle_example}
\end{figure}

Table \ref{table:circle_example} shows the ratio of the costs obtained when comparing different pairs of networks
under a cost function equal to the squared Euclidean distance between vertex positions.  
We considered the methods NetOTC, OT, and FGW, as they allow the use of the Euclidean cost function.
The ratio of the cost between $G_2$ and $G_1$ and the cost between $G_2$ and $G_3$
varies greatly between the methods. 
Observe that NetOTC is the only method with a ratio that exceeds 1, that is, 
NetOTC finds $G_2$ to be closer to $G_3$ than $G_1$. 
This example illustrates that NetOTC is sensitive to topological differences between the given networks,
and in particular that topological similarity can dominate differences in the vertex costs.

\begin{table}[ht]
\small
\centering
\begin{tabular}{c c c c}
\hline 
Algorithm & $G_2$ vs. $G_1$ & $G_2$ vs. $G_3$ & Ratio\\ 
\hline
\hline
NetOTC &  0.5714 &  0.4464 & 1.28 \\
\hline 
OT & 0.2857 & 0.4464 & 0.64 \\ 
\hline 
FGW & 0.0313 & 0.2725 & 0.11 \\
\hline 
\end{tabular}
\caption{Comparison of OT-based costs between networks in Figure \ref{fig:circle_example}.}
\label{table:circle_example}
\end{table}

\subsection{Network Classification}\label{sec:network_classification}

In our next experiment, we examine the utility of NetOTC for network classification tasks.
We consider a selection of benchmark network datasets from \cite{KKMMN2016}.  Each dataset contains a 
collection of networks with discrete vertex attributes and class labels.
We considered the datasets AIDS \citep{riesen2008iam}, BZR \citep{sutherland2003spline}, 
Cuneiform \citep{kriege2018recognizing}, MCF-7 \citep{Yan2008MiningS}, MOLT-4 \citep{Yan2008MiningS}, 
MUTAG \citep{debnath1991structure}, and Yeast \citep{Yan2008MiningS}.
For each dataset, we employed an attribute-based cost function, where $c(u,v) =  0$ if vertices $u$ and $v$ have the same attribute and $c(u,v) = 1$ otherwise. 
For each OT-based comparison method, we constructed a 5-nearest neighbor classifier using 
a random training set containing 80\% of the available networks and used this classifier to predict
the labels of the remaining networks.

Table \ref{table:graph_class_acc} shows the average classification accuracy over 5 random samplings of the training and test sets for each comparison method.
As the table demonstrates, NetOTC is competitive with other network OT based methods on the classification tasks, outperforming other methods in several cases,
without the need for tuning or specification of free parameters.

\begin{table}[ht]
\centering
\small
\begin{tabular}{c c c c c c c c c c}
\hline 
Algorithm & AIDS  & BZR  & Cuneiform  & MCF-7  & MOLT-4 & MUTAG & Yeast\\ 
\hline
\hline
NetOTC & $88.0 \pm 4.9$ & \textbf{84.8} $\pm$ \textbf{6.6} & 73.2 $\pm$ 7.8 & $92.8 \pm 4.2$  & \textbf{92.0} $\pm$ \textbf{2.0} & \textbf{85.4} $\pm$ \textbf{7.1} & $90.8 \pm 6.4$\\ 
\hline 
OT & $84.4 \pm 6.1$ & $76.4 \pm 4.6$& $71.3 \pm 7.7$ & \textbf{93.6} $\pm$ \textbf{3.3} & \textbf{92.0} $\pm$ \textbf{2.0} & $63.2 \pm 7.3$ & $91.2 \pm 7.0$ \\ 
\hline
GW & $98.8 \pm 1.8$ & $78.0 \pm 8.5$ & $12.8 \pm 4.6$ & \textbf{93.6} $\pm$ \textbf{3.3}  & $91.6 \pm 2.6$ & $81.6 \pm 7.0$& \textbf{91.6} $\pm$ \textbf{6.2} \\
\hline 
FGW  & \textbf{99.2} $\pm$ \textbf{1.1}  & $80.0 \pm 7.1$  & \textbf{74.8} $\pm$ \textbf{3.6} & $92.8 \pm 4.2$ & $91.6 \pm 2.6$ & $84.3 \pm  8.6$ & $89.2 \pm 6.6$  \\ 
\hline
COPT & $98.0\pm 1.4$ & $73.6 \pm 7.9$ & $16.6 \pm 3.1$ & $92.4 \pm 4.8$ & $91.6 \pm 2.6$ & $80.0 \pm 5.6$ & $90.4 \pm 6.7$   \\
\hline
\end{tabular}
\caption{5-nearest neighbor classification accuracies for networks with discrete vertex attributes.
Average accuracies observed over 5 random samplings of the training and test sets are reported along with their standard deviation.
}
\label{table:graph_class_acc}
\end{table}

\subsection{Network Isomorphism}\label{sec:graph_isomorphism}

\begin{figure}[ht]
\centering
\begin{subfigure}{0.7\textwidth}
\includegraphics[width=\textwidth, trim=0.1cm 0.1cm 0.1cm 0.1cm, clip]{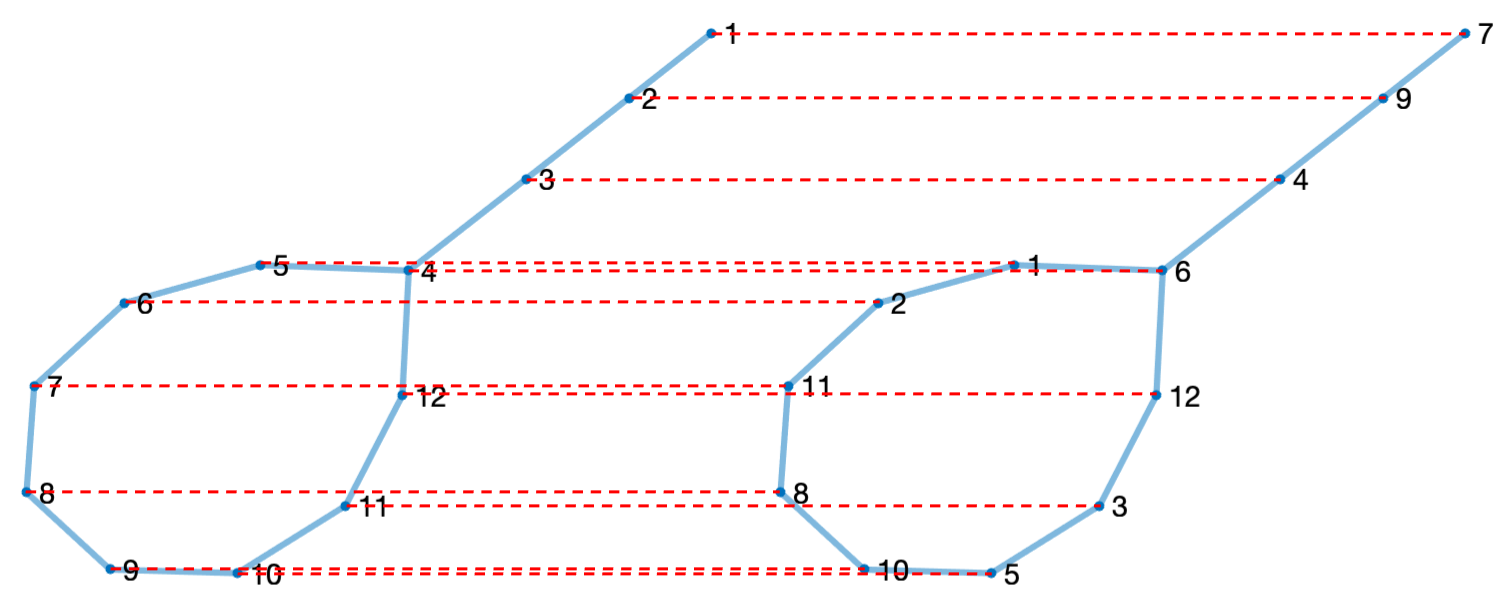}
\caption{NetOTC alignment.}
\end{subfigure}
\begin{subfigure}{0.7\textwidth}
\includegraphics[width=\textwidth, trim=0.1cm 0.1cm 0.1cm 0.1cm, clip]{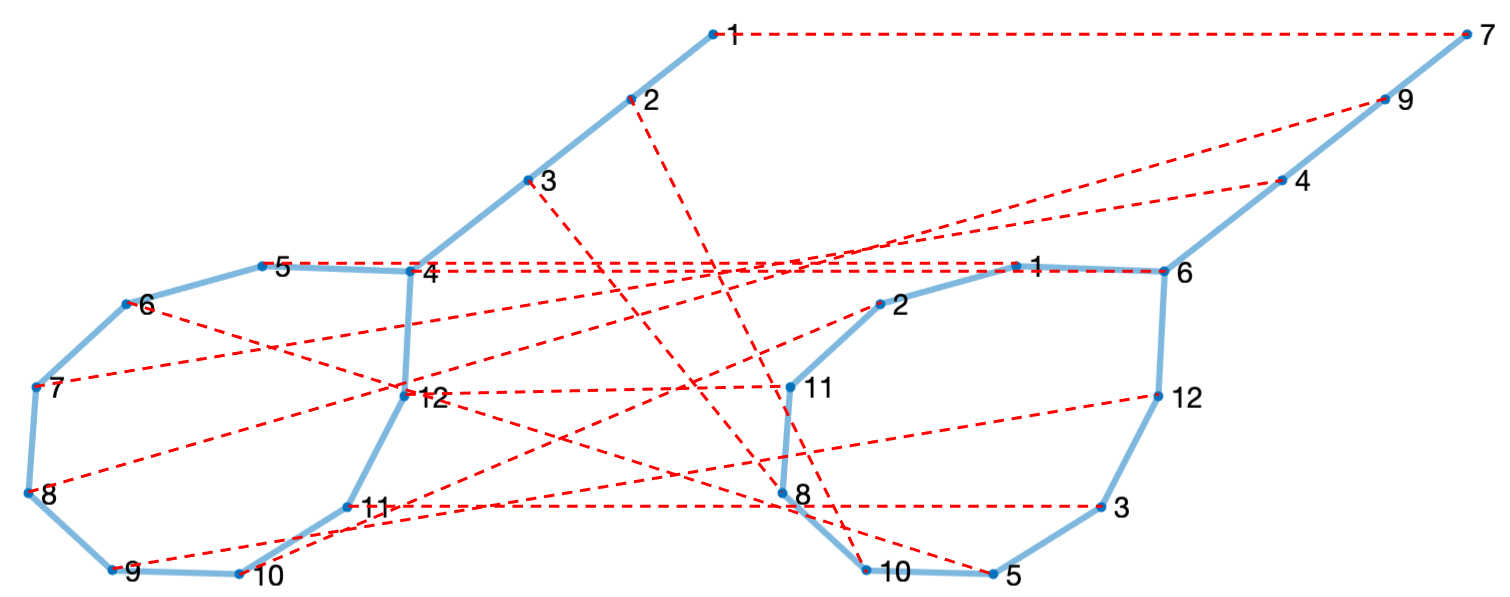}
\caption{OT alignment.}
\end{subfigure}
\begin{subfigure}{0.7\textwidth}
\includegraphics[width=\textwidth, trim=0.1cm 0.1cm 0.1cm 0.1cm, clip]{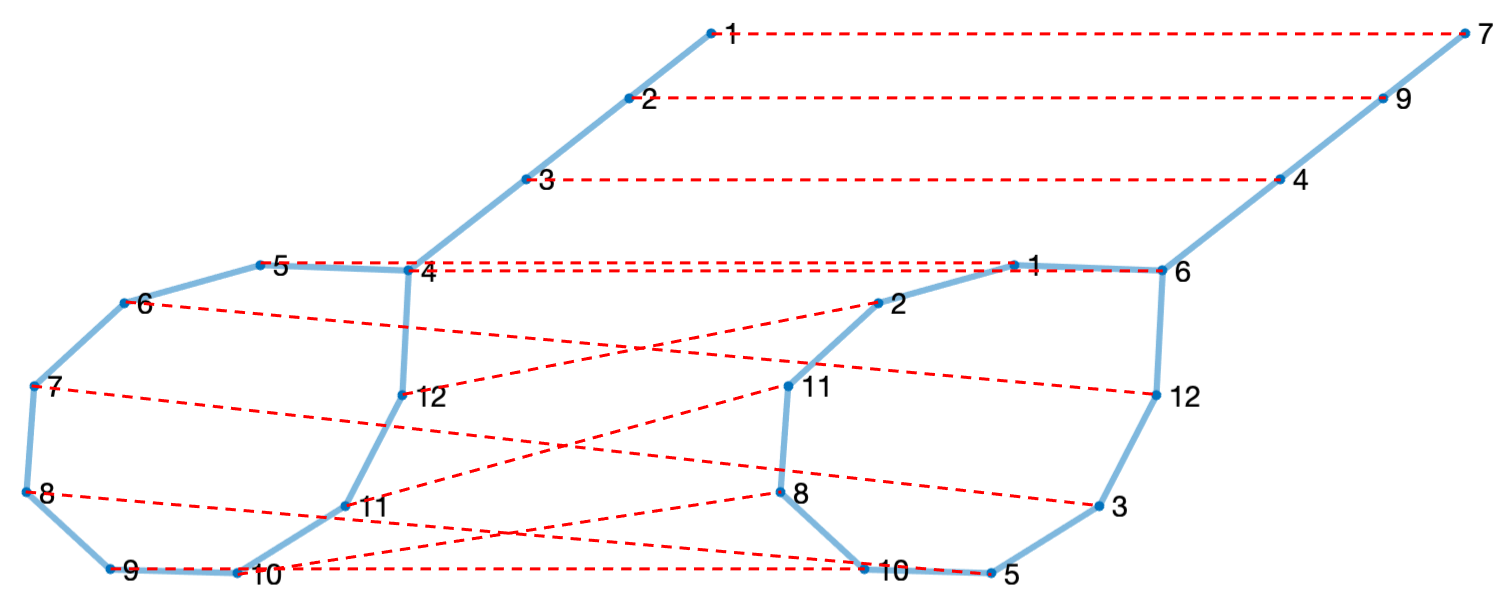}
\caption{FGW alignment.}
\end{subfigure}
\caption{Vertex alignment of two isomorphic lollipop networks obtained by NetOTC, OT, and FGW. NetOTC correctly finds the isomorphism map, while other methods do not.}
\label{fig:isomorphism_example}
\end{figure}

Two undirected, unweighted networks $G_1 = (U,E_1)$ and $G_2 = (V, E_2)$ are isomorphic if there is a bijection $\phi: U \to V$ of their vertex sets that preserves edges in the sense that $(u,u') \in E_1$ if and only if
$(\phi(u), \phi(u')) \in E_2$.  Determining when two networks are isomorphic, and if so identifying an isomorphism, 
are important problems in network theory that have received much attention in the literature \citep{McKay2014PracticalGI, Cordella2004AI}.  
In general, it is challenging to find isomorphisms in an efficient fashion \citep{Babai2015GraphII}.

To evaluate the ability of the alignment methods under study to successfully identify network isomorphisms, we carried out the following experiment. 
Given a network $G_1$, we create an isomorphic copy $G_2$ by applying a permutation $\phi$ to its vertices. 
We then applied NetOTC, FGW, GW, and OT to $G_1$ and $G_2$ with a degree-based cost function.
From each method we obtained a soft vertex alignment $\pi_{\text{v}}: U \times V \rightarrow \mathbb{R}_+$ of the given networks,
and from $\pi_{\text{v}}$ we derived a hard vertex alignment $\psi(u) = \argmax_{v \in V} \pi_{\text{v}}(u,v)$.
If the hard alignment $\psi$ is identical to the isomorphism $\phi$, the algorithm has successfully identified the isomorphism map.
See Appendix \ref{app:iso_exp_details} for more details.
Figure \ref{fig:isomorphism_example} shows an example of isomorphic ``lollipop'' networks 
(another example is shown in Appendix \ref{app:iso_exp_details} Figure \ref{fig:wheel_fgw_ot}).  
NetOTC correctly finds the isomorphism map between two isomorphic networks, but the other methods do not.
Note that the cost function used by each method to align the vertices depends only on their degree, and that the majority of the vertices 
in the lollipop network have degree 2.
This example demonstrates that, while the objective function of the NetOTC problem is univariate (it depends only on the cost between
vertices at time zero), both the NetOTC distance and optimal transition coupling depend critically on the topological properties of 
the given networks.


\begin{table}[ht]
\small
\centering
\begin{tabular}{c c c c c}
\hline 
 & NetOTC & FGW & GW & OT \\ 
\hline
\hline
Erdos-Renyi $(n \in \{ 6,\dots,15\}, \, p=1/3)$ & \textbf{96.73} & 71.50 & 54.67 & 2.80 \\
\hline
Erdos-Renyi $(n\in \{ 6,\dots,15\}, \, p=2/3)$ & \textbf{94.86} & 57.19 & 48.63 & 8.56 \\
\hline 
Erdos-Renyi $(n\in \{ 16,\dots,25\}, \, p=1/4)$ & \textbf{99.64} & 88.45 & 69.68 & 0.00 \\
\hline 
Erdos-Renyi $(n\in \{ 16,\dots,25\}, \, p=3/4)$ & \textbf{100.00} & 71.33 & 50.00 & 0.00 \\
\hline 
SBM $(7,7,7,7)$ & \textbf{100.00} & 84.62 & 54.85 & 0.00 \\
\hline 
SBM $(10,8,6)$ & \textbf{100.00} & 78.33 & 58.67 & 0.00 \\
\hline 
SBM $(7,7,7)$ & \textbf{100.00} & 71.28 & 41.89 & 0.00\\
\hline 
Random weighted adjacency matrix $\{ 0,1,2 \}$ &  \textbf{100.00} &  96.67 & 96.33 & 5.67 \\ 
\hline
Random Lollipop network& \textbf{98.00} & 13.67 & 6.00 & 0.00 \\
\hline 
\end{tabular}
\caption{Isomorphism map identification success rate (\%). We generate 300 random networks in each class. For each random network, we permute its vertices and apply the algorithms to the two isomorphic networks. We report the percentage of times the output alignment of an algorithm yields an isomorphism of the given graphs.}
\label{table:isomorphism}
\end{table}

Further experiments demonstrate the ability of NetOTC to recover isomorphisms in different classes of networks:
random (Erdos-Renyi) networks, stochastic block models (SBMs), networks with a random weighted (0,1,2) adjacency matrix, and random lollipop networks.
Table \ref{table:isomorphism} shows the average performance of each method for 300 random networks of each class.
See Appendix \ref{app:iso_exp_details} for further details of the network generation process.
NetOTC exhibits perfect performance for networks with random adjacency matrices and all types of SBMs, with very high accuracy 
in Erdos-Renyi and random lollipop networks.
On all classes but the random adjacency matrix, the competing methods perform markedly worse.
The poor performance of OT demonstrates the substantial performance gains obtained from coupling the full random walks on $G_1$ and $G_2$,
rather than their stationary distributions.

\subsection{Block Alignment in Stochastic Block Models}
\label{sec:sbm_alignment}

\begin{figure}[ht]
\centering
\begin{subfigure}{0.48\textwidth}
\includegraphics[width=\textwidth, trim=2cm 1cm 1cm 1cm, clip]{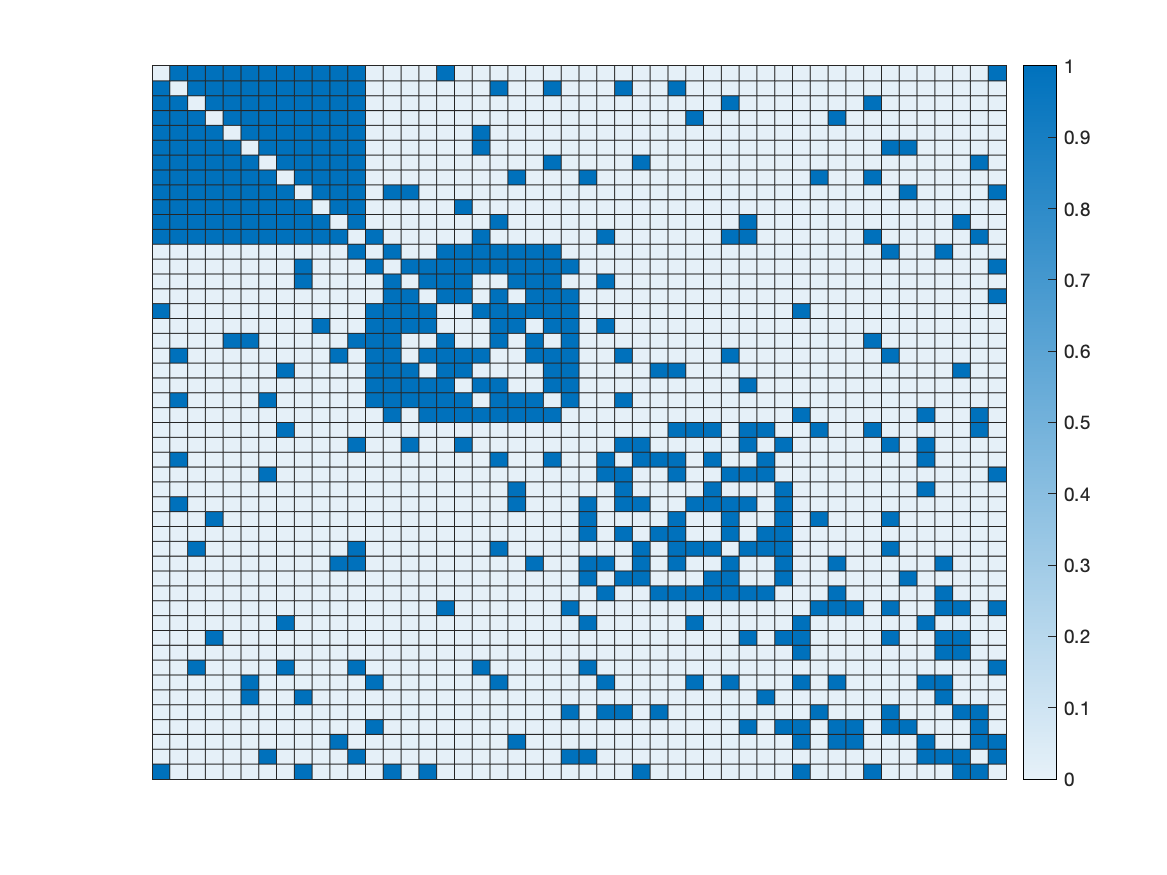}
\caption*{$G_1$}
\end{subfigure}
\begin{subfigure}{0.48\textwidth}
\includegraphics[width=\textwidth, trim=2cm 1cm 1cm 1cm, clip]{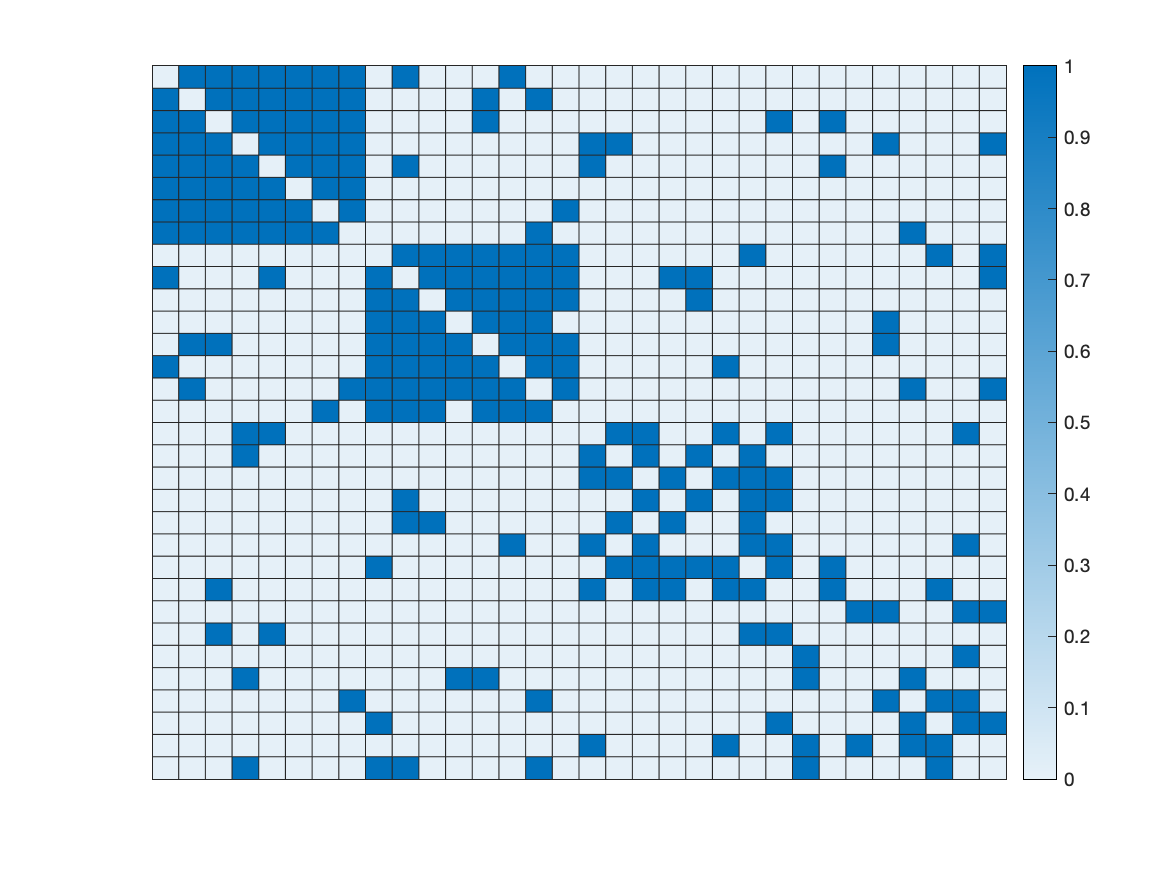}
\caption*{$G_2$}
\end{subfigure}
\caption{The adjacency matrices of random networks $G_1$ and $G_2$ drawn from SBMs.  Block structure arises from 
different connection probabilities within the blocks. Networks $G_1$ and $G_2$ are designed to have the same block structure with a different number of vertices.}
\label{fig:sbm_example}
\end{figure}

Stochastic block models \citep{Holland1983StochasticBF} (SBMs) are frequently used to model random 
networks with community structure.  SBMs have found application in a variety of network problems, 
including community detection and network clustering, see for example
\cite{JMLR:v18:16-480, Lee2019ARO, Abbe2015CommunityDI}.
In an SBM, each vertex is assigned (deterministically or at random) to one of a small number of groups, also
known as blocks.  Once group assignment is complete,
edges are placed between pairs of vertices independently, with the probability of an edge being present
depending on the group assignments of its endpoints.
In most cases, edge probabilities are higher within groups than between groups, so that the 
vertices in a group constitute, informally at least, a community.

We wished to assess the ability of OT-based comparison methods to align the vertices 
and edges of stochastically equivalent blocks in SBMs of different sizes.
To this end, we generated 10 realizations $G_1, G_2$ of SBMs with 4 blocks.
In each case, the network $G_1$ had 12 vertices per block, and $G_2$ had 8 vertices per block.
For each network, the within block connection probabilities were 1, 0.8, 0.6, and 0.4, while the between block connection probability was equal to 0.1.
The adjacency matrix of a typical realization of $G_1$ and $G_2$ is depicted in Figure \ref{fig:sbm_example}.
Note that the networks $G_1$ and $G_2$ are undirected and unweighted.

We applied the five comparison methods under study to each of the 10 realizations of $G_1$ and $G_2$ using the standardized degree-based cost function.
Each method returns a vertex alignment $\pi_{\text{v}}: U \times V \rightarrow \mathbb{R}_+$ associated with
the respective optimal transport plans.
Vertex alignment accuracy was assessed by summing $\pi_{\text{v}}$ over vertex pairs 
in corresponding blocks, i.e., blocks with the same connection probability.

As described above, the NetOTC optimal transport plan also provides a native edge alignment 
$\pi_{\text{e}}: U^2 \times V^2 \rightarrow \mathbb{R}_+$.  
For other methods, we formed an edge alignment by setting 
$\pi_{\text{e}}((u,u'),(v,v')) = \pi_{\text{v}}(u,v) \pi_{\text{v}}(u',v')$.
Edge alignment accuracy was evaluated by summing the alignment probabilities of all pairs of edges 
connecting stochastically equivalent blocks, i.e., summing all $\pi_{\text{e}}((u,u'),(v,v'))$ where $u$ and $v$, and $u'$ and $v'$ are from blocks with equal connection probabilities.

Figure \ref{fig:sbm_align_exp} shows the vertex and edge alignment accuracies for each of the methods tested.
As background, we note that random guessing yields an accuracy of 25\% for vertex alignment 
and 6.25\% for edge alignment.
For vertex alignment, NetOTC, GW, and FGW exhibit similar performance (substantially better than 
random guessing) while OT and COPT do worse.
As indicated by the error bars in Figure \ref{fig:sbm_align_exp}, the vertex alignment accuracy of NetOTC has a lower variance than the accuracies of OT, GW, and FGW.
As expected, NetOTC outperforms other methods from the standpoint of edge alignment.

\begin{figure}[ht]
\centering
\begin{subfigure}{0.37\textwidth}
\includegraphics[width=\textwidth, trim=0cm 0cm 0cm 0cm, clip]{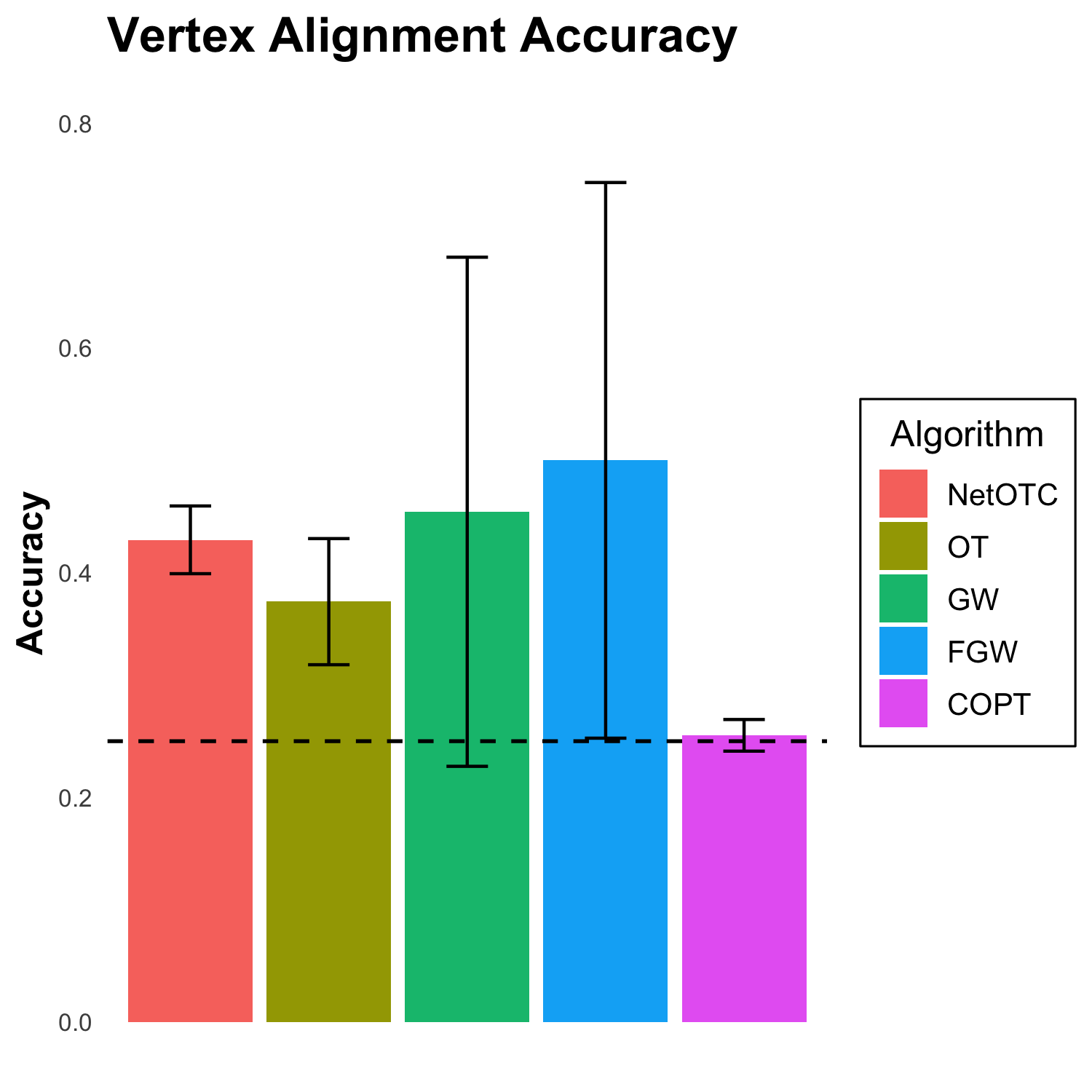}
\end{subfigure}
\begin{subfigure}{0.37\textwidth}
\includegraphics[width=\textwidth, trim=0cm 0cm 0cm 0cm, clip]{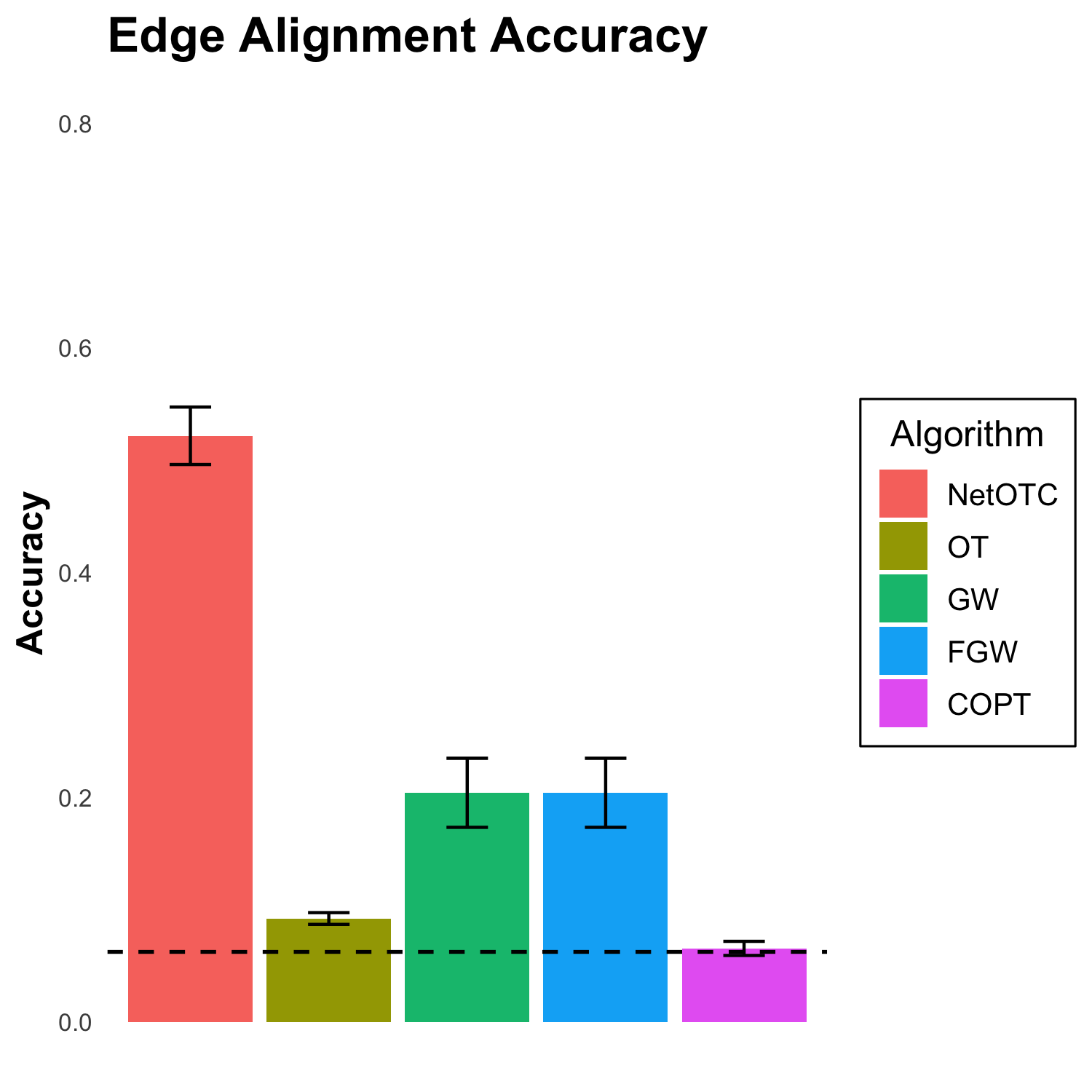}
\end{subfigure}
\caption{SBM alignment accuracies.
Average accuracies observed over 10 random pairs of SBMs are reported along with their standard deviation.
The horizontal dashed line in each plot indicates the accuracy of random guessing.}
\label{fig:sbm_align_exp}
\end{figure}

\subsection{Network Factors}
\label{sec:graph_factors}

Lastly, we considered the task of aligning corresponding vertices when the network $G_2$ is a factor of $G_1$
and the cost is compatible with the factor map (see Section \ref{sec:factors}).
We construct networks $G_1$ and $G_2$ via vertex embeddings in $\mathbb{R}^5$ as follows. 
The network $G_2$ has 6 vertices, each associated with a feature vector generated from a 5-dimensional normal 
distribution with mean zero and variance $\sigma^2 \mathbf{I}$.
The network $G_1$ has 30 vertices, each associated with a feature vector sampled from a 6-component Gaussian mixture model. 
The means of the 6 components correspond to the feature vectors associated with the 6 vertices of $G_2$, while the variances are $\mathbf{I}_5$.
Here, the factor map $f$ is determined by the component from which the feature vector was sampled. 
Next, we randomly set the edge weights of network $G_2$ to an integer between 1 and 10.
Then, the edge weights of network $G_1$ are randomly determined so that equation \eqref{eq:factor_def} holds, 
and therefore $G_2$ is a factor of $G_1$.
The networks considered are undirected, in order to enable comparison with other methods. 
See Table \ref{table:graph_factors_directed} of Appendix \ref{app:factors_exp_details} for results on directed networks;
there was no significant difference in the performance of NetOTC between directed and undirected networks.
The experiment was repeated in a setting where $G_2$ is an approximate factor of $G_1$, that is, when the factor condition only holds approximately.
See Appendix \ref{app:factors_exp_details} for explanations of how we generated an approximate factor.

NetOTC, FGW, and OT were applied to the generated networks $G_1$ and $G_2$ using an embedding-based cost equal to the squared Euclidean distances between the vectors associated with the vertices.
Table \ref{table:graph_factors} reports the vertex alignment accuracy of each method for different values of the variance $\sigma^2$.
Vertex alignment accuracy was assessed by summing the mass of the optimal coupling on factor pairs of the form $(u, f(u))$.  
NetOTC outperforms the other methods, with the performance gap growing as $\sigma$ decreases.
For an exact factor with compatible cost, which occurs when $\sigma=2.5$, 
NetOTC returns a perfect alignment, as guaranteed by Corollary \ref{cor:factor}.
Results from other cases demonstrate that the performance of NetOTC is robust when the factor and cost conditions hold only approximately.
It is also noteworthy that FGW and OT yield nearly identical alignment accuracy in all cases.

\begin{table}[ht]
\small
\centering
\begin{tabularx}{\textwidth}{@{} X *6{c} @{}}
\toprule
  & \multicolumn{3}{c}{Exact factor} & \multicolumn{3}{c}{Approximate factor} \\
  \cmidrule(r){2-4} \cmidrule(l){5-7}
  & NetOTC & FGW & OT & NetOTC & FGW & OT  \\ \midrule

$\sigma = 2.5$ & \textbf{100.00} $\pm$ \textbf{0.00} & $98.17 \pm 3.95$ & $98.38 \pm 3.40$ & \textbf{98.57} $\pm$ \textbf{0.20} & $98.07 \pm 4.55$ & $98.32 \pm 3.96$ \\ 
$\sigma = 2.0$ & \textbf{99.95} $\pm$ \textbf{0.46} & $96.57 \pm 5.22$ & $96.75 \pm 5.08$ & \textbf{98.09} $\pm$ \textbf{1.64}& $95.83 \pm 5.35$ & $96.00 \pm 5.11$\\
$\sigma = 1.5$ & \textbf{97.45} $\pm$ \textbf{5.10} & $90.20 \pm 8.24$ & $90.57 \pm 8.18$ & \textbf{96.43} $\pm$ \textbf{4.40} & $91.23 \pm 7.85$ & $91.65 \pm 7.72$\\
$\sigma = 1.0$ & \textbf{81.60} $\pm$ \textbf{13.88} & $73.23 \pm 12.08$ & $73.85 \pm 12.05$ & \textbf{79.44} $\pm$ \textbf{12.14}& $72.00 \pm 11.74$ & $72.00 \pm 12.06$ \\ \bottomrule
\end{tabularx}
\caption{Undirected networks: alignment accuracies of network factors. Average accuracies observed over 100 random network factors are reported along with their standard deviation.}
\label{table:graph_factors}
\end{table}

\vskip.3in

\section{Conclusion}\label{sec:gotc_conclusion}








In this paper we have introduced the NetOTC method for comparison and alignment of weighted networks. 
This new approach is based on constrained optimal transport of the random walks (Markov chains) associated with each network.
Given two networks and a vertex-related cost function, NetOTC identifies an optimal coupling between their associated 
random walks that minimizes the expected cost at time zero. 
NetOTC applies to both directed and undirected networks, as well as networks with different numbers of nodes. 
Once edge weights and node-cost functions are specified, NetOTC has no free parameters and involves no randomization. 
The expected cost resulting from the optimal coupling serves as a numerical measure of the dissimilarity between the networks.
In addition, the optimal transport plan itself offers interpretable, probabilistic alignments of both vertices and edges between the two networks.

We have demonstrated that NetOTC effectively incorporates global and local structure by focusing on coupling the full random walks, 
rather than stationary or other node-level distributions. 
We established several theoretical properties of NetOTC that support its use, including metric properties of the minimizing cost 
as well as its connection with short- and long-run average cost. 
A key feature of NetOTC is that it respects edges: edges are aligned with positive probability 
only if they are present in the given networks. 
In addition, we introduced a new notion of factor for weighted networks and established a close connection between factors and NetOTC. 
Complementing the theory, we presented simulations and numerical experiments showing that NetOTC is competitive with, 
and sometimes superior to, other optimal transport-based network comparison methods in the literature. 
In particular, NetOTC showed promise in identifying isomorphic networks using a local (degree-based) cost function.
%
\vskip.3in
\section{Proofs}\label{sec:proofs}

In this section, we provide the proof of our results.
Before proceeding, we introduce some necessary background.
Given two stationary processes $X = X_0, X_1, \dots$ and $Y = Y_0, Y_1, \dots$ taking values in finite sets $U$ and $V$, respectively, 
and a cost $c: U \times V \rightarrow \bbR_+$,
the optimal joining distance $\oj(X, Y)$ is the minimum of $\mathbb{E}[c(\tilde{X}_0, \tilde{Y}_0)]$ over stationary couplings 
$(\tilde{X}, \tilde{Y})$ of $X$ and $Y$.
Stationary couplings are referred to as \emph{joinings} in the ergodic theory literature \citep{furstenberg1967disjointness,dejoinings}.
As every transition coupling of stationary Markov chains $X$ and $Y$ is also a joining, we have 
\begin{equation}
\label{eq:oj_otc_bound}
\oj(X, Y) \leq \min_{(\tilde{X}, \tilde{Y}) \in \tc(X,Y)} \E c(\tilde{X}_0, \tilde{Y}_0) = \rho(G_1, G_2) 
\end{equation}
When $c$ is a metric, $\oj(\cdot, \cdot)$ is a metric between stationary 
processes \citep{gray1975generalization}.

Now we introduce some additional notation. For any finite set $S$, let $\Delta_S$ 
denote the set of probability distributions on $S$. We will regard $\lambda \in \Delta_S$ as a row vector, 
and write $\lambda(s)$ for the $\lambda$-probability of $s \in S$. 
If $g : S \to \mathbb{R}$ is any function, let $\langle \lambda, g \rangle = \sum_{s \in S} \lambda(s) g(s)$, 
which is the expectation of $g$ with respect to $\lambda$. 
For stochastic matrices (equivalently, Markov transition kernels) $\bP$ and $\bQ$ let $\Pi_{TC}(\bP,\bQ)$ denote the set 
of stochastic matrices $\bR$ satisfying the transition coupling condition \eqref{def:TC}. 
By Proposition 4 in \cite{o2020optimal} the NetOTC cost can be written as 
$\rho(G_1, G_2) = \min\left\{ \langle \lambda, c \rangle: \bR \in \Pi_{\mbox{\tiny TC}}(\bP, \bQ), 
\lambda \bR = \lambda, \lambda \in \Delta_{U \times V} \right\}$. 

\subsection{Result for NetOTC Edge Preservation}

\edgepreservation*
\begin{proof}
    Let $\bP$ and $\bQ$ be the transition kernels associated with graphs $G_1$ and $G_2$, and let
    \begin{align*}
        \lambda^*, \bR^* = \argmin_{\lambda, \,\bR} \left\{ \langle \lambda, c \rangle: \bR \in \Pi_{\mbox{\tiny TC}}(\bP, \bQ), \lambda \bR = \lambda, \lambda \in \Delta_{U \times V} \right\},
    \end{align*}
    be a solution for the NetOTC problem.
    The edge alignment probability can be expressed in terms of $\lambda^*, \bR^*$ as $\pi_{\text{e}}((u,u'),(v,v')) = \lambda^*(u,v) \, \bR^*(u',v'|u,v)$. Since $\bR^* \in \Pi_{\mbox{\tiny TC}}(\bP, \bQ)$ and $\pi_{\text{e}}((u,u'),(v,v')) > 0$, we have
    \begin{align*}
        \bP(u'|u)  & = \sum_{\tilde{v} \in V} \bR^*(u' ,  \tilde{v}  \mid  u, v) \geq \bR^*(u',v'|u,v) > 0, \\
        \bQ(v'|v) & = \sum_{\tilde{u} \in U} \bR^*(\tilde{u}, v'  \mid  u,v) \geq \bR^*(u',v'|u,v) > 0,
    \end{align*}
which implies $(u,u') \in E_1$ and $(v,v') \in E_2$.

\end{proof}
\subsection{Properties of \texorpdfstring{$\rho(G_1,G_2)$}{rho}}

\gotceq*
\begin{proof}
Let $(\tilde{X},\tilde{Y})$ be a transition coupling of $X$ and $Y$.  As $(\tilde{X},\tilde{Y})$ is stationary,
the ergodic theorem (Theorem C.1 of \cite{levin2017markov}) ensures that the limit 
\begin{equation*}
\hat{c}(\tilde{X},\tilde{Y}) := \lim_k c_k(\tilde{X}_0^{k-1}, \tilde{Y}_0^{k-1})
\end{equation*}
exists almost surely and that
$\mathbb{E} \hat{c}(\tilde{X},\tilde{Y}) = \mathbb{E} c(\tilde{X},\tilde{Y})$.
It then follows from the definition of $\overline{c}$ that
\begin{equation*}
\mathbb{E} \overline{c}(\tilde{X},\tilde{Y}) = \mathbb{E} \hat{c}(\tilde{X},\tilde{Y}) = \mathbb{E} c(\tilde{X},\tilde{Y}).
\end{equation*}
Taking minima over the set of all transition couplings of $X$ and $Y$ yields the result. 
\end{proof}

\gotclowerbound*
\begin{proof}
Every transition coupling of $X$ and $Y$ is also a joining of $X$ and $Y$ and therefore, as noted in (\ref{eq:oj_otc_bound}),
$\rho(G_1, G_2) \geq \oj(X, Y)$.
Let $(\tilde{X}, \tilde{Y})$ be any joining of $X$ and $Y$.  Stationarity of $(\tilde{X}, \tilde{Y})$ implies that
\begin{equation*}
\mathbb{E}c(\tilde{X}_0, \tilde{Y}_0) =  \mathbb{E}c_k(\tilde{X}_0^{k-1}, \tilde{Y}_0^{k-1})
\end{equation*}
and therefore $\oj(X, Y) =  \ojk(X, Y)$.
Moreover, $(\tilde{X}_0^{k-1}, \tilde{Y}_0^{k-1})$ is also a coupling of $X_0^{k-1}$ and $Y_1^k$ so
\begin{equation*}
\ojk(X, Y) \geq \min_{(\tilde{X}_0^{k-1}, \tilde{Y}_0^{k-1}) \in \Pi(X_0^{k-1}, Y_0^{k-1})} \E c_k(\tilde{X}_0^{k-1}, \tilde{Y}_0^{k-1}).
\end{equation*}
Combining these inequalities gives the result.
\end{proof}

\subsection{Results for Undirected Networks with a Common Vertex Set}
Here we consider undirected networks on a common vertex set.  We assume that the networks are connected. 
Throughout this section, we will make use of the well-known fact that the stationary distribution of a 
simple random walk $X = X_0, X_1, ...$ on a connected undirected network $G = (U, E, w)$ satisfies 
$\bp(u) = d(u)/D$, where $d(u)$ is the weighted degree of $u$ and $D$ is the sum of all the weights in the network.

\graphequiv*
\begin{proof}
Let $d_1(u) = \sum_{u' \in U} w_1(u, u')$ and $d_2(u) = \sum_{u' \in U} w_2(u, u')$ be the degree functions for $G_1$ and $G_2$.
Suppose first that $G_1 \sim G_2$ and thus there exists $C > 0$ such that $w_1(u, u') = C w_2(u, u')$ for every $u,u' \in U$. Then the two transition matrices $\bP$ and $\bQ$ associated with $G_1$ and $G_2$ are equal since
\begin{equation*}
    \bP(u' | u) \, = \, \frac{ w_1(u,u') }{ d_1(u) } = \, \frac{  w_2(u,u') }{ d_2(u) } \, = \bQ(u' | u).
\end{equation*}
Suppose now that $\bP$ and $\bQ$ are equal.
Then for every $(u, u') \in E_1$, we have
\begin{equation*}
\frac{w_1(u,u')}{d_1(u)} = \frac{w_2(u,u')}{d_2(u)},
\end{equation*}
or equivalently $w_1(u, u') = C_u w_2(u, u')$ where $C_u = d_1(u)/d_2(u)$.
As $G_1$ and $G_2$ are undirected, it is easy to see that $C_u = C_{u'}$.
As $G_1$ and $G_2$ are connected, there exists a sequence of edges $(u_1, u_2), (u_2,u_3), (u_{n-1}, u_n) \in E$ 
such that $U \subseteq \{u_1, ..., u_n\}$.
Repeating the arguments above for all edges in this sequence we conclude that $C_u = C_{u'}$ for every $u, u' \in U$,
and it follows that $G_1 \sim G_2$.
\end{proof}

\gotcismetric*
\begin{proof}
The symmetry of $\rho$ is clear.
It is established in Proposition 25 of \cite{o2020optimal}
that the optimal transition coupling cost satisfies the triangle inequality for Markov chains when the cost $c$ does, and
therefore $\rho$ satisfies the triangle inequality.
Thus it suffices to show that $\rho(G_1, G_2) = 0$ if and only if $G_1 \sim G_2$.
Let $G_1$ and $G_2$ be networks satisfying $G_1 \sim G_2$ with associated transition matrices $\bP$ and $\bQ$.
By Proposition \ref{prop:equiv}, $\bP$ and $\bQ$ are equal and clearly $\rho(G_1, G_2) = 0$ since $\langle \lambda, c \rangle = 0$ is achieved by $\lambda$ satisfying $\lambda(u,v) = \bp(u) \In(u = v)$,
which is stationary for the transition coupling satisfying 
\begin{equation*}
\def\arraystretch{1.5}
\bR(u',v'|u,v) = \left\{\begin{array}{ll} 
	\bP(u'|u)\In(u' = v') \quad & u = v \\
	\bP(u'|u) \bP(v'|v) \quad & \mbox{otherwise}
\end{array}\right. .
\end{equation*}
Now suppose that $G_1 \nsim G_2$.
By Proposition \ref{prop:equiv}, $\bP$ and $\bQ$ are necessarily distinct, and consequently so are their associated stationary Markov chains.
As it defines a distance on stationary processes $\oj(X, Y) > 0$, and
applying \eqref{eq:oj_otc_bound}, we conclude that $\rho(G_1, G_2) > 0$ as well.
\end{proof}

\lowerbound*
\begin{proof}
The random walks $X$ and $Y$ associated with $G_1$ and $G_2$ have stationary distributions $\bp(u) = d_1(u) / D$ and $\bq(u) = d_2(u) / D$.
Using the well-known connection between total variation distance and optimal transport under the 0-1 
cost (see, e.g., Equation 6.11 of \cite{Villani2008OptimalTO}), we have
\begin{align*}
    \min_{(\tilde{X}_0, \tilde{Y}_0) \in \Pi(X_0, Y_0)} \E c(\tilde{X}_0, \tilde{Y}_0) & = \min_{(\tilde{X}_0, \tilde{Y}_0) \in \Pi(X_0, Y_0)} \E \In(\tilde{X}_0 \neq \tilde{Y}_0)\\
    & = \frac{1}{2} \sum\limits_{u \in U} |\bp(u) - \bq(u)| = \frac{1}{2D} \sum\limits_{u \in U} |d_1(u) - d_2(u)|.
\end{align*}
Applying Proposition \ref{prop:gotclowerbound} yields the bound for $k=1$.
To obtain the bound for $k=2$, let $\delta_2((u,u'),(v,v')) = \In((u,u') \neq (v,v'))$ and note that 
\begin{equation*}
\delta_2((u,u'),(v,v')) \leq \In(u\neq v) + \In(u' \neq v') = 2 c_2((u, v), (u',v')).
\end{equation*}
By Proposition \ref{prop:gotclowerbound},
\begin{equation*}
\rho(G_1, G_2) \geq \min_{(\tilde{X}_0^1, \tilde{Y}_0^1) \in \Pi(X_0^1, Y_0^1)} \E c_2(\tilde{X}_0^1, \tilde{Y}_0^1) \geq \frac{1}{2} \min_{(\tilde{X}_0^1, \tilde{Y}_0^1) \in \Pi(X_0^1, Y_0^1)} \E \delta_2(\tilde{X}_0^1, \tilde{Y}_0^1).
\end{equation*}
Then using the connection between the transport cost with respect to $\delta_2$ and the total variation distance once again, we obtain
\begin{align*}
\rho(G_1, G_2) &\geq \frac{1}{4} \sum\limits_{u, u' \in U} |\mathbb{P}(X_0^1 = (u,u')) - \mathbb{P}(Y_0^1 = (u,u'))| \\
&= \frac{1}{4} \sum\limits_{u, u' \in U} |\bp(u) \mathbb{P}(X_1 = u' | X_0 = u) - \bq(u) \mathbb{P}(Y_1 = u' | Y_0 = u)| \\
&= \frac{1}{4} \sum\limits_{u, u' \in U} \left|\frac{d_1(u)}{D} \frac{w_1(u,u')}{d_1(u)} - \frac{d_2(u)}{D} \frac{w_2(u,u')}{d_2(u)}\right| \\
&= \frac{1}{4D} \sum\limits_{u, u' \in U} |w_1(u,u') - w_2(u,u')|.
\end{align*}
\end{proof}

\subsection{Results Concerning Network Factors}

In this section, we prove the results about factor maps, including Theorems \ref{thm:factortc} and \ref{thm:two_factor}.

\TCinTermsFactors*
\begin{proof}
In this setting, the conditions in Definition \ref{def:TC} for $(\tilde{X},\tilde{Y})$ to be a transition coupling are precisely equivalent to Condition \ref{eq:factor_def} for the restrictions of $\pi_U$ and $\pi_V$ to $W$.
\end{proof}

\factortc*
\begin{proof}

To prove 1., let $f$ be a factor map from $G_1$ to $G_2$. We first show $Y \eqd f(X)$. 
In order to simplify notation, we will let $f(X_0^{n-1}) = f(X_0), \dots, f(X_{n-1})$. Let us prove by induction that for any $v_0^n \in V$, we have 
\begin{equation*}
\mathbb{P} \left( f(X_0^n) = v_0^n \right) = \mathbb{P} \left( Y_0^n = v_0^n \right).
\end{equation*}
The base case ($n=0$) is immediate from Equation (\ref{eq:stat_dist_factor_def}). For the inductive step, we suppose it is true for some $n \geq 0$. Let $v_0^{n+1} \in V$. Then we have
\begin{align*}
\mathbb{P} \left( f(X_0^n) = v_0^{n+1} \right) & = \sum_{u_0^{n+1} \in f^{-1}(v_0^{n+1})} \mathbb{P}( X_0^{n+1} = u_0^{n+1} ) \\
 & = \sum_{u_0^{n+1} \in f^{-1}(v_0^{n+1})} \mathbb{P}( X_0^{n} = u_0^{n} ) \cdot  \mathbb{P}( X_{n+1} = u_{n+1} \mid  X_0^{n} = u_0^{n} )  \\
  & = \sum_{u_0^n \in f^{-1}(v_0^n)} \mathbb{P}( X_0^{n} = u_0^{n} ) \cdot \sum_{u_{n+1} \in f^{-1}(v_{n+1})} \bP( u_{n+1} \mid u_n ) \\
  & = \bQ(v_{n+1} \mid v_n) \cdot  \sum_{u_0^n \in f^{-1}(v_0^n)}  \mathbb{P}( X_0^{n} = u_0^{n} )  \\
  & = \mathbb{P}( Y_{n+1} = v_{n+1} \mid Y_n = v_n) \cdot \mathbb{P}( Y_0^n = v_0^n ) \\
  & = \mathbb{P}( Y_0^{n+1} = v_0^{n+1} ),
\end{align*}
where we have used Equation (\ref{eq:factor_trans_probs}) and the inductive hypothesis.

Next, we show that $(X,f(X))$ is a transition coupling of $X$ and $Y$. We begin by verifying that $(X, f(X))$ is Markov.
Fixing $(u, v) \in U \times V$ and $n \geq 1$, we have 
\begin{align*}
    \mathbb{P}((X_n,f(X_n))=(u, v)|\{(X_i,f(X_i))\}_{i<n}) & =\mathbb{P}((X_n,f(X_n))=(u, v)|\{X_i\}_{i<n}) \\
    & = \mathbb{P}(X_n=u|\{X_i\}_{i<n}) \In(f(u)=v) \\
    & = \mathbb{P}(X_n=u|X_{n-1}) \In(f(u)=v) \\
    & = \mathbb{P}((X_n,f(X_n))=(u, v)|X_{n-1}) \\
    & = \mathbb{P}((X_n,f(X_n))=(u, v)|(X_{n-1},f(X_{n-1}))),
\end{align*}
so the process $(X, f(X))$ is Markov.

This Markov chain clearly has a $U$ marginal that is equal in distribution to $X$ and the $V$ marginal is equal in distribution to $Y$ as we proved above. 
Thus the joint process is a coupling of $X$ and $Y$.
So it suffices to check the transition coupling condition.
Let $\bR \in [0,1]^{|U||V| \times |U||V|}$ denote the transition matrix satisfying
\begin{equation*}
\def\arraystretch{1.5}
\bR(u',v'|u,v) = \left\{	\begin{array}{ll} 
	\bP(u'|u)\In(f(u') = v') \quad & f(u) = v \\
	\bP(u'|u)\bQ(v'|v) \quad & \mbox{otherwise}
\end{array}\right. .
\end{equation*}

Let $\bp$ be the stationary distribution of $X$. Then $\lambda(u,v) = \bp(u) I(f(u)=v)$ is equal in distribution to $(X_0,f(X_0))$.
Observe that for all $(u',v') \in U \times V$, we have
\begin{align*}
\sum\limits_{(u,v) \in U \times V} \lambda(u,v) \bR(u',v'|u,v) &= \sum\limits_{(u,v) \in U \times V} \bp(u) \In(f(u)=v) \bP(u'|u) \In(f(u')=v') \\
&= \In(f(u')=v') \sum\limits_{v \in V} \sum\limits_{u \in f^{-1}(v)} \bp(u) \bP(u'|u) \\
&= \In(f(u')=v') \sum\limits_{u \in U} \bp(u) \bP(u'|u) \\
&= \bp(u') \In(f(u') = v') \\
&= \lambda(u', v'),
\end{align*}
and therefore $\lambda$ is stationary for $\bR$. So lastly, we only need to show that $\bR \in \Pitc (\bP, \bQ)$.
For pairs $(u, v) \in U \times V$ with $f(u) \neq v$, we see that $\bR( \cdot|u,v)$ is the independent coupling of $\bP(\cdot|u)$ and $\bQ(\cdot|v)$, which clearly satisfies the transition coupling condition. Thus we need only check the transition coupling condition for pairs $(u,v)$ satisfying $f(u) = v$.
Let $u, u' \in U$ and $v\in V$ and suppose that $f(u)=v$.
Then
\begin{equation*}
    \sum\limits_{v' \in V} \bR(u', v'|u,v) = \sum_{v' \in V} \bP(u'|u) \In(f(u')=v') = \bP(u'|u).
\end{equation*}
Now checking the other half of the transition coupling condition, let $u \in U$ and $v, v'\in V$ be such that $f(u)=v$.
Then by Equation (\ref{eq:factor_trans_probs}), we have
\begin{equation*}
    \sum_{u' \in U} \bR(u',v'|u,v) = \sum\limits_{u' \in U} \bP(u'|u) \In(f(u')=v') = \sum\limits_{u' \in f^{-1}(v')} \bP(u'|u) = \bQ( v'|v).
\end{equation*}
Thus $(X, f(X))$ is a transition coupling of $X$ and $Y$. It is forward-deterministic by construction.

Now we prove 2. To that end, suppose $(\tilde{X},\tilde{Y})$ is a forward-deterministic coupling of $X$ and $Y$, and let $f : U \to V$ be the induced map. For notation, let $R$ be the transition matrix associated to the joint Markov chain $(\tilde{X},\tilde{Y})$. To verify that $f$ is a factor map, let $v,v' \in V$ and $u \in f^{-1}(v)$. Since $(\tilde{X},\tilde{Y})$ is forward deterministic, for every $u' \in f^{-1}(v')$ we have that $\bP(u' \mid u) = R((u',v') \mid (u,v))$. Then, also using the transition coupling property of $R$, we see that
\begin{align*}
\sum_{u' \in f^{-1}(v')} \bP(u' \mid u) = \sum_{u' \in f^{-1}(v')} R( (u',v') \mid (u,v) ) = \sum_{u' \in U} R( (u',v') \mid (u,v) )  = \bQ(v',v).
\end{align*}
\end{proof}

\factor*
\begin{proof}
By part 1. of Theorem \ref{thm:factortc}, we have that $(X,f(X))$ is a transition coupling of $X$ and $Y$.
Let $(\tilde{X}, \tilde{Y})$ be any transition coupling of $X$ and $Y$.
Then,
\begin{align*}
    \mathbb{E}c(\tilde{X}_0,\tilde{Y}_0) \geq \mathbb{E}c(\tilde{X}_0,f(\tilde{X}_0)) = \mathbb{E}c(X_0, f(X_0)).
\end{align*}
Taking an infimum over all transition couplings $(\tilde{X}, \tilde{Y})$ of $X$ and $Y$, we conclude that $(X, f(X))$ is an optimal transition coupling of $X$ and $Y$,
as desired.
\end{proof}

In the following proof of our two-factor result (Theorem \ref{thm:two_factor}), we will use the notion of relatively independent couplings. Suppose $X$, $Y$, and $Z$ are random variables and there are maps $f$ and $g$ such that $f(X) \eqd Z$ and $g(Y) \eqd Z$. The main property that we need is that there exists a coupling $(\tilde{X}, \tilde{Y})$ of $X$ and $Y$ such that $f(\tilde{X}) = g(\tilde{Y})$ almost surely. The existence of such a coupling is usually demonstrated by constructing the \textit{relatively independent coupling} of $X$ and $Y$ over $Z$, which is defined by the property that
\begin{equation*}
\mathbb{P}( \tilde{X} \in A, \tilde{Y} \in B) = \mathbb{E} \bigl[ \mathbb{P}(X \in A \mid Z) \cdot \mathbb{P}( Y \in B \mid Z) \bigr],
\end{equation*}
where the expectation is taken with respect to $Z$.
In words, the relatively independent coupling makes $\tilde{X}$ and $\tilde{Y}$ conditionally independent given $Z$.
This construction is useful in optimal transport for proving the triangle inequality. In the context of ergodic theory, when the random variables are replaced by stationary processes, it is called the relatively independent joining of $X$ and $Y$ \citep{de2005introduction}. We note if the stationary processes are the random walks on some strongly connected networks and the maps $f$ and $g$ are factor maps in the sense of Section \ref{sec:factors}, then the relatively independent joinings are in fact transition couplings.
\begin{prop} \label{prop:RelInd}
Suppose $G_1 = (U,E_1,w_1)$, $G_2 = (V, E_2, w_2)$, and $G_3 = (W, E_3, w_3)$ are strongly connected weighted directed networks with associated random walks $X$, $Y$, and $Z$. Further suppose that there are factor maps $f : U \to W$ from $G_1$ to $G_3$ and $g : V \to W$ from $G_2$ to $G_3$. Then there is a transition coupling $(\tilde{X},\tilde{Y})$ of $X$ and $Y$  such that $f(\tilde{X}) = g(\tilde{Y})$ holds almost surely.
\end{prop}
\begin{proof}
Define the coupling $(\tilde{X},\tilde{Y})$ to be the Markov chain with stationary distribution $r$ and transition kernel $\bR$ given as follows. For $w \in W$ with $\mathbb{P}(Z_0 = w)>0$ and $(u,v) \in U \times V$ such that $f(u) = g(v) = w$, let
\begin{equation*}
r(u,v) = \frac{\bp(u) \cdot \bq(v)}{ \mathbb{P}(Z_0 = w)},
\end{equation*}
and otherwise let $r(u,v) = 0$.
Furthermore, for $(w,w') \in E_3$ and $(u,v), (u',v') \in U \times V$ such that $f(u) = g(v) = w$ and $f(u') = g(v') = w'$, let
\begin{equation*}
R( (u',v') \mid (u,v) ) = 
                    \frac{\bP(u' \mid u) \cdot \bQ(v' \mid v) }{ \mathbb{P}( Z_1 = w' \mid Z_0 = w)}.
                    \end{equation*}
 If $f(u) = g(v) = w$ while $(u',v')$ does not satisfy $f(u') = g(v') = w'$, then let $R( (u',v') \mid (u,v) ) = 0$. Finally, if $f(u) = g(v) = w$ does not hold, then let $R( (u',v') \mid (u,v) ) = \bP(u' \mid u) \cdot \bQ(v \mid v')$. Using this definition, one may immediately verify Condition \eqref{def:TC}, and thus $(\tilde{X},\tilde{Y})$ is a transition coupling of $X$ and $Y$. Furthermore, by construction we have $f(\tilde{X}) = g(\tilde{Y})$ almost surely.
\end{proof}

With this result in hand, we may now proceed to our second main result concerning factors. 

\twofactor*
\begin{proof}
Let $(\tilde{X},\tilde{Y})$ be a transition coupling of $X$ and $Y$. Since $f$ and $g$ are factor maps, we see that $(f(\tilde{X}), g(\tilde{Y}))$ is a transition coupling of $W$ and $Z$. Then by the compatibility condition on the cost function we have
\begin{equation} \label{eqn:2factor}
\mathbb{E} c_{ext}(\tilde{X},\tilde{Y}) = \mathbb{E} c( f(\tilde{X}), f(\tilde{Y})).
\end{equation}

To ease the notational burden of the following argument, we do not distinguish between different couplings of the same random variables. In particular, the notation $\tilde{X}$ may represent formally distinct random variables (defined on different probability spaces) from one instance to the next, although it always denotes a random variable that is equal in distribution to $X$.

Now let $(\tilde{W},\tilde{Z})$ be a transition coupling of $W$ and $Z$. We claim that there exists a transition coupling $(\tilde{X},\tilde{Y})$ such that $(f(\tilde{X}), g(\tilde{Y})) \eqd (\tilde{W},\tilde{Z})$. To exhibit the desired transition coupling, we repeatedly use Proposition \ref{prop:RelInd}. Let $(\tilde{X},\tilde{W},\tilde{Z})$ denote the relatively independent joining of $(X,f(X))$ and $(\tilde{W},\tilde{Z})$ over $W$ (with the factor maps given by the natural projections of $(X,f(X))$ onto $f(X) \eqd W$ and of $(\tilde{W},\tilde{Z})$ onto $\tilde{W} \eqd W$, respectively). Similarly, let $(\tilde{W}, \tilde{Z},\tilde{Y})$ be the relatively independent joining of $(\tilde{W},\tilde{Z})$ and $(g(Y), Y)$ over $Z$. Now let $(\tilde{X}, \tilde{Y}, \tilde{W}, \tilde{Z})$ be the relatively independent joining of $(\tilde{X},\tilde{W},\tilde{Z})$ and $(\tilde{W}, \tilde{Z},\tilde{Y})$ over $(\tilde{W}, \tilde{Z})$. Then the projection of $(\tilde{X}, \tilde{Y}, \tilde{W}, \tilde{Z})$ onto the first two coordinates gives a transition coupling of $X$ and $Y$ with the property that $(f(\tilde{X}), g(\tilde{Y})) \eqd (\tilde{W},\tilde{Z})$. We have thus established that every transition coupling $(\tilde{W},\tilde{Z})$ of $W$ and $Z$ can be written as $(f(\tilde{X}), g(\tilde{Y}))$ for some transition coupling $(\tilde{X},\tilde{Y})$ of $X$ and $Y$.

The two conclusions of the theorem are immediate consequences of (\ref{eqn:2factor}) and the result of the previous paragraph.
\end{proof}

\vskip.3in

\section*{Acknowledgments}
KO and ABN were supported in part by NSF grants DMS-1613072 and DMS-1613261.
KM gratefully acknowledges the support of NSF CAREER grant DMS-1847144.
BY, KM, and ABN would like to acknowledge the support of NSF DMS-2113676.

\vskip.3in

\bibliographystyle{rss}
\bibliography{references}

\begin{thebibliography}{94}
\expandafter\ifx\csname natexlab\endcsname\relax\def\natexlab#1{#1}\fi
\expandafter\ifx\csname url\endcsname\relax
  \def\url#1{\texttt{#1}}\fi
\expandafter\ifx\csname urlprefix\endcsname\relax\def\urlprefix{URL }\fi

\bibitem[{Abbe(2018)}]{JMLR:v18:16-480}
Abbe, E. (2018) Community detection and stochastic block models: Recent
  developments.
\newblock \emph{Journal of Machine Learning Research}, \textbf{18}, 1--86.
\newblock \urlprefix\url{http://jmlr.org/papers/v18/16-480.html}.

\bibitem[{Abbe and Sandon(2015)}]{Abbe2015CommunityDI}
Abbe, E. and Sandon, C. (2015) Community detection in general stochastic block
  models: Fundamental limits and efficient algorithms for recovery.
\newblock \emph{2015 IEEE 56th Annual Symposium on Foundations of Computer
  Science},  670--688.

\bibitem[{Babai(2015)}]{Babai2015GraphII}
Babai, L. (2015) Graph isomorphism in quasipolynomial time.
\newblock \emph{ArXiv}, \textbf{abs/1512.03547}.

\bibitem[{Barak \emph{et~al.}(2019)Barak, Chou, Lei, Schramm and
  Sheng}]{Barak2019NearlyEA}
Barak, B., Chou, C.-N., Lei, Z., Schramm, T. and Sheng, Y. (2019) {(Nearly)
  Efficient Algorithms for the Graph Matching Problem on Correlated Random
  Graphs}.
\newblock In \emph{NeurIPS}.

\bibitem[{Barbe \emph{et~al.}(2020)Barbe, Sebban, Gon{\c{c}}alves, Borgnat and
  Gribonval}]{barbe2020graph}
Barbe, A., Sebban, M., Gon{\c{c}}alves, P., Borgnat, P. and Gribonval, R.
  (2020) Graph diffusion {W}asserstein distances.
\newblock In \emph{European Conference on Machine Learning and Principles and
  Practice of Knowledge Discovery in Databases}.

\bibitem[{Belkin and Niyogi(2003)}]{belkin2003laplacian}
Belkin, M. and Niyogi, P. (2003) Laplacian eigenmaps for dimensionality
  reduction and data representation.
\newblock \emph{Neural Computation}, \textbf{15}, 1373--1396.

\bibitem[{Blum \emph{et~al.}(2020)Blum, Hopcroft and
  Kannan}]{blum2020foundations}
Blum, A., Hopcroft, J. and Kannan, R. (2020) \emph{Foundations of Data
  Science}.
\newblock Cambridge University Press.
\newblock \urlprefix\url{https://books.google.com/books?id=koHCDwAAQBAJ}.

\bibitem[{Brugere \emph{et~al.}(2023)Brugere, Wan and
  Wang}]{Brugere2023DistancesFM}
Brugere, T., Wan, Z. and Wang, Y. (2023) Distances for markov chains, and their
  differentiation.
\newblock \emph{ArXiv}, \textbf{abs/2302.08621}.
\newblock \urlprefix\url{https://api.semanticscholar.org/CorpusID:257020095}.

\bibitem[{Chen \emph{et~al.}(2020)Chen, Gan, Cheng, Li, Carin and
  Liu}]{chen2020graph}
Chen, L., Gan, Z., Cheng, Y., Li, L., Carin, L. and Liu, J. (2020) Graph
  optimal transport for cross-domain alignment.
\newblock In \emph{International Conference on Machine Learning},  1542--1553.
  PMLR.

\bibitem[{Chen \emph{et~al.}(2022)Chen, Lim, Memoli, Wan and
  Wang}]{chen2022weisfeiler}
Chen, S., Lim, S., Memoli, F., Wan, Z. and Wang, Y. (2022) Weisfeiler-lehman
  meets gromov-{W}asserstein.
\newblock In \emph{Proceedings of the 39th International Conference on Machine
  Learning}, Proceedings of Machine Learning Research,  3371--3416. PMLR.

\bibitem[{Chen \emph{et~al.}(2023)Chen, Lim, M'emoli, Wan and
  Wang}]{Chen2023TheWD}
Chen, S., Lim, S., M'emoli, F., Wan, Z. and Wang, Y. (2023) The
  weisfeiler-lehman distance: Reinterpretation and connection with gnns.
\newblock \emph{ArXiv}, \textbf{abs/2302.00713}.

\bibitem[{Cho \emph{et~al.}(2010)Cho, Lee and Lee}]{cho2010reweighted}
Cho, M., Lee, J. and Lee, K.~M. (2010) Reweighted random walks for graph
  matching.
\newblock In \emph{Computer Vision -- ECCV 2010} (eds. K.~Daniilidis,
  P.~Maragos and N.~Paragios),  492--505. Berlin, Heidelberg: Springer Berlin
  Heidelberg.

\bibitem[{Cordella \emph{et~al.}(2004)Cordella, Foggia, Sansone and
  Vento}]{Cordella2004AI}
Cordella, L.~P., Foggia, P., Sansone, C. and Vento, M. (2004) A (sub)graph
  isomorphism algorithm for matching large graphs.
\newblock \emph{IEEE Transactions on Pattern Analysis and Machine
  Intelligence}, \textbf{26}, 1367--1372.

\bibitem[{Cour \emph{et~al.}(2007)Cour, Srinivasan and Shi}]{cour2007balanced}
Cour, T., Srinivasan, P. and Shi, J. (2007) Balanced graph matching.
\newblock In \emph{Advances in Neural Information Processing Systems}, vol.~19.

\bibitem[{Cullina and Kiyavash(2017)}]{Cullina2017ExactAR}
Cullina, D. and Kiyavash, N. (2017) {Exact alignment recovery for correlated
  Erdos Renyi graphs}.
\newblock \emph{ArXiv}, \textbf{abs/1711.06783}.

\bibitem[{Cullina \emph{et~al.}(2020)Cullina, Kiyavash, Mittal and
  Poor}]{Cullina2020PartialRO}
Cullina, D., Kiyavash, N., Mittal, P. and Poor, H.~V. (2020) {Partial Recovery
  of Erdős-R{\'e}nyi Graph Alignment via k-Core Alignment}.
\newblock \emph{Abstracts of the 2020 SIGMETRICS/Performance Joint
  International Conference on Measurement and Modeling of Computer Systems}.

\bibitem[{De~La~Rue(2005)}]{de2005introduction}
De~La~Rue, T. (2005) An introduction to joinings in ergodic theory.
\newblock \emph{arXiv preprint math/0507429}.

\bibitem[{de~la Rue(2009)}]{dejoinings}
de~la Rue, T. (2009) Joinings in ergodic theory.
\newblock In \emph{Encyclopedia of Complexity and Systems Science}. Springer,
  New York, NY.

\bibitem[{Debnath \emph{et~al.}(1991)Debnath, Lopez~de Compadre, Debnath,
  Shusterman and Hansch}]{debnath1991structure}
Debnath, A.~K., Lopez~de Compadre, R.~L., Debnath, G., Shusterman, A.~J. and
  Hansch, C. (1991) Structure-activity relationship of mutagenic aromatic and
  heteroaromatic nitro compounds. correlation with molecular orbital energies
  and hydrophobicity.
\newblock \emph{Journal of Medicinal Chemistry}, \textbf{34}, 786--797.

\bibitem[{Ding \emph{et~al.}(2020)Ding, Ma, Wu and Xu}]{Ding2020EfficientRG}
Ding, J., Ma, Z., Wu, Y. and Xu, J. (2020) Efficient random graph matching via
  degree profiles.
\newblock \emph{Probability Theory and Related Fields}, \textbf{179}, 29--115.

\bibitem[{Dong and Sawin(2020)}]{dong2020copt}
Dong, Y. and Sawin, W. (2020) {COPT}: Coordinated optimal transport on graphs.
\newblock In \emph{Advances in Neural Information Processing Systems}.

\bibitem[{Ellis(1976)}]{ellis1976thedj}
Ellis, M.~H. (1976) The $\overline{d}$-distance between two {Markov} processes
  cannot always be attained by a {Markov} joining.
\newblock \emph{Israel Journal of Mathematics}, \textbf{24}, 269--273.

\bibitem[{Ellis(1978)}]{ellis1978distances}
Ellis, M.~H. (1978) Distances between two-state {Markov} processes attainable
  by {Markov} joinings.
\newblock \emph{Transactions of the American Mathematical Society},
  \textbf{241}, 129--153.

\bibitem[{Ellis(1980)}]{ellis1980kamae}
Ellis, M.~H. (1980) {On Kamae's conjecture concerning the d-distance between
  two-state {Markov} processes}.
\newblock \emph{The Annals of Probability},  372--376.

\bibitem[{Ellis \emph{et~al.}(1980)}]{ellis1980conditions}
Ellis, M.~H. \emph{et~al.} (1980) Conditions for attaining $\bar{d} $ by a
  {Markovian} joining.
\newblock \emph{The Annals of Probability}, \textbf{8}, 431--440.

\bibitem[{Elmsallati \emph{et~al.}(2016)Elmsallati, Clark and
  Kalita}]{Elmsallati2016GlobalAO}
Elmsallati, A., Clark, C. and Kalita, J.~K. (2016) Global alignment of
  protein-protein interaction networks: A survey.
\newblock \emph{IEEE/ACM Transactions on Computational Biology and
  Bioinformatics}, \textbf{13}, 689--705.

\bibitem[{Engel \emph{et~al.}(2021)Engel, Nardo and Rancan}]{Engel2021}
Engel, J., Nardo, M. and Rancan, M. (2021) \emph{Network Analysis for Economics
  and Finance: An Application to Firm Ownership},  331--355.
\newblock Cham: Springer International Publishing.

\bibitem[{Enqvist \emph{et~al.}(2009)Enqvist, Josephson and
  Kahl}]{enqvist2009optimal}
Enqvist, O., Josephson, K. and Kahl, F. (2009) Optimal correspondences from
  pairwise constraints.
\newblock In \emph{2009 IEEE 12th International Conference on Computer Vision},
   1295--1302.

\bibitem[{Fagiolo \emph{et~al.}(2010)Fagiolo, Reyes and
  Schiavo}]{fagiolo2010evolution}
Fagiolo, G., Reyes, J. and Schiavo, S. (2010) The evolution of the world trade
  web: a weighted-network analysis.
\newblock \emph{Journal of Evolutionary Economics}, \textbf{20}, 479--514.

\bibitem[{Feizi \emph{et~al.}(2020)Feizi, Quon, Recamonde‐Mendoza,
  M{\'e}dard, Kellis and Jadbabaie}]{Feizi2020SpectralAO}
Feizi, S., Quon, G.~T., Recamonde‐Mendoza, M., M{\'e}dard, M., Kellis, M. and
  Jadbabaie, A. (2020) Spectral alignment of graphs.
\newblock \emph{IEEE Transactions on Network Science and Engineering},
  \textbf{7}, 1182--1197.

\bibitem[{Furstenberg(1967)}]{furstenberg1967disjointness}
Furstenberg, H. (1967) Disjointness in ergodic theory, minimal sets, and a
  problem in diophantine approximation.
\newblock \emph{Theory of Computing Systems}, \textbf{1}, 1--49.

\bibitem[{Garey and Johnson(1990)}]{garey1990computers}
Garey, M.~R. and Johnson, D.~S. (1990) \emph{Computers and Intractability; A
  Guide to the Theory of NP-Completeness}.
\newblock W. H. Freeman \& Co.

\bibitem[{Glasner(2003)}]{glasner2003ergodic}
Glasner, E. (2003) \emph{Ergodic {T}heory via {J}oinings}, vol. 101.
\newblock American Mathematical Society.

\bibitem[{Gold and Rangarajan(1996)}]{gold1996agraduated}
Gold, S. and Rangarajan, A. (1996) A graduated assignment algorithm for graph
  matching.
\newblock \emph{IEEE Transactions on Pattern Analysis and Machine
  Intelligence}, \textbf{18}, 377 -- 388.

\bibitem[{Gray \emph{et~al.}(1975)Gray, Neuhoff and
  Shields}]{gray1975generalization}
Gray, R.~M., Neuhoff, D.~L. and Shields, P.~C. (1975) A generalization of
  {O}rnstein's $\overline{d}$ distance with applications to information theory.
\newblock \emph{The Annals of Probability},  315--328.

\bibitem[{Grover and Leskovec(2016)}]{Grover2016node2vecSF}
Grover, A. and Leskovec, J. (2016) node2vec: Scalable feature learning for
  networks.
\newblock \emph{Proceedings of the 22nd ACM SIGKDD International Conference on
  Knowledge Discovery and Data Mining}.

\bibitem[{Hamilton(2020)}]{HamiltonGraph}
Hamilton, W.~L. (2020) Graph representation learning.
\newblock \emph{Synthesis Lectures on Artificial Intelligence and Machine
  Learning}, \textbf{14}, 1--159.

\bibitem[{Holland \emph{et~al.}(1983)Holland, Laskey and
  Leinhardt}]{Holland1983StochasticBF}
Holland, P.~W., Laskey, K.~B. and Leinhardt, S. (1983) Stochastic blockmodels:
  First steps.
\newblock \emph{Social Networks}, \textbf{5}, 109--137.

\bibitem[{Howard(1960)}]{howard1960dynamic}
Howard, R.~A. (1960) \emph{Dynamic Programming and {Markov} Processes.}
\newblock John Wiley.

\bibitem[{Jackson \emph{et~al.}(2014)Jackson, Kennedy, Bradbury and
  Karney}]{jackson2014social}
Jackson, G.~L., Kennedy, D., Bradbury, T.~N. and Karney, B.~R. (2014) A social
  network comparison of low-income black and white newlywed couples.
\newblock \emph{Journal of Marriage and Family}, \textbf{76}, 967--982.

\bibitem[{Jiang \emph{et~al.}(2017)Jiang, Tang, Ding, Gong and
  Luo}]{jiang2017graph}
Jiang, B., Tang, J., Ding, C., Gong, Y. and Luo, B. (2017) Graph matching via
  multiplicative update algorithm.
\newblock In \emph{Advances in Neural Information Processing Systems}, vol.~30.

\bibitem[{Kalaev \emph{et~al.}(2008)Kalaev, Bafna and
  Sharan}]{Kalaev2008FastAA}
Kalaev, M., Bafna, V. and Sharan, R. (2008) Fast and accurate alignment of
  multiple protein networks.
\newblock \emph{Journal of computational biology : a journal of computational
  molecular cell biology}, \textbf{16 8}, 989--99.

\bibitem[{Kelley \emph{et~al.}(2003)Kelley, Sharan, Karp, Sittler, Root,
  Stockwell and Ideker}]{Kelley2003ConservedPW}
Kelley, B.~P., Sharan, R., Karp, R.~M., Sittler, T., Root, D.~E., Stockwell,
  B.~R. and Ideker, T. (2003) Conserved pathways within bacteria and yeast as
  revealed by global protein network alignment.
\newblock \emph{Proceedings of the National Academy of Sciences of the United
  States of America}, \textbf{100}, 11394 -- 11399.

\bibitem[{Kersting \emph{et~al.}(2016)Kersting, Kriege, Morris, Mutzel and
  Neumann}]{KKMMN2016}
Kersting, K., Kriege, N.~M., Morris, C., Mutzel, P. and Neumann, M. (2016)
  Benchmark data sets for graph kernels.
\newblock \url{http://graphkernels.cs.tu-dortmund.de}.

\bibitem[{Klau(2009)}]{Klau2009ANG}
Klau, G.~W. (2009) A new graph-based method for pairwise global network
  alignment.
\newblock \emph{BMC Bioinformatics}, \textbf{10}, S59 -- S59.

\bibitem[{Korula and Lattanzi(2014)}]{Korula2014AnER}
Korula, N. and Lattanzi, S. (2014) An efficient reconciliation algorithm for
  social networks.
\newblock \emph{ArXiv}, \textbf{abs/1307.1690}.

\bibitem[{Kriege \emph{et~al.}(2018)Kriege, Fey, Fisseler, Mutzel and
  Weichert}]{kriege2018recognizing}
Kriege, N.~M., Fey, M., Fisseler, D., Mutzel, P. and Weichert, F. (2018)
  Recognizing {C}uneiform signs using graph based methods.
\newblock In \emph{International Workshop on Cost-Sensitive Learning},  31--44.
  PMLR.

\bibitem[{Kuchaiev \emph{et~al.}(2010)Kuchaiev, Milenkovi{\'c}, Memisevic,
  Hayes and Przulj}]{Kuchaiev2010TopologicalNA}
Kuchaiev, O., Milenkovi{\'c}, T., Memisevic, V., Hayes, W.~B. and Przulj, N.
  (2010) Topological network alignment uncovers biological function and
  phylogeny.
\newblock \emph{Journal of The Royal Society Interface}, \textbf{7}, 1341 --
  1354.

\bibitem[{Kuchaiev and Przulj(2011)}]{Kuchaiev2011IntegrativeNA}
Kuchaiev, O. and Przulj, N. (2011) Integrative network alignment reveals large
  regions of global network similarity in yeast and human.
\newblock \emph{Bioinformatics}, \textbf{27 10}, 1390--6.

\bibitem[{Lassalle(2013)}]{Lassalle2013CausalTP}
Lassalle, R. (2013) Causal transport plans and their monge–kantorovich
  problems.
\newblock \emph{Stochastic Analysis and Applications}, \textbf{36}, 452 -- 484.

\bibitem[{Lee and Wilkinson(2019)}]{Lee2019ARO}
Lee, C. and Wilkinson, D. (2019) A review of stochastic block models and
  extensions for graph clustering.
\newblock \emph{Applied Network Science}, \textbf{4}, 1--50.

\bibitem[{Leordeanu and Hebert(2005)}]{leordeanu2005aspectral}
Leordeanu, M. and Hebert, M. (2005) A spectral technique for correspondence
  problems using pairwise constraints.
\newblock In \emph{IEEE International Conference on Computer Vision}, vol.~2,
  1482--1489 Vol. 2.

\bibitem[{Leordeanu \emph{et~al.}(2009)Leordeanu, Hebert and
  Sukthankar}]{leordeanu2009aninteger}
Leordeanu, M., Hebert, M. and Sukthankar, R. (2009) An integer projected fixed
  point method for graph matching and map inference.
\newblock In \emph{Advances in Neural Information Processing Systems}, vol.~22.

\bibitem[{Levin and Peres(2017)}]{levin2017markov}
Levin, D.~A. and Peres, Y. (2017) \emph{{Markov} Chains and Mixing Times}, vol.
  107.
\newblock American Mathematical Soc.

\bibitem[{Lind and Marcus(1995)}]{lind1995anintroductiontosymbolicdynamics}
Lind, D. and Marcus, B. (1995) \emph{An Introduction to Symbolic Dynamics and
  Coding}.
\newblock Cambridge University Press.

\bibitem[{Loiola \emph{et~al.}(2007)Loiola, Abreu, Boaventura-Netto, Hahn and
  Querido}]{loiola2007asurvey}
Loiola, E., Abreu, N., Boaventura-Netto, P., Hahn, P. and Querido, T. (2007) A
  survey of the quadratic assignment problem.
\newblock \emph{European Journal of Operational Research}, \textbf{176},
  657--690.

\bibitem[{Lyzinski \emph{et~al.}(2014)Lyzinski, Fishkind and
  Priebe}]{Lyzinski2014SeededGM}
Lyzinski, V., Fishkind, D.~E. and Priebe, C.~E. (2014) {Seeded graph matching
  for correlated Erd{\"o}s-R{\'e}nyi graphs}.
\newblock \emph{J. Mach. Learn. Res.}, \textbf{15}, 3513--3540.

\bibitem[{Ma and Liao(2020)}]{ma2020review}
Ma, C.-Y. and Liao, C.-S. (2020) A review of protein--protein interaction
  network alignment: From pathway comparison to global alignment.
\newblock \emph{Computational and Structural Biotechnology Journal},
  \textbf{18}, 2647--2656.

\bibitem[{Maretic \emph{et~al.}(2019)Maretic, Gheche, Chierchia and
  Frossard}]{maretic2019got}
Maretic, H.~P., Gheche, M.~E., Chierchia, G. and Frossard, P. (2019) {GOT: An
  Optimal Transport framework for Graph comparison}.
\newblock In \emph{Advances in Neural Information Processing Systems}.

\bibitem[{Maretic \emph{et~al.}(2020)Maretic, Gheche, Minder, Chierchia and
  Frossard}]{maretic2020wasserstein}
Maretic, H.~P., Gheche, M.~E., Minder, M., Chierchia, G. and Frossard, P.
  (2020) Wasserstein-based graph alignment.
\newblock \emph{arXiv preprint arXiv:2003.06048}.

\bibitem[{McKay and Piperno(2014)}]{McKay2014PracticalGI}
McKay, B.~D. and Piperno, A. (2014) Practical graph isomorphism, ii.
\newblock \emph{J. Symb. Comput.}, \textbf{60}, 94--112.

\bibitem[{M{\'e}moli(2011)}]{memoli2011gromov}
M{\'e}moli, F. (2011) {G}romov--{W}asserstein distances and the metric approach
  to object matching.
\newblock \emph{Foundations of Computational Mathematics}, \textbf{11},
  417--487.

\bibitem[{Milano \emph{et~al.}(2017)Milano, Guzzi, Tymofiyeva, Xu, Hess, Veltri
  and Cannataro}]{Milano2017AnEA}
Milano, M., Guzzi, P.~H., Tymofiyeva, O., Xu, D., Hess, C.~P., Veltri, P. and
  Cannataro, M. (2017) An extensive assessment of network alignment algorithms
  for comparison of brain connectomes.
\newblock \emph{BMC Bioinformatics}, \textbf{18}.

\bibitem[{Mislove \emph{et~al.}(2007)Mislove, Marcon, Gummadi, Druschel and
  Bhattacharjee}]{mislove2007measurement}
Mislove, A., Marcon, M., Gummadi, K.~P., Druschel, P. and Bhattacharjee, B.
  (2007) Measurement and analysis of online social networks.
\newblock In \emph{Proceedings of the 7th ACM SIGCOMM conference on Internet
  measurement},  29--42.

\bibitem[{O'Connor \emph{et~al.}(2021)O'Connor, McGoff and
  Nobel}]{o2021estimation}
O'Connor, K., McGoff, K. and Nobel, A.~B. (2021) Estimation of stationary
  optimal transport plans.
\newblock \emph{arXiv preprint arXiv:2107.11858}.

\bibitem[{O'Connor \emph{et~al.}(2022)O'Connor, McGoff and
  Nobel}]{o2020optimal}
O'Connor, K., McGoff, K. and Nobel, A.~B. (2022) Optimal transport for
  stationary {Markov} chains via policy iteration.
\newblock \emph{Journal of Machine Learning Research}, \textbf{23}, 1--52.

\bibitem[{Ornstein(1973)}]{ornstein1973application}
Ornstein, D.~S. (1973) An application of ergodic theory to probability theory.
\newblock \emph{The Annals of Probability}, \textbf{1}, 43--58.

\bibitem[{Perozzi \emph{et~al.}(2014)Perozzi, Al-Rfou and
  Skiena}]{Perozzi2014DeepWalkOL}
Perozzi, B., Al-Rfou, R. and Skiena, S. (2014) {DeepWalk: online learning of
  social representations}.
\newblock \emph{Proceedings of the 20th ACM SIGKDD international conference on
  Knowledge discovery and data mining}.

\bibitem[{Peyr{\'e} and Cuturi(2019)}]{peyre2019computational}
Peyr{\'e}, G. and Cuturi, M. (2019) \emph{Computational Optimal Transport}.
\newblock Now Publishers, Inc.

\bibitem[{Peyr{\'e} \emph{et~al.}(2016)Peyr{\'e}, Cuturi and
  Solomon}]{peyre2016gromov}
Peyr{\'e}, G., Cuturi, M. and Solomon, J. (2016) {G}romov-{W}asserstein
  averaging of kernel and distance matrices.
\newblock In \emph{International Conference on Machine Learning},  2664--2672.
  PMLR.

\bibitem[{Plummer(2007)}]{plummer2007graph}
Plummer, M.~D. (2007) Graph factors and factorization: 1985--2003: a survey.
\newblock \emph{Discrete Mathematics}, \textbf{307}, 791--821.

\bibitem[{Riesen and Bunke(2008)}]{riesen2008iam}
Riesen, K. and Bunke, H. (2008) {IAM} graph database repository for graph based
  pattern recognition and machine learning.
\newblock In \emph{Joint International Workshops on Statistical Techniques in
  Pattern Recognition and Structural and Syntactic Pattern Recognition},
  287--297. Springer.

\bibitem[{Schellewald and Schn{\"o}rr(2005)}]{schellewald2005probabilistic}
Schellewald, C. and Schn{\"o}rr, C. (2005) Probabilistic subgraph matching
  based on convex relaxation.
\newblock In \emph{Energy Minimization Methods in Computer Vision and Pattern
  Recognition},  171--186. Berlin, Heidelberg: Springer Berlin Heidelberg.

\bibitem[{Singh \emph{et~al.}(2008)Singh, Xu and Berger}]{Singh2008global}
Singh, R., Xu, J. and Berger, B. (2008) Global alignment of multiple protein
  interaction networks with application to functional orthology detection.
\newblock \emph{Proceedings of the National Academy of Sciences}, \textbf{105},
  12763--12768.

\bibitem[{Song \emph{et~al.}(2016)Song, Gao, Wang and An}]{song2016measuring}
Song, J., Gao, Y., Wang, H. and An, B. (2016) Measuring the distance between
  finite {Markov} decision processes.
\newblock In \emph{Proceedings of the 2016 international conference on
  autonomous agents \& multiagent systems},  468--476. International Foundation
  for Autonomous Agents and Multiagent Systems.

\bibitem[{Sutherland \emph{et~al.}(2003)Sutherland, O'{B}rien and
  Weaver}]{sutherland2003spline}
Sutherland, J.~J., O'{B}rien, L.~A. and Weaver, D.~F. (2003) Spline-fitting
  with a genetic algorithm: A method for developing classification structure-
  activity relationships.
\newblock \emph{Journal of Chemical Information and Computer Sciences},
  \textbf{43}, 1906--1915.

\bibitem[{Titouan \emph{et~al.}(2019)Titouan, Courty, Tavenard and
  Flamary}]{titouan2019optimal}
Titouan, V., Courty, N., Tavenard, R. and Flamary, R. (2019) Optimal transport
  for structured data with application on graphs.
\newblock In \emph{International Conference on Machine Learning},  6275--6284.

\bibitem[{Toivonen \emph{et~al.}(2011)Toivonen, Zhou, Hartikainen and
  Hinkka}]{toivonen2011compression}
Toivonen, H., Zhou, F., Hartikainen, A. and Hinkka, A. (2011) Compression of
  weighted graphs.
\newblock In \emph{International Conference on Knowledge Discovery and Data
  Mining},  965--973.

\bibitem[{Torr(2003)}]{torr2003solving}
Torr, P. H.~S. (2003) {Solving {M}arkov Random Fields using Semi Definite
  Programming}.
\newblock \emph{AI \& Society}.

\bibitem[{van Wyk and van Wyk(2004)}]{van2004apocs}
van Wyk, B. and van Wyk, M. (2004) A {POCS}-based graph matching algorithm.
\newblock \emph{IEEE Transactions on Pattern Analysis and Machine
  Intelligence}, \textbf{26}, 1526--1530.

\bibitem[{Vayer \emph{et~al.}(2020)Vayer, Chapel, Flamary, Tavenard and
  Courty}]{vayer2020fused}
Vayer, T., Chapel, L., Flamary, R., Tavenard, R. and Courty, N. (2020) Fused
  {G}romov-{W}asserstein distance for structured objects.
\newblock \emph{Algorithms}, \textbf{13}, 212.

\bibitem[{Vayer \emph{et~al.}(2019)Vayer, Flamary, Tavenard, Chapel and
  Courty}]{vayer2019sliced}
Vayer, T., Flamary, R., Tavenard, R., Chapel, L. and Courty, N. (2019) Sliced
  {G}romov-{W}asserstein.
\newblock In \emph{Advances in Neural Information Processing Systems}.

\bibitem[{Villani(2008)}]{Villani2008OptimalTO}
Villani, C. (2008) \emph{Optimal Transport: Old and New}, vol. 338.
\newblock Springer-Verlag Berlin Heidelberg.

\bibitem[{Vishwanathan \emph{et~al.}(2010)Vishwanathan, Schraudolph, Kondor and
  Borgwardt}]{svn2010graph}
Vishwanathan, S., Schraudolph, N.~N., Kondor, R. and Borgwardt, K.~M. (2010)
  Graph kernels.
\newblock \emph{Journal of Machine Learning Research}, \textbf{11}, 1201--1242.

\bibitem[{Xu \emph{et~al.}(2019{\natexlab{a}})Xu, Luo and
  Carin}]{xu2019scalable}
Xu, H., Luo, D. and Carin, L. (2019{\natexlab{a}}) Scalable
  {G}romov-{W}asserstein learning for graph partitioning and matching.
\newblock In \emph{Advances in Neural Information Processing Systems},
  3052--3062.

\bibitem[{Xu \emph{et~al.}(2019{\natexlab{b}})Xu, Luo, Zha and
  Duke}]{xu2019gromov}
Xu, H., Luo, D., Zha, H. and Duke, L.~C. (2019{\natexlab{b}})
  {G}romov-{W}asserstein learning for graph matching and node embedding.
\newblock In \emph{International Conference on Machine Learning},  6932--6941.

\bibitem[{Yan \emph{et~al.}(2016)Yan, Yin, Lin, Deng, Zha and
  Yang}]{yan2016ashort}
Yan, J., Yin, X.-C., Lin, W., Deng, C., Zha, H. and Yang, X. (2016) A short
  survey of recent advances in graph matching.
\newblock In \emph{Proceedings of the 2016 ACM on International Conference on
  Multimedia Retrieval},  167--174. New York, NY, USA: Association for
  Computing Machinery.
\newblock \urlprefix\url{https://doi.org/10.1145/2911996.2912035}.

\bibitem[{Yan \emph{et~al.}(2008)Yan, Cheng, Han and Yu}]{Yan2008MiningS}
Yan, X., Cheng, H., Han, J. and Yu, P.~S. (2008) Mining significant graph
  patterns by leap search.
\newblock In \emph{SIGMOD Conference}.

\bibitem[{Yartseva and Grossglauser(2013)}]{Yartseva2013OnTP}
Yartseva, L. and Grossglauser, M. (2013) On the performance of percolation
  graph matching.
\newblock In \emph{COSN '13}.

\bibitem[{Yu \emph{et~al.}(2018)Yu, Yan, Wang, Liu and Li}]{yu2018generalizing}
Yu, T., Yan, J., Wang, Y., Liu, W. and Li, b. (2018) Generalizing graph
  matching beyond quadratic assignment model.
\newblock In \emph{Advances in Neural Information Processing Systems}, vol.~31.
  Curran Associates, Inc.

\bibitem[{Zalesky \emph{et~al.}(2010)Zalesky, Fornito and
  Bullmore}]{zalesky2010network}
Zalesky, A., Fornito, A. and Bullmore, E.~T. (2010) Network-based statistic:
  identifying differences in brain networks.
\newblock \emph{Neuroimage}, \textbf{53}, 1197--1207.

\bibitem[{Zaslavskiy \emph{et~al.}(2009)Zaslavskiy, Bach and
  Vert}]{zaslavskiy2009apath}
Zaslavskiy, M., Bach, F. and Vert, J.-P. (2009) {A Path Following Algorithm for
  the Graph Matching Problem}.
\newblock \emph{IEEE Transactions on Pattern Analysis and Machine
  Intelligence}, \textbf{31}, 2227--2242.

\bibitem[{Zhang(2000)}]{zhang2000existence}
Zhang, S. (2000) Existence and application of optimal {Markovian} coupling with
  respect to non-negative lower semi-continuous functions.
\newblock \emph{Acta Mathematica Sinica}, \textbf{16}, 261--270.

\bibitem[{Zhou and De~la Torre(2016)}]{zhou2016factorized}
Zhou, F. and De~la Torre, F. (2016) Factorized graph matching.
\newblock \emph{IEEE Transactions on Pattern Analysis and Machine
  Intelligence}, \textbf{38}, 1774--1789.

\end{thebibliography}

\newpage
\begin{appendices}
\section{Experimental Details}
\label{app:exp_details}
In this appendix, we provide further details for the experiments discussed in Section \ref{sec:experiments}. We refer to \texttt{ExactOTC} as the exact implementation of NetOTC, while \texttt{EntropicOTC} specifically represents the regularized version of the NetOTC algorithm.

\subsection{Network Classification}
\label{app:class_exp_details}

In order to compute approximate solutions to the NetOTC problem, we used the \texttt{EntropicOTC} algorithm with $L = 10$, $T = 50$, $\xi = 100$, and $50$ Sinkhorn iterations.
The FGW cost was computed with a default parameter choice of $\alpha = 0.5$.
The experiment was run on Matlab and a 24-core node in a university-owned computing cluster.

\subsection{Network Isomorphism}
\label{app:iso_exp_details}

The classes of networks we dealt with in this experiment are as follows. We give details on how we generated the random networks.
\begin{itemize}
    \item SBM: A description of an SBM is provided in Section \ref{sec:sbm_alignment}. The given set informs the number of vertices in each block. For example, SBM (7,7,7,7) indicates an SBM having 4 blocks with 7 vertices in each block. SBM (10,8,6) has 3 blocks, and each block has 10, 8, and 6 vertices. The connection probabilities within the block were fixed to 0.7, and the probabilities between blocks were 0.1.
    \item Erdos-Renyi network: Erdos-Renyi network is a random network in which every pair of vertices is connected with an independent probability $p$ by an unweighted edge.
    Remark that the Erdos-Renyi network is equivalent to an SBM with a single block. 
    For each network generation, the number of vertices was randomly determined from the given set. The connection probability $p$ was fixed as given in each class.
    \item Random weighted adjacency matrix: Each element of the adjacency matrix is randomly sampled from the given set. For example, random weighted adjacency matrix $\{0,1,2\}$ indicates a network that its adjacency matrix elements are all sampled from the set $\{0,1,2\}$. In order to restrict the network to an undirected network, we only sample the upper triangular matrix and the lower triangular matrix is symmetrically filled. Note that the number of vertices is randomly determined between 6 and 20.
    \item Random Lollipop network: An example of a lollipop network is presented in Figure \ref{fig:isomorphism_example}. 
    A lollipop consists of a candy part and a stick part. The number of vertices in the candy part is randomly chosen between 7 and 15. The number of vertices in the stick part is also determined between 7 and 15. We also add edges inside the candy to vary the lollipop. With a probability of 0.5, we connect an edge between pair of vertices in the candy.
\end{itemize}
 
We discuss how we establish if the algorithm successfully identified the isomorphism map. For given two isomorphic networks $G_1 = (U, E_1, w_1)$ and $G_2 = (V, E_2, w_2)$, each algorithm returns an vertex alignment $\pi_{\text{v}}: (U,V) \rightarrow [0,1]$. Define a hard alignment function $\psi(\cdot) = \argmax_{v\in V} \pi(\cdot, v)$, where $\psi:U\rightarrow V$ returns the most aligned vertex in $G_2$ for each vertex in $G_1$.
If $\psi$ satisfies the following conditions, the algorithm correctly detects the isomorphism.
\begin{enumerate}
    \item $\psi$ is bijective.
    \item For every $(u_1,u_2)\in E_1$, $(\psi(u_1),\psi(u_2)) \in E_2$ and $w_2(\psi(u_1),\psi(u_2)) = w_1(u_1,u_2)$.
    \item For every $(v_1,v_2)\in E_2$, $(\psi^{-1}(v_1),\psi^{-1}(v_2)) \in E_1$ and  $w_1(\psi^{-1}(v_1),\psi^{-1}(v_2)) = w_2(v_1,v_2)$.
\end{enumerate}

Figure \ref{fig:wheel_fgw_ot} shows an additional example of detecting network isomorphism. As in Figure \ref{fig:isomorphism_example}, NetOTC successfully detects isomorphism, while OT and FGW do not.

\begin{figure}[ht]
\centering
\begin{subfigure}{0.7\textwidth}
\includegraphics[width=\textwidth, trim=0.15cm 0.15cm 0.15cm 0.15cm, clip]{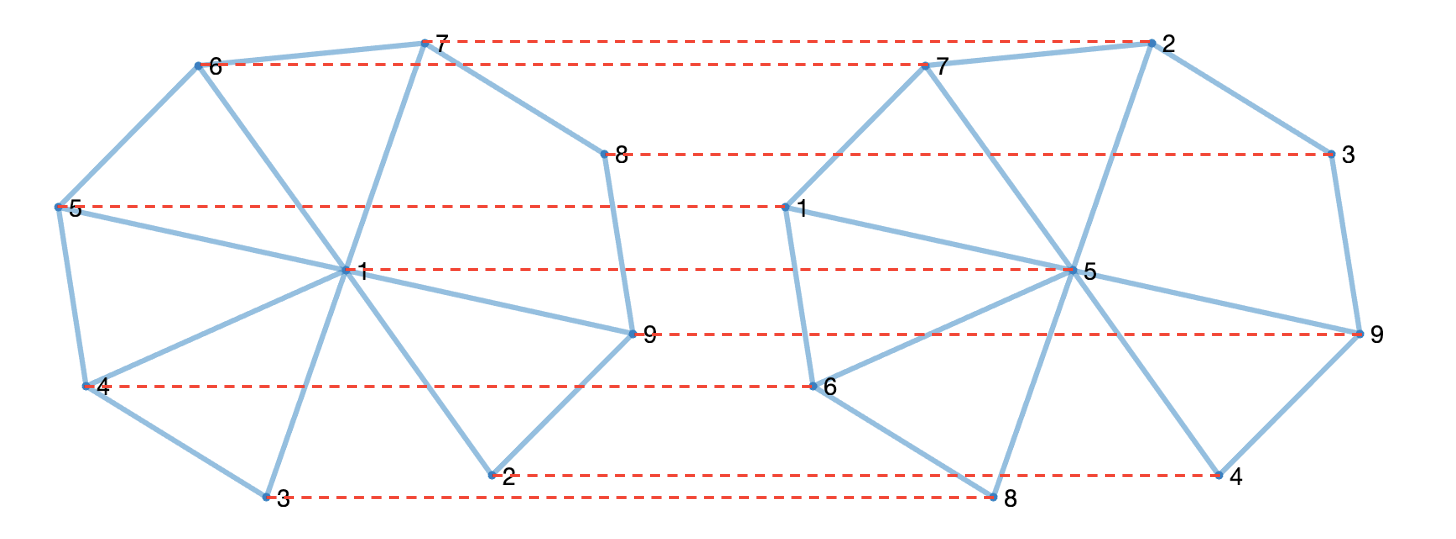}
\caption{NetOTC alignment.}
\end{subfigure}
\begin{subfigure}{0.7\textwidth}
\includegraphics[width=\textwidth, trim=0.15cm 0.15cm 0.15cm 0.15cm, clip]{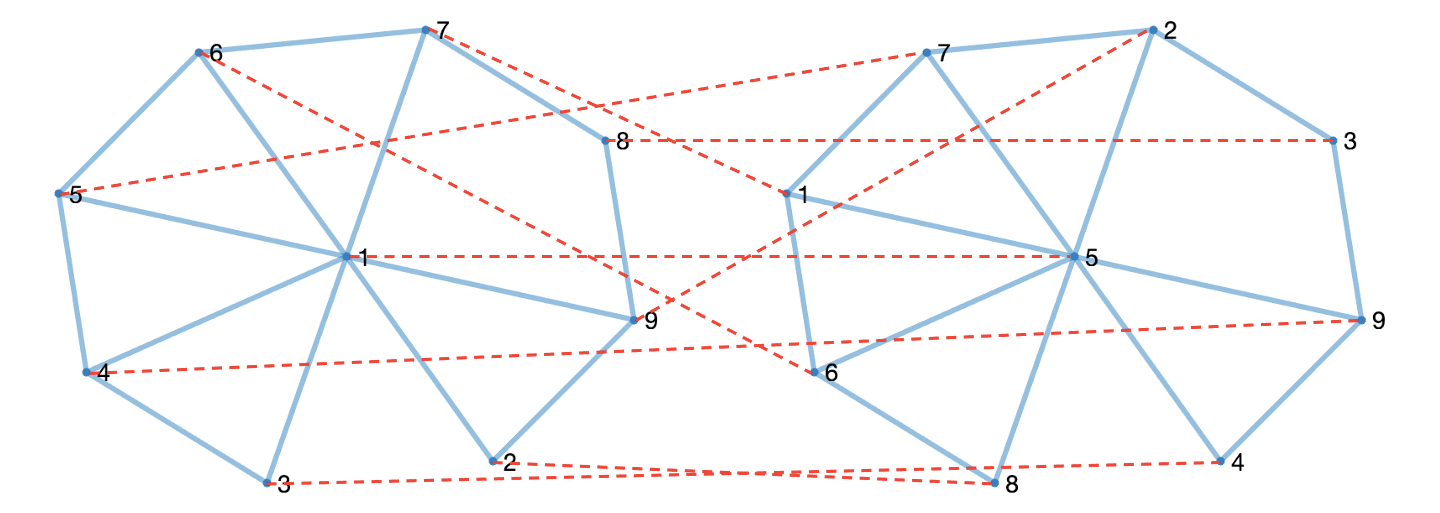}
\caption{OT alignment.}
\end{subfigure}
\begin{subfigure}{0.7\textwidth}
\includegraphics[width=\textwidth, trim=0.15cm 0.15cm 0.15cm 0.15cm, clip]{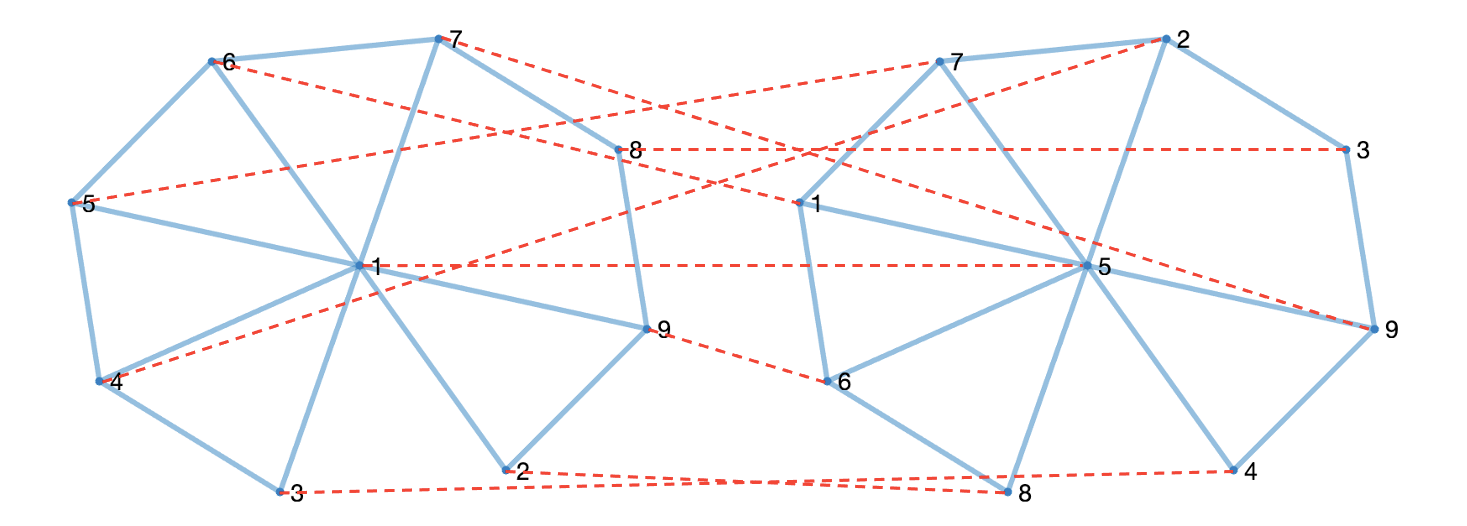}
\caption{FGW alignment.}
\end{subfigure}
\caption{Alignment of two isomorphic wheel networks obtained by NetOTC, OT, and FGW. To make the task challenging, two edges were removed.}
\label{fig:wheel_fgw_ot}
\end{figure}

The \texttt{ExactOTC} algorithm was utilized to solve the NetOTC problem, and a fixed $\alpha = 0.5$ was used for FGW. 
Note that the algorithms were applied only to the connected network among the generated networks. 95.48\% of the generated networks were connected on average.
The experiment was developed in Matlab and run on a 24-core node in a university-owned computing cluster.

\subsection{Stochastic Block Model Alignment}
\label{app:sbm_exp_details}

Cross validation was performed for FGW (to select $\alpha \in \{0, 0.1, ..., 1\}$) by randomly generating 10 pairs of SBM networks and computing alignments of vertices and edges.
Parameters that yielded the highest average alignment accuracy were selected, where separate parameters were chosen for optimizing vertex and edge alignment.
0.1 and 1 were selected for vertex and edge alignment, respectively. We note that this implies GW and FGW were equivalent for edge alignment in our experiment.
The \texttt{ExactOTC} algorithm was used to compute solutions to the NetOTC problem.
The experiment was developed and run in Matlab on a personal machine.

\subsection{Network Factors}
\label{app:factors_exp_details}

Network $G_2$ has $b$ vertices embedded in $\mathbb{R}^5$ where the vertices are sampled from $\mathcal{N}_5(0, \sigma^2 \mathbf{I}_5)$. We denote the vertices of $G_2$ as $V_1, \cdots, V_b$. 
Then, we sample $m$ points from $\mathcal{N}_5(V_i, \mathbf{I}_5)$ for each $i=1,...,b$ and the points will be the vertices of $G_1$. The total number of vertices of $G_1$ is $bm$. We set $b=6$ and $m=5$ in this experiment. 
We used $\alpha=0.5$ as a default trade-off parameter when applying FGW. We note that the choice of $\alpha \in \{0,0.1,\dots,1\}$ doesn't affect the alignment result much.

This experiment was conducted not only in exact factor situations but also when we have approximate factors. We call it an approximate factor when Definition \ref{def:factor} approximately holds as follows. For every $v, v' \in V$ and a given error rate $\epsilon > 0$,
$$\sum \limits_{u' \in f^{-1}(v')} w_1(u, u') \in \left[ (1-\epsilon) \frac{d_1(u)}{d_2(v)}w_2(v, v'), (1+\epsilon) \frac{d_1(u)}{d_2(v)}w_2(v, v') \right], \quad \forall u \in f^{-1}(v),$$
and 
$$ \sum_{u \in f^{-1}(v)} \sum_{u' \in f^{-1}(v')} w_1(u,u') = \sum_{u \in f^{-1}(v)} \frac{d_1(u)}{d_2(v)} w_2(v,v').$$
In particular, we allowed 5\% error ($\epsilon=0.05$) for the second condition in this experiment.

\begin{table}[t]
\small
\centering
\begin{tabularx}{0.5\textwidth}{@{} X *6{c} @{}}
\toprule
  & \multicolumn{1}{c}{Exact factor} & \multicolumn{1}{c}{Approximate factor} \\
  \cmidrule(r){2-2} \cmidrule(l){3-3}
  & NetOTC & NetOTC  \\ \midrule
$\sigma = 2.5$ & \textbf{100.00} $\pm$ \textbf{0.00}  & \textbf{97.94} $\pm$ \textbf{0.26}\\ 
$\sigma = 2.0$ & \textbf{99.91} $\pm$ \textbf{0.65} & \textbf{97.57} $\pm$ \textbf{1.03}\\
$\sigma = 1.5$ & \textbf{97.46} $\pm$ \textbf{4.84} & \textbf{95.87} $\pm$ \textbf{3.37} \\
$\sigma = 1.0$ & \textbf{84.09} $\pm$ \textbf{11.02} & \textbf{84.06} $\pm$ \textbf{11.07}\\ \bottomrule
\end{tabularx}
\caption{Directed networks: alignment accuracies of network factors. Average accuracies observed over 100 random network factors are reported along with their standard deviation.}
\label{table:graph_factors_directed}
\end{table}

Table \ref{table:graph_factors_directed} reports the vertex alignment accuracy of NetOTC for directed networks at various variance settings. We may check that the performances are similar to applying NetOTC to undirected factor network pairs.
Similar to the Stochastic Block Model Alignment experiment, the \texttt{ExactOTC} algorithm was used to obtain solutions to the NetOTC problem and was run in Matlab on a personal machine.

\subsection{Runtime Analysis of NetOTC}

\begin{table}[t!]
\centering
\begin{tabular}{cccccccc} \toprule
  & AIDS & BZR & Cuneiform & MCF-7 & MOLT-4 & MUTAG & Yeast \\  \midrule \midrule
 \# of Vertices & 13.25  & 34.65  & 20.2 & 26.65  & 26.7  &  17.9 & 20.43 \\  \midrule
ExactOTC & 4.35  & 37.83  & 10.52  & 22.78 & 24.47 &   9.12 & 10.82\\  \midrule
 EntropicOTC & 0.43 & 5.57  & 1.04 & 2.91 & 3.07  &  0.77  & 1.12 \\ 
\bottomrule
\end{tabular}
\caption{Average Runtimes for datasets investigated in Section~\ref{sec:network_classification} (sec). We also report the average number of vertices for the randomly selected 20 pairs of graphs.}
\label{tab:classification_time}
\end{table}

\begin{table}[t!]
\centering
\begin{tabular}{ccc} \toprule
 & ExactOTC & EntropicOTC  \\  \midrule \midrule
 SBM-48-32 & 41.92 & 6.26  \\\midrule
 SBM-96-64 & 334.53 & 156.58 \\\midrule
  SBM-128-96 & 2143.42 & 1042.46 \\
\bottomrule
\end{tabular}
\caption{Runtimes for pairs of SBM (sec). SBM-$a$-$b$ denotes a comparison of two SBMs with a total of $a$ vertices and $b$ vertices. 
}
\label{tab:SBM_time}
\end{table}

We report the runtime of NetOTC for the experiments discussed in Sections~\ref{sec:network_classification} and \ref{sec:sbm_alignment}. The runtime metrics are presented for both \texttt{ExactOTC} and \texttt{EntropicOTC}, employing our Matlab implementation of these methods.
First, we randomly selected 20 pairs of graphs from each of the seven datasets used in Section~\ref{sec:network_classification}. The average computation time for each pair is detailed in Table~\ref{tab:classification_time}.
Next, we investigated the computation time for comparing two stochastic block models (SBMs). Following the procedure detailed in Section~\ref{sec:sbm_alignment}, we generated SBMs with 4 blocks. The within-block connection probabilities were set to 1, 0.8, 0.6, and 0.4, while the between-block connection probability was fixed at 0.1. The runtime results for comparing these SBMs are presented in Table~\ref{tab:SBM_time}.

\end{appendices}

\end{document}